\documentclass{article}

\usepackage[final]{neurips_2023}

\usepackage[utf8]{inputenc} %
\usepackage[T1]{fontenc}    %
\usepackage{hyperref}       %
\usepackage{url}            %
\usepackage{booktabs}       %
\usepackage{amsfonts}       %
\usepackage{nicefrac}       %
\usepackage{microtype}      %
\usepackage{xcolor}         %

\usepackage{algorithm}
\usepackage{algpseudocode}
\usepackage{amssymb}

\usepackage{natbib}
\usepackage{amsmath}
\usepackage{amsthm}
\usepackage{mathtools}
\usepackage{thmtools} 
\usepackage{thm-restate}
\usepackage{placeins}
\usepackage{hyperref}
\hypersetup{colorlinks,linkcolor={blue},citecolor={blue},urlcolor={red}} 
\usepackage{cleveref}
\usepackage{xcolor}
\usepackage{wrapfig}

\newtheorem{theorem}{Theorem}
\newtheorem{corollary}{Corollary}
\newtheorem{lemma}{Lemma}
\newtheorem{assumption}{Assumption}

\newtheorem{definition}{Definition}
\theoremstyle{definition}

\newtheorem{remark}{Remark}

\DeclareMathOperator*{\argmax}{arg\,max}
\DeclareMathOperator*{\argmin}{arg\,min}

\newcommand{\calA}{\mathcal{A}}
\newcommand{\calD}{\mathcal{D}}
\newcommand{\calE}{\mathcal{E}}
\newcommand{\calF}{\mathcal{F}}

\newcommand{\calH}{\mathcal{H}}
\newcommand{\calI}{\mathcal{I}}

\newcommand{\calP}{\mathcal{P}}
\newcommand{\calQ}{\mathcal{Q}}

\newcommand{\calX}{\mathcal{X}}

\newcommand{\starm}{m^*}

\newcommand{\barm}{\bar{m}}

\newcommand{\hatf}{\hat{f}}

\newcommand{\hatm}{\hat{m}}

\newcommand{\hatR}{\hat{R}}

\newcommand{\safe}{\textbf{safe}}
\newcommand{\true}{\textbf{True}}
\newcommand{\false}{\textbf{False}}

\newcommand{\Ndim}{\texttt{Ndim}}

\newcommand{\msafe}{m^*}
\newcommand{\msafealg}{\hat{m}}

\newcommand{\CReg}{\text{CReg}}

\newcommand{\propThreshold}{\beta_{\max}}

\newcommand{\Reg}{\text{Reg}}
\newcommand{\hatReg}{\widehat{\text{Reg}}}

\newcommand{\F}{\mathcal{F}}

\newcommand{\E}{\mathop{\mathbb{E}}}
\newcommand{\bP}{\mathop{\mathbb{P}}}
\newcommand{\A}{\mathcal{A}}

\newcommand{\Xscript}{\mathcal{X}}
\newcommand{\ordO}{\mathcal{O}}
\newcommand{\ordOt}{\tilde{\mathcal{O}}}
\newcommand{\Unif}{\text{Unif}}
\newcommand{\I}{\mathbb{I}}

\newcommand{\eventModelOracle}{\mathcal{W}_1}
\newcommand{\eventRiskEstimation}{\mathcal{W}_2}

\newcommand{\EstOracle}{\text{EstOracle}}
\newcommand{\EvalOracle}{\text{EvalOracle}}

\newcommand{\N}{\mathbb{N}}

\newcommand{\supscript}[2]{{#1}^{(#2)}}
\newcommand{\dkl}{D_{KL}}
\newcommand{\Ber}{\text{Ber}}
\newcommand{\comment}[1]{}

\usepackage[shortlabels]{enumitem}

\title{Proportional Response: Contextual Bandits for Simple and Cumulative Regret Minimization}

\author{%
  Sanath Kumar Krishnamurthy\\ %
  Management Science and Engineering\\
  Stanford University\\
  \texttt{sanathsk@stanford.edu} \\
  \And
  Ruohan Zhan \\
  Industrial Engineering and Decision Analytics\\
  Hong Kong University of Science and Technology\\
  \texttt{rhzhan@ust.hk} \\
  \AND
  Susan Athey\\
  Graduate School of Business\\
  Stanford University\\
  \texttt{athey@stanford.edu} \\
  \And
  Emma Brunskill\\
  Computer Science Department\\
  Stanford University\\
  \texttt{ebrun@cs.stanford.edu} \\
}

\begin{document}

\maketitle

\begin{abstract}
In many applications, e.g. in healthcare and e-commerce, the goal of a contextual bandit may be to learn an optimal treatment assignment policy at the end of the experiment. That is, to minimize simple regret. However, this objective remains understudied. We propose a new family of computationally efficient bandit algorithms for the stochastic contextual bandit setting, where a tuning parameter determines the weight placed on cumulative regret minimization (where we establish near-optimal minimax guarantees) versus simple regret minimization (where we establish state-of-the-art guarantees). Our algorithms work with any function class, are robust to model misspecification, and can be used in  continuous arm settings. This flexibility comes from constructing and relying on ``conformal arm sets" (CASs).
CASs provide a set of arms for every context, encompassing the context-specific optimal arm with a certain probability across the context distribution. Our positive results on simple and cumulative regret guarantees are contrasted with a negative result, which shows that no algorithm can achieve instance-dependent simple regret guarantees while simultaneously achieving minimax optimal cumulative regret guarantees.
\end{abstract}

\section{Introduction}
Learning and deploying personalized treatment assignment policies is crucial across domains such as healthcare and e-commerce \cite{murphy2003optimal,li2010contextual}. Traditional randomized control trials (RCTs), while foundational for policy learning \citep{banerjee2016influence, das2016impact}, can be inefficient and costly \cite{offer2021adaptive}. This motivates the study of adaptive sequential experimentation algorithms for the stochastic contextual bandit (CB) settings. The algorithm interacts with a finite sequence of users drawn \emph{stochastically} from a fixed but unknown distribution. At each round, the algorithm receives a \emph{context} (a user's feature vector), selects an action, and gets a corresponding reward. At the end of this adaptive experiment, the algorithm outputs a learned policy (mapping between contexts and actions). 

Our algorithms are designed with the dual objectives of minimizing \emph{simple regret} and \emph{cumulative regret}. Simple regret quantifies the difference between the expected rewards achieved by the optimal policy and the policy learned at the conclusion of the experimental process. In contrast, cumulative regret encapsulates the summation of differences between the expected rewards generated by the optimal policy and the exploration policies employed at each sequential round of decision-making.\footnote{Our formal definition of simple regret compares against the best policy in our policy class, while our cumulative regret definition compares against the global optimal policy (induced by the true conditional expected reward model). The reason for this discrepancy is because we use a regression based approach (due to computational considerations) for constructing our exploration policies.} %
Although there are many settings where simple regret is an important consideration, the majority of research in the contextual bandit field has focused on the minimization of cumulative regret. To the best of our knowledge, there is no general-purpose computationally efficient algorithm for pure exploration objectives like simple regret minimization in the contextual bandit setting. Further, there has been relatively little work so far into algorithms that explore the trade-off between multiple objectives like cumulative regret and simple regret (though see~\cite{athey2022contextual, erraqabi2017trading, yao2021power} for studies that address this empirically or juxtapose minimizing cumulative regret with estimating treatment effects or arm parameters). Our work seeks to address these gaps. We show that there is a trade-off between simple and cumulative regret minimization (formalized later in a lower-bound result). To navigate this trade-off, we proposes a new algorithm called Risk Adjusted Proportional Response (RAPR) with a tuning parameter $\omega\in[1, K]$, which governs the weight placed on the two objectives.\footnote{$K$ is the number of arms for the finite arm setting.} The algorithm is general-purpose (in that it can address any user-specified reward and policy classes), ensures near-optimal guarantees, and is also computationally efficient.

\textbf{Types of guarantees.} In our analysis, we consider two different types of bounds on simple and cumulative regret, worst-case and instance-dependent guarantees. Here instance-dependent guarantees refer to bounds that surpass worst-case rates by exploiting instances with large gaps between the conditional expected rewards of the optimal and sub-optimal arms. Recent work by \citep{foster2020beyond} has shown that it is not possible for contextual algorithms to have instance-dependent guarantees on cumulative regret (without suffering an exponential dependence on model class complexity); the authors instead  develop algorithms that achieve minimax optimal (worst case optimal) cumulative regret guarantees (with square-root dependence on model class complexity). \cite{li2022instance} developed the first general-purpose contextual bandit algorithm for pure exploration, and their algorithm achieved instance-dependent guarantees. They also show that instance-dependent best policy identification guarantees must come at the cost of worse than minimax optimal cumulative regret (discussed in detail later). We show a similar lower bound on cumulative regret for algorithms that achieve better instance-dependent simple regret guarantees, and propose the first family of algorithms that flexibly navigate such trade-offs.

\textbf{Overview of our guarantees.} The simple regret guarantees of RAPR are never worse than the minimax optimal rates (\Cref{theorem:Simple-Regret-for-RAPR}). Depending on the instance, RAPR achieves simple regret guarantees that are up to $O(1/\sqrt{\omega})$ times smaller compared to minimax optimal rates (\Cref{theorem:Simple-Regret-for-RAPR}). This improvement factor of $O(1/\sqrt{\omega})$ over minimax optimal rates is asymptotically achieved for instances where \emph{realizability} holds (the reward model class is well specified) and the gap between the best and second best arm in terms of conditional expected reward is at least $\Delta>0$ at every context (best-case instance in \Cref{theorem:Simple-Regret-for-RAPR}). RAPR provides these instance-dependent guarantees without the knowledge of any instance information. Unfortunately, the corresponding cumulative regret for the above instances is a factor of $O(\sqrt{\omega})$ times larger compared to minimax optimal rates (\Cref{theorem:Cumulative-Regret-for-RAPR}). The cumulative regret guarantees of our algorithm only degrade relative to the minimax optimal rate if the instance allows for better simple regret guarantees. Our lower bound (\Cref{thm:lower-new}) considers the instances described above with $\Delta=0.24$ (the gap between best and second best arm in terms of conditional expected reward). \Cref{thm:lower-new} shows that, for any algorithm that bounds the simple regret on these instances to $O(1/\sqrt{\omega})$ of the minimax optimal rates, its cumulative regret will be at least $\Omega(\sqrt{\omega})$ times the minimax optimal rates. 
RAPR thus achieves a near-optimal trade-off between guarantees on simple vs cumulative regrets when $T$ is large enough. The trade-off contrasts with non-contextual bandits, where successive elimination ensures improved (compared to minimax) instance-dependent guarantees for both simple and cumulative regret \citep{even2006action,slivkins2019introduction}.

\textbf{Types of CB algorithms.} Contextual bandit algorithms broadly fall into two categories: regression-free and regression-based. Regression-free algorithms create an explicit policy distribution, randomly choosing a policy for decision-making at any time-step \citep{agarwal2014taming,beygelzimer2011contextual,dudik2011efficient,li2022instance}. While these algorithms provide worst-case cumulative regret guarantees \citep{agarwal2014taming,beygelzimer2011contextual,dudik2011efficient} or instance-dependent PAC guarantees for policy learning \citep{li2022instance} without additional assumptions, they can be computationally intensive \cite{foster2020beyond}: they require solving and storing the output of $\Omega(\text{poly}(T))$ cost-sensitive classification (CSC) problems \citep{krishnamurthy2017active} at every epoch (or update step). In contrast, regression-based algorithms \citep[e.g.,][]{abbasi2011improved,foster2020beyond,simchi2020bypassing} construct a conditional arm distribution using regression estimates of the expected reward, allowing for methods that need only solve $\ordO(1)$ regression or CSC problems at every epoch (or update step). Traditionally, these algorithms relied on realizability assumptions for optimal regret guarantees, but recent advances allow for misspecified reward model classes \citep{carranza2023flexible,foster2020adapting,krishnamurthy2021adapting}. We develop regression-based algorithms and do not assume realizability. RAPR is the first general-purpose regression-based algorithm with attractive pure exploration (simple regret) guarantees.

\textbf{Overview of our algorithm.} We now describe the RAPR algorithm in more detail. We first define a surrogate objective for simple regret, the optimal cover, which is inversely proportional to the probability that the bandit exploration policy chooses the arm recommended by the unknown optimal policy. The optimal cover bounds the variance of evaluating the unknown optimal policy under our exploration policy. This surrogate objective can be minimized by appropriately designing our exploration policy/action selection kernels. To maintain the attractive computational properties of regression-based algorithms, RAPR does not construct an explicit distribution over policies as that distribution would have large support and would be computationally and memory intensive to maintain. Instead, the goal of minimizing the optimal cover is attained by directly constructing a distribution over arms for each arriving context. This in turn builds on a novel general-purpose uncertainty quantification at each context. Much of the existing literature constructs confidence intervals with point-wise guarantees, but existing approaches to constructing them rely on assumptions like linear realizability. For general function classes, these intervals may be too wide and are often computationally expensive to construct. To overcome this issue, we develop Conformal Arm Sets (CASs), which are a set of potentially optimal arms at each context. This uncertainty quantification is regression-based and computationally efficient to construct; it's general-purpose and shrinks at ``fast rates” (with square-root dependency on expected squared error bounds for regression). Unfortunately, these sets come with some risk of not containing the arm recommended by the optimal policy at every context. Nevertheless, we can use this uncertainty quantification to construct a distribution over arms at each context that helps us minimize the optimal cover by balancing the benefits and risks of relying on these CASs. The unavoidable trade-off between our simple and cumulative regret guarantees is an artifact of these risky sets. Beyond allowing us to trade off simple and cumulative regret guarantees, the flexibility of the approach also helps us extend to continuous arm settings and allows us to handle model misspecification.

\textbf{Other Related Work.} Our work connects to the literature on pure exploration, extensively studied in MAB settings  (see overview in  \citep{lattimore2020bandit}). 
\citep{even2006action,hassidim2020optimal} study elimination-based algorithms for fixed confidence best-arm identification (BAI). \citep{kasy2021adaptive,russo2016simple} study variants of Thompson Sampling with optimal asymptotic designs for BAI. \citep{karnin2013almost} propose sequential halving for fixed budget BAI. Our algorithm provides fixed confidence simple regret guarantees and can be seen as a generalization of successive elimination \citep{even2006action} to the contextual bandit setting. The key technical difference is that it is often impossible to construct sub-gaussian confidence intervals on conditional expected rewards. The uncertainty quantification we use is similar to the notion of conformal prediction (see \citep{vovk2005algorithmic} for a detailed exposition). Until recently, pure exploration had been nearly unstudied in  contextual bandits.  \citep{zanette2021design} provide a static exploration algorithm that achieves the minimax lower bound on sample complexity for linear contextual bandits.   \citep{li2022instance} then provided the first algorithm with instance-dependent $(\epsilon,\delta)$-PAC guarantees for contextual bandits. This algorithm is regression-free (adapts techniques from \citep{agarwal2014taming}) and requires a sufficiently large dataset of offline contexts as input. Hence, unfortunately, it inherits high memory and runtime requirements \citep[See][for a more detailed discussion]{foster2020beyond}. However, these costs come with the benefit that their notion of instance dependence leverages structure not only in the true conditional expected reward (as in \Cref{theorem:Simple-Regret-for-RAPR}) but also in the policy class (similar to policy disagreement coefficient \cite{foster2020instance}). They also  prove a negative result, showing that it is not possible for an algorithm to have instance-dependent $(0,\delta)$-PAC guarantees and achieve  minimax optimal  cumulative regret guarantees. Our hardness result is similar but complementary to their result, for we show a similar result for simple regret (rather than their $(0,\delta)$-PAC sample complexity).\footnote{In $(\epsilon,\delta)$ PAC sample complexity results, given an input $(\epsilon,\delta)$, the objective is to minimize the number of samples needed in order to output an $\epsilon$-optimal policy with probability at least $1-\delta$ (a "fix accuracy, compute budget" setting). In contrast, in our simple regret case, we consider how to minimize the error $\epsilon$ as the number of samples increases.} Our work also recovers some cumulative regret guarantees for the continuous arm case \citep{majzoubi2020efficient,zhu2022contextual}, with new guarantees on simple regret and robustness to misspecification. Note that our restriction to ``slightly randomized" policies for the continuous arm case results in regret bounds with respect to a ``slightly randomized" (smooth) benchmark \citep[see][for smooth regret]{zhu2022contextual}.

\subsection{Stochastic Contextual Bandits}
\label{sec:scb-setting}

We consider the stochastic contextual bandit setting,  with context space $\Xscript$, (compact) arm space $\A$, and a fixed but unknown distribution $D$ over contexts and arm rewards. $D_{\Xscript}$ refers to the marginal distribution over contexts, and $T$ signifies the number of rounds or sample size. 
At each time $t \in [T]$\footnote{For any $n\in\mathbb{N}^+$, we use notation $[n]$ to denote the set $\{1,...,n \}$}, the environment draws a context $x_t$ and a reward vector $\smash{r_t \in [0,1]^{\A}}$ from $D$; the learner chooses an arm $a_t$ and observes a reward $r_t(a_t)$. To streamline notation for discrete and continuous arm spaces, we consider a finite measure space $(\A,\Sigma,\mu)$ over the set of arms, with $K$ shorthand for $\mu(\A)$.\footnote{Here $\Sigma$ is a $\sigma$-algebra over $\A$ and $\mu$ is a bounded set function from $\Sigma$ to the real line.} For ease of exposition, we focus on the finite/discrete arm setting. Here $\A=[K]$ and $\mu$ is the count measure, and $\mu(S)=|S|$ for any $S\subseteq\A$. A (deterministic) policy $\pi$ maps contexts to \emph{singleton arm sets} $\Sigma_1:=\{{a}|a\in\A\}$\footnote{The introduction of $\Sigma_1$ is to allow for easy generalization to the continuous arm setting.}. With some abuse of notation, we also let $\pi$ refer to the kernel given by $\pi(a|x)=I(a\in \pi(x))$. An action selection kernel (randomized policy) $p: \A\times\Xscript\rightarrow [0,1]$ is a probability kernel that describes a distribution $p(\cdot|x)$ over arms at every context $x$. We let $D(p)$ be the induced distribution over $\Xscript \times \A \times [0,1]$, where sampling $(x, a, r(a)) \sim D(p)$ is equivalent to sampling $(x, r) \sim D$ and then sampling $a \sim p(\cdot|x)$.

A reward model $f$ maps $\Xscript\times\A$ to $[0,1]$, with $f^*(x,a):=\E_D[r_t(a)|x_t=x]$ denoting the true conditional expected reward model. Our algorithm works with a reward model class $\F$ and a policy class $\Pi$. For a given model $f$ and an action selection kernel $p$, we denote the expected instantaneous reward of $p$ with $f$ as $R_f(p)$. 
We write $R_{f^*}(p)$ as $R(p)$ to simplify notation  when no confusion arises.  The optimal policy associated with reward function $f$ is defined as $\pi_f$\footnote{subject to any tie-breaking rule.}. 
\begin{equation*}
    \begin{aligned}
        R_f(p) := \E_{x \sim D_{\Xscript}}\E_{a\sim p(\cdot|x)}[f(x, a)],\;\text{and } \pi_f \in \arg\max_{\pi} R_f(\pi).
    \end{aligned}
\end{equation*}
The policy $\pi_f$ induced by $f\in\calF$ is assumed to be within policy class $\Pi$ without loss of generality.\footnote{Note that $\Pi$ may contain policies that are not induced by models in the class $\F$.} For any $S\subseteq \A$, with some abuse of notation, we let $f(x,S)=\int_{a\in S} f(x,a)d\mu(a)/\mu(S)$. Note that $\pi_f(x)\in \arg\max_{S\in\Sigma_1} f(x,S)$ for all $x$. 
The \textit{regret} of a policy $\pi$ with respect to $f$ is  the difference between the optimal value and the actual value of $\pi$, denoted as $\Reg_f(\pi) := R_f(\pi_f) - R_f(\pi)$. Finally, we let $\pi^*$ denote the optimal policy in the class $\Pi$ and let $\Reg_{\Pi}(\cdot)$ denote the regret with respect to $\pi^*$. That is, $\pi^*\in\arg\max_{\pi\in\Pi} R(\pi)$ and $\Reg_{\Pi}(\pi):=R(\pi^*)-R(\pi)$.

\textbf{Objectives.} Contextual bandit algorithms adaptively construct action sampling kernels (exploration policies) $\{p_t\}_{t\in[T]}$ used to collect data over the $T$ rounds. At the end of the adaptive experiment, the adaptively collected data is used to learn a policy $\hat{\pi}\in\Pi$. We study two main objectives to measure quality of these outputs: [Objective 1] \textit{Cumulative regret minimization.} Cumulative regret ($\CReg_T$) is given by $\CReg_T:= \sum_{t=1}^T \Reg_{f^*}(p_t)$. It compares the cumulative expected reward obtained during the experiment with the expected reward of the policy ($\pi_{f^*}$) induced by the true conditional expected reward ($f^*$). We seek to minimize cumulative regret which is equivalent to maximizing cumulative expected reward during the experiment. [Objective 2] \textit{Simple regret minimization.} Simple regret is given by $\Reg_{\Pi}(\hat{\pi})$. It compares the expected reward of the learnt policy $\hat{\pi}\in\Pi$ against the value of the optimal policy in the class $\Pi$. We seek to minimize simple regret which is equivalent to maximizing expected reward of the policy learnt at the end of the experiment. To understand the kind of exploration kernels ($\{p_t\}_{t\in[T]}$) that help with policy learning, we now identify a surrogate objective for simple regret (called optimal cover) that is in terms of the kernels used for exploration.

\begin{definition}[Cover]
Given a kernel $p$ and a policy $\pi$, we define the cover of policy $\pi$ under the kernel $p$ to be,
\begin{equation}
    \label{eq:decisional-divergence}
    V(p,\pi):=\E_{x\sim D_{\Xscript}, a\sim \pi(\cdot|x)}\bigg[\frac{\pi(a|x)}{p(a|x)}\bigg].
\end{equation}
Additionally, for any pair of kernels $(p,q)$, we let $V(p,q):=\E_{x\sim D_{\Xscript}, a\sim q(\cdot|x)}[q(a|x)/p(a|x)]$. Finally, we use the term optimal cover for kernel $p$ to refer to $V(p,\pi^*)$.
\end{definition}
The cover measures the quality of data collected under the action selection kernel $p$ for evaluating a given policy $\pi$ and bounds the variance of commonly used unbiased estimators for policy value \citep[e.g.,][]{agarwal2014taming,hadad2021confidence,zhan2021off}. In particular, the cover under optimal policy $\frac{1}{T}\sum_{t=1}^T V(p_t,\pi^*)$ can be treated as a surrogate objective for simple regret minimization (proven in \Cref{app:prove-objective}), which is particularly instructional in designing our algorithm to minimize simple regret.

\textbf{Extending notation to continuous arms.} 
In the continuous arm setting, evaluating arbitrary deterministic policies can be infeasible without extra assumptions \citep{mou2023kernel}. Thus, we focus on ``slightly randomized'' policies by generalizing $\Sigma_1$ to be the arm sets with measure one  ($\Sigma_1:=\{S\in\Sigma|\mu(S)=1\}$).\footnote{Note that our restriction to ``slightly randomized" policies for the continuous arm case results in regret bounds with respect to a ``slightly randomized" (smooth) benchmark. Hence for the continuous arm case, our cumulative regret bounds translate to smooth regret bounds from \cite{zhu2022contextual} with $K=1/h$. Where $h$ is the measure of smoothness in smooth regret (a leading objective for this setting).  } %
The granularity of these sets can be adjusted by scaling the finite measure $\mu$, which also affects the value of $K=\mu(\A)$. We then continue defining policies be maps from $\Xscript$ to $\Sigma_1$ and $\Pi$ is a class of such policies. We overload notation and define the induced kernel as $\pi(a|x)=I(a\in\pi(x))$, which is a valid definition since $\int_{a}I(a\in\pi(x))d\mu(a)=\mu(\pi(x))=1$. All the remaining definitions, including $R_f(\pi),\pi_f,\pi^*$ and $V(p,\pi)$, relied on these induced kernels and continue to hold. While there are some measure theoretic issues that remain to be discussed, we defer these details to \Cref{sec:expanded-notation}. 

\textbf{Uniform sampling.} Our algorithm frequently selects an arm uniformly from a constructed set of arms. In the context of a set $S\subseteq\A$, uniform sampling refers to selecting an arm from the distribution $q(a):=I(a\in S)/\mu(S)$. This constitutes a probability measure since its integral over $\A$ equals 1. In the discrete arm setting, uniform sampling from a set $S\subseteq\A$ implies selecting an arm according to the distribution $I(a\in S)/|S|$.

\subsection{Oracle Assumptions}
\label{sec:oracle-assumptions}

Our algorithm relies on two sub-routines. For generality, we abstract away these sub-routines by stating them as oracle assumptions, for which we describe two oracles, $\EstOracle$ and $\EvalOracle$, in Assumptions \ref{ass:regression-oracle} and \ref{ass:policy-evaluation-oracles} respectively. The $\EstOracle$ sub-routine is for estimating conditional expected reward models (\Cref{ass:regression-oracle}), and the $\EvalOracle$ sub-routine is for estimating policy values (\Cref{ass:policy-evaluation-oracles}) according to the true and estimated reward models. 

These sub-routine tasks are supervised learning problems. Hence, the average errors for the corresponding tasks can be bounded in terms of the number of samples ($n$) and a confidence parameter ($\delta'$). The oracle assumptions specify the estimation rates. We let $\xi:\N\times[0,1]\rightarrow[0,1]$ denote the estimation rate for these oracles. For simplicity, we assume that they share the same rate and that $\xi(n,\delta')$ scales polynomially in $1/n$ and $\log(1/\delta')$. In order to simplify the analysis, we also require $\xi(n/3,\delta'/n^3)$ be non-increasing in $n$.\footnote{This ensures that $\xi_m$ defined in \Cref{lemma:event-pure-exploration} is non-increasing in $m$ for any epoch schedule with increasing epoch lengths.} We now formally describe these oracle assumptions, starting with $\EstOracle$.

\begin{assumption}[Estimation Oracle]
\label{ass:regression-oracle}
We assume access to a reward model estimation oracle ($\EstOracle$) that takes as input an action selection kernel $p$, and $n$ independently and identically drawn samples from the distribution $D(p)$. The oracle then outputs an estimated model $\hatf\in\F$ such that for any $\delta'\in(0,1)$, the following holds with probability at least $1-\delta'$:
$$\E_{x\sim D_{\Xscript}}\E_{a\sim p(\cdot|x)}[ (\hatf(x, a) - f^*(x,a))^2 ] \leq B + \xi(n,\delta')$$
Where $B\geq 0$ is a fixed but unknown constant that may depend on the model class $\F$ and distribution $D$, but is independent of the action selection kernel $p$. 
\end{assumption}

In \Cref{ass:regression-oracle}, the parameter $B$ measures the bias of model class $\calF$; under realizability, $B$ equals 0. The function $\xi$ characterizes the estimation variance, which decreases with increasing sample size. \textit{As long as the variance term (which shrinks as we gather more data) is larger than the fixed unknown bias ($B$)}, we have from \Cref{ass:regression-oracle} that the expected squared error for the estimated reward model is bounded by $2\xi$. We use this bound on expected squared error to further bound how accurately the estimated reward model evaluates policies in the class $\Pi$ (\Cref{lemma:bound_reward_diff_continuous}). However, since $B$ is unknown, we need a test to detect when this policy evaluation bounds starts failing (which can only happen after the variance term gets dominated by the unknown bias term). To construct this test, our algorithm relies on $\EvalOracle$, which provides consistent independent policy value estimates and helps compare them with policy value estimates with respect to the estimated reward model.

\begin{assumption}[Evaluation Oracle]
\label{ass:policy-evaluation-oracles}
We assume access to an oracle ($\EvalOracle$) that takes as input an action selection kernel $p$, $n$ independently and identically drawn samples from the distribution $D(p)$, a set of $m$ models $\{g_i|i\in[m]\}\subseteq\F$, and another action selection kernel $q$. The oracle then outputs a policy evaluation estimator $\hatR$ of true policy value, and a set of $m$ policy evaluation estimators $\{\hatR_{g_i}|i\in[m] \}\}$ that estimate policy value with respect to the models $g_1,g_2,\dots,g_m$ respectively. Such that for any $\delta'\in(0,1)$, the following conditions simultaneously hold with probability at least $1-(m+1)\delta'$:
\begin{itemize}[itemsep=-2pt,leftmargin=10pt,topsep=-3pt]
    \item $|\hatR(\pi)-R(\pi)| \leq \sqrt{2V(p,\pi)\xi(n,\delta')} + 2\xi(n,\delta')/(\min_{(x,a)\in\Xscript\times\calA}p(a|x))$ for all $\pi\in\Pi\cup\{q\}$.
    \item $|\hatR_{f}(\pi)-R_{f}(\pi)| \leq \sqrt{2\xi(n,\delta')}$ for all $\pi\in\Pi\cup\{q\}$ and for all $f\in\{g_i|i\in[m]\}$.
\end{itemize}
\end{assumption}

When $\calF$ and $\Pi$ are finite, one can construct oracles such that Assumptions \ref{ass:regression-oracle} and \ref{ass:policy-evaluation-oracles} hold with $\xi(n,\delta')=\ordO(\log(\max(|\calF|,|\Pi|)/\delta')/n)$. One example of such a construction is given by using empirical squared loss minimization for $\EstOracle$, using inverse propensity scores (IPS) for estimating $R(\pi)$ in $\EvalOracle$, and using the empirical average for estimating $R_f(\pi)$ in $\EvalOracle$. The guarantees of these assumptions can be derived using Bernstein's inequality and union bounding. When $\calF$ has pseudo-dimension \citep{koltchinskii2011oracle} bounded by $d$ and $\Pi$ has the Natarajan-dimension bounded by $d$ \citep{jin2022policy,jin2022upper}, one can construct oracles such that Assumptions \ref{ass:regression-oracle} and \ref{ass:policy-evaluation-oracles} hold with $\xi(n,\delta')=\ordO(d\log(nK/\delta')/n)$.

\section{Algorithm}

\begin{algorithm}[ht]
  \caption{$\omega$ Risk Adjusted Proportional Response ($\omega$-RAPR)}
  \label{alg:RAPR}
  \textbf{input:} Trade-off parameter $\omega \in [1,K]$, proportional response threshold $\beta_{\max}=1/2$, and confidence parameter $\delta$ (used in definition of $\xi_m$).
  \begin{algorithmic}[1] %
  \State Let $p_1(a|x) \equiv 1/\mu(\A) = 1/K$, $\hatf_1 \equiv 0$, $\alpha_1 = 3K$, $\tau_1 = 3$, and $\safe = \true$.
  \For{epoch $m=1,2,\dots$}
  \State $\tau_m = 2\tau_{m-1}$. \Comment{Doubling epochs.}
  \If{$\safe$}
    \For{round $t=\tau_{m-1}+1,\dots, \tau_{m}$ }
      \State Observe context $x_t$, sample $a_t \sim p_m(\cdot|x_t)$, and observe $r_t(a_t)$.
    \EndFor
    \State Let $S_m$ denote the data collected in epoch $m$.
    \State We split $S_m$ into three equally sized sets $S_{m,1},S_{m,2}$ and $S_{m,3}$.
    \State Let $\hatf_{m+1} \leftarrow \EstOracle(p_m,S_{m,1})$, and let $C_{m+1}$ be given by \Cref{def:confidence_set}.%
    \State Let $\eta_{m+1}$ be the solution to \eqref{eq:choose-etam} and let $\alpha_{m+1}:=3K/\eta_{m+1}$.%
    \Comment{$S_{m,2}$ is used here.}
    \State Now let $p_{m+1}$ be given by \eqref{eq:define_pm_continuous}.
    \State Let $\hatR_{m+1}, \{\hatR_{m+1,\hatf_i}|i\in[m+1] \} \leftarrow \EvalOracle(p_m,S_{m,3},\{\hatf_i|i\in[m+1]\},p_{m+1})$.
    \If{\eqref{eq:misspecification-test} does not hold. }
        \State $\msafealg, \safe \leftarrow m, \false$. 
    \EndIf
  \Else{}
    \For{round $t=\tau_{m-1}+1,\dots, \tau_{m}$ }
      \State Observe context $x_t$, sample $a_t \sim p_{\msafealg}(\cdot|x_t)$, and observe $r_t(a_t)$.
    \EndFor
    \State Let $S_m$ denote the data collected in epoch $m$.
    \State Let $\hatR_{m+1}, \{\hatR_{m+1,\hatf_i}|i\in[\msafealg] \} \leftarrow \EvalOracle(p_{\msafealg},S_{m},\{\hatf_i|i\in[\msafealg]\},p_{\msafealg})$.
  \EndIf
  \EndFor
  \end{algorithmic}
\end{algorithm}

At a high level, our algorithm operates in two modes, indicated by a Boolean variable ``$\safe$". During mode one ($\safe=\true$), where estimated reward models are sufficiently accurate at evaluating policies in the class $\Pi$,\footnote{Where the estimated reward models pass the misspecification test.} we use our estimated models to update our action selection kernel used during exploration. During mode two ($\safe=\false$), where the condition for mode one no longer holds, we stop updating the action selection kernel used for exploration. Operationally our algorithm runs in epochs/batches indexed by $m$. Epoch $m$ begins at round $t=\tau_{m-1}+1$ and ends at $t=\tau_m$, and we use $m(t)$ to denote the epoch index containing round $t$. We let $\msafealg$ denote the critical epoch, at the end of which our algorithm changes mode (with ``$\safe$" being updated from ``$\true$" to ``$\false$"); we refer to $\msafealg$ as the \emph{algorithmic safe epoch}. For all rounds in  epoch $m\leq \msafealg$, our algorithm samples action using the action selection kernel $p_m$ defined later in \eqref{eq:define_pm_continuous}. For $m> \msafealg$, our algorithm samples action using  $p_{\msafealg}$ --the action selection kernel used in the algorithmic safe epoch $\hat{m}$.

We now describe the critical components of our algorithm. These include (i) data splitting and using oracle sub-routines; (ii)  \emph{misspecification tests}, which we use to identify the $\safe$-mode switching epoch $\msafealg$; and (iii) \emph{conformal arm sets}, which presents a new form of uncertainty quantification that is critical in constructing $p_{m+1}$ at the end of each epoch $m\in[\msafealg]$. Finally, we use these components to describe our final algorithm.

\textbf{Data splitting and oracle sub-routines.} Consider an epoch $m\in[\msafealg]$. Let $S_m$ denote the set of samples collected in this epoch: $S_m=\{(x_t,a_t,r_t(a_t))|t\in[\tau_{m-1},\tau_m]\}$. Our algorithm splits $S_m$ into three equally-sized subsets: $S_{m,1},S_{m,2}$ and $S_{m,3}$. \Cref{alg:RAPR} outlines using these subsets and the oracles (described in \Cref{sec:oracle-assumptions}) to estimate reward models and evaluate policies. Based on Assumptions \ref{ass:regression-oracle} and \ref{ass:policy-evaluation-oracles}, we bound the errors for these estimates in terms of $\xi_{m+1}=2\xi((\tau_m-\tau_{m-1})/3,\delta/(16m^3))$, where $\delta$ is a specified confidence parameter. As we will see later, our algorithm relies on these bounds to test for misspecification and construct action selection kernels.

\textbf{Misspecification test.} We first discuss the need for our misspecification test. Note that \Cref{ass:regression-oracle} is flexible and allows our reward model class $\F$ to be misspecified. In particular, the squared error of our reward model estimate may depend on an unknown bias term $B$. To account for this unknown $B$, it is useful to center our analysis around the safe epoch $\msafe:=\argmax\{m\geq 1|\xi_{m+1}\geq 2B\}$, which denotes the last epoch where variance dominates bias. We show that for any epoch $m\in[\msafe]$, the estimated reward model $\hatf_{m+1}$ is ``sufficiently accurate'' at evaluating the expected reward of any policy in $\Pi\cup\{p_{m+1}\}$. This property is critical in ensuring that the constructed action selection kernel $p_{m+1}$ has low exploration regret $\Reg_{f^*}(p_{m+1})$ and a small optimal cover ($V(p_{m+1},\pi^*)$). Since $B$ and $\msafe$ are unknown, we need to test whether the estimated reward model is sufficiently accurate at evaluating these policies. When the test fails, the algorithm sets the variable ``safe'' to $\false$ and stops updating the action selection kernel used for exploration. The core idea for this test comes from \cite{working2023mspaper} although its application to simple regret minimization is new, and the form of our test differs a bit. We now state our misspecification test \eqref{eq:misspecification-test}. 
At the end of each epoch $m$, the  test is passed if \eqref{eq:misspecification-test} holds:
\begin{equation}
\label{eq:misspecification-test}
    \begin{aligned}
             &\max_{\pi\in\Pi\cup\{p_{m+1}\}}{|\hatR_{m+1,\hatf_{m+1}}(\pi)-\hatR_{m+1}(\pi)|}  - {\sqrt{\alpha_m\xi_{m+1}}\sum_{\barm\in[m]}\frac{\hatR_{m+1,\hatf_{\barm}}(\pi_{\hatf_{\barm}})-\hatR_{m+1,\hatf_{\barm}}(\pi)}{40\barm^2 \sqrt{\alpha_{\barm-1}\xi_{\barm}}}}\\ 
             &\leq 2.05\sqrt{\alpha_m\xi_{m+1}} + 1.1\sqrt{\xi_{m+1}},
    \end{aligned}
\end{equation}
where $\alpha_{\barm}$ empirically bounds  $V(p_{\barm},\pi^*)$, the optimal cover for the action selection kernel used in epoch $\barm$ (see \eqref{eq:condition_alpham_inductive}). The first term in \eqref{eq:misspecification-test} measures how well the estimated reward model $\hatf_{m+1}$ evaluates the policy $\pi$, and the second term accounts for under-explored policies (policies that have high regret under the reward model $\hatf_{m}$ would be less explored in epoch $m$).

\textbf{Conformal arm sets.} We proceed to introduce the notion of \emph{conformal arm sets} (CASs), based on which we construct the action selection kernels employed by our algorithms. At the beginning of each epoch $m$, we construct CASs, denoted as $\{C_m(x,\zeta)|x\in\Xscript,\zeta\in[0,1]\}$; here  $\zeta$ controls the probability with which the set $C_m$ contains the optimal arm. The construction of these sets rely on the models ($\hatf_1,\dots,\hatf_m$) estimated from data up to epoch $m-1$, as defined below.

\begin{definition}[Conformal Arm Sets] 
\label{def:confidence_set}
Consider $\zeta\in(0,1)$. At epoch $m$, for context $x$, the arm set $C_{m}(x, \zeta)$ is given by \eqref{eq:conformal-arm-set}.
\begin{equation}
\label{eq:conformal-arm-set}
\begin{aligned}
    &C_{m}(x, \zeta) :=   \pi_{\hatf_m}(x) \bigcup\bar{C}_{m}(x, \zeta),\;\; \bar{C}_{m}(x, \zeta) :=   \bigcap_{\barm \in[m]} \Tilde{C}_{\barm}\Big(x, \frac{\zeta}{2\barm^2}\Big),\\
    &\Tilde{C}_{\barm}(x, \zeta') := \Bigg\{a: 
    \hatf_{\barm}(x, \pi_{\hatf_{\barm}}(x)) - \hatf_{\barm}(x,a)\leq \frac{20\sqrt{\alpha_{\barm-1}\xi_{\barm}}}{\zeta'}
    \Bigg\}  \; \forall \barm\in[m], \zeta'\in(0,1).
\end{aligned}
\end{equation}
\end{definition}

Similar to conformal prediction (CP) \citep{vovk2005algorithmic,shafer2008tutorial}, CASs have marginal coverage guarantees. We show that with high probability, we have $\pi^*(x)$ lies in $C_{m}(x,\zeta)$ with probability at least $1-\zeta$ over the context distribution. That is, $\Pr_{x\sim D_{\Xscript} }(\pi^*(x)\in C_{m}(x,\zeta)) \geq 1 - \zeta$ with high-probability (see \Cref{sec:basic-result-on-CAS}). However, there is also a key technical difference. While CP provides coverage guarantees for the conditional random outcome, CASs provide coverage guarantees for $\pi^*(x)$ -- which is not a random variable given the context $x$. Hence, intervals estimated by CP need to be wide enough to account for conditional outcome noise, whereas CASs do not. 
CASs also have several advantages compared to pointwise confidence intervals used in UCB algorithms. First, CASs are computationally easier to construct. Second, CAS widths have a polynomial dependency on model class complexity, whereas pointwise intervals may have an exponential dependence for some function classes \citep[see lower bound examples in][]{foster2020instance}. Third, pointwise intervals require realizability, whereas the guarantees of CASs hold even without realizability (as long as the misspecification test in \eqref{eq:misspecification-test} holds). However, it's important to remember that these benefits of CASs come with the risk of only covering $\pi^*(x)$ marginally over the context distribution -- that is, these sets may not contain $\pi^*(x)$ at all $x$.

\textbf{Risk Adjusted Proportional Response Algorithm.} We now describe the design of our algorithm, which is summarized in \Cref{alg:RAPR}. The algorithm depends on the following input parameters: $\omega\in[1,K]$ which controls the trade-off between simple and cumulative regret, the proportional response threshold $\propThreshold = 1/2$, and confidence parameter $\delta$. The algorithm also computes $\eta_{m+1}$ (risk adjustment parameter for $p_{m+1}$), $\alpha_{m+1}$ (empirical bound on optimal cover for $p_{m+1}$), and $\lambda_{m+1}(\cdot)$ (empirical bound on average CAS size). At the end of every epoch $m\in[\msafealg]$, we construct the action selection kernel $p_{m+1}$ given by \eqref{eq:define_pm_continuous}.   
\begin{equation}
\label{eq:define_pm_continuous}
    p_{m+1}(a|x) = \frac{(1-\beta_{\max})I[a\in C_{m+1}(x,\beta_{\max}/\eta_{m+1})]}{\mu\big(C_{m+1}(x,\beta_{\max}/\eta_{m+1})\big)} + \int_0^{\propThreshold}\frac{I[a\in C_{m+1}(x,\beta/\eta_{m+1}) ]}{\mu\big(C_{m+1}(x,\beta/\eta_{m+1})\big)}\mbox{d}\beta.
\end{equation}
At any context $x$, sampling arm $a$ from $p_{m+1}(\cdot|x)$ is equivalent to the following. Sample $\beta$ uniformly from $[0,1]$, then sample arm $a$ uniformly from the set $C_{m+1}(x,\min(\propThreshold,\beta)/ \eta_{m+1})$. A small $\beta$ results in a larger CAS and a higher probability of containing the optimal arm for the sampled context. However, uniformly sampling an arm from a larger CAS also implies a lower probability on every arm in the set. Sampling $\beta$ uniformly allows us to respond proportionately to the risk of not sampling the optimal arm while enjoying the benefits of smaller CASs. We refer to this as the \textit{Proportional Response Principle}.

Similarly note that, a larger risk-adjustment parameter $\eta_{m+1}$ encourages reliance on less risky albeit larger CASs. We want to choose $\eta_{m+1}$  to tightly bound the the optimal cover (surrogate for simple regret), subject to cumulative regret constrains imposed by the trade-off parameter $\omega$. To do this, we first let $\lambda_{m+1}(\eta)$ be a high-probability empirical upper bound on $E_{x\sim D_{\Xscript}}[\mu(C_{m+1}(x,\propThreshold/\eta))]$. Hence, using \eqref{eq:condition_alpham_inductive}, we can upper bound the optimal cover ($V(p_{m+1},\pi^*)$) by $\frac{\lambda_{m+1}(\eta_{m+1})}{1-\propThreshold} + \frac{K}{\eta_{m+1}}$. Our choice of $\eta_{m+1}$ approximately minimizes this upper bound on the optimal cover, by choosing the largest feasible $\eta\in[\eta_m,\sqrt{\omega K/\alpha_m}]$ such that $\lambda_{m+1}(\eta)\leq \frac{K}{\eta}$ (see \eqref{eq:choose-etam}). Note that this choice of $\eta_{m+1}$ balances the risk of a small $\eta$ (large $\frac{K}{\eta}$) with the benefits of a small $\lambda_{m+1}(\eta)$ (small $\frac{\lambda_{m+1}(\eta)}{1-\propThreshold}$).
\begin{equation}
    \label{eq:choose-etam}
    \begin{aligned}
        &\lambda_{m+1}(\eta) := \min\bigg(1+\frac{1}{|S_{m,2}|}\sum_{t\in S_{m,2}}\mu(\bar{C}_{m+1}\Big(x, \frac{\beta_{\max}}{\eta}\Big)) + \sqrt{\frac{K^2\ln(8|S_{m,2}|(m+1)^2/\delta)}{2|S_{m,2}|}},K\bigg),\\
        &\eta_{m+1} \leftarrow \max\Bigg\{\eta_{m},\;\max\Bigg\{\eta=\frac{|S_{m,2}|}{n}\bigg|n \in [|S_{m,2}|], \eta\leq \sqrt{\frac{\omega K}{\alpha_{m}}},\lambda_{m+1}(\eta) \leq \frac{K}{\eta}  \Bigg\}\Bigg\}.
    \end{aligned}
\end{equation}
With $\eta_{m+1}$ chosen, the action selection kernel $p_{m+1}$ is decided. Now let $\alpha_{m+1}=3K/\eta_{m+1}$, which is a high-probability empirical upper bound on $V(p_{m+1},\pi^*)$. We then use $\alpha_{m+1}$ at the end of epoch $m+1$ to construct CASs, compute the risk-adjustment parameter, and  test for misspecification. 

\textbf{Computation.} We have $\ordO(\log(T))$ epochs. At the end of any epoch $m\in[\msafealg]$, we solve three optimization problems. The first is for estimating $\hatf_{m+1}$, which often reduces to empirical squared loss minimization and is computationally tractable for several function classes $\F$. The second is for computing the risk-adjustment parameter in \eqref{eq:choose-etam} which can be solved via binary search. The third is for the misspecification test in \eqref{eq:misspecification-test}, which can be solved via two calls to a cost-sensitive classification (CSC) solver (don't need this when assuming realizability, further if we only care about cumulative regret, sufficient to use the simpler test in \cite{krishnamurthy2021adapting}). Finally, to learn a policy $\hat{\pi}$ at the end of $T$ rounds, we need to solve \eqref{eq:variance-penalized-optimization} using a CSC solver (under realizability we can set $\hat{\pi}=\pi_{\hatf_{m(T)-1}}$). Hence, overall, $\omega$-RAPR makes exponentially fewer calls to solvers compared to regression-free algorithms like \cite{li2022instance}.

\section{Main Results}

Our algorithm/analysis/results hold for both the discrete and continuous arm cases. As discussed before, minimizing optimal cover helps us ensure improved simple regret guarantees. Hence $\alpha_{m}\in[1,3K]$ (the high-probability empirical upper bound on the optimal cover $V(p_{m},\pi^*)$) will play a crititcal role thoughout this results section. We start with stating our cumulative regret bounds. 

\begin{restatable}{theorem}{thmRAPRCumulativeRegret}
\label{theorem:Cumulative-Regret-for-RAPR}
Suppose Assumptions \ref{ass:regression-oracle} and \ref{ass:policy-evaluation-oracles} hold. Then with probability $1-\delta$, $\omega$-RAPR attains the following cumulative regret guarantee. Here $\xi_{m+1} = \xi((\tau_m-\tau_{m-1})/3,\delta/(16m^3)$ for all $m$.
\begin{subequations}
    \begin{align}
            \CReg_T &\leq \Tilde{\ordO}\Bigg( \sum_{t=\tau_1+1}^{T} \sqrt{\frac{K}{\alpha_{m(t)}}\frac{\alpha_{m(t)-1}}{\alpha_{m(t)}}}\bigg(\sqrt{KB}+\sqrt{K\xi_{m(t)}}\bigg) \Bigg) \label{eq:instance-dep-cumulative-regret}\\ 
            &\leq \Tilde{\ordO}\Bigg(\sqrt{\omega KB}T + \sum_{t=\tau_1+1}^{T} \sqrt{\omega K\xi_{m(t)}} \Bigg). \label{eq:minimax-cumulative-regret} 
    \end{align}
\end{subequations}
Where we use $\Tilde{\ordO}$ to hide terms logarithmic in $T,K,\xi(T,\delta)$.
\end{restatable}
We start with discussing \eqref{eq:minimax-cumulative-regret}. The first part $\sqrt{\omega KB}T$ comes from the bias of the regression oracle with model class $\F$ and will vanish under the realizability assumption. The second part  $\sum_{t=\tau_1+1}^{T} \sqrt{\omega K\xi_{m(t)}}$, when setting $\omega=1$,  recovers near-optimal (upto logarithmic factors) minimax cumulative regret guarantees for common model classes, as demonstrated by the following examples. 
\begin{corollary}
    We consider $\omega$-RAPR with appropriate oracles in the following cases and let $B$ denote the corresponding bias terms. When $\F$ and $\Pi$ are finite, $\CReg_T\leq \Tilde{\ordO}(\sqrt{\omega KB}T + \sqrt{\omega K T \log(\max(|\F|,|\Pi|)/\delta)})$ with probability at least $1-\delta$. When $\F$ has a finite pseudo dimension $d$, $\Pi$ has a finite Natarajan dimension $d$, and $\A$ is finite, $\CReg_T\leq \Tilde{\ordO}(\sqrt{\omega KB}T + \sqrt{\omega K T d\log(TK/\delta)})$ with probability at least $1-\delta$. Note that under realizability ($B=0$), $1$-RAPR achieves near-optimal minimax cumulative guarantees.
\end{corollary}
In \eqref{eq:instance-dep-cumulative-regret}, we observe that the multiplicative $\sqrt{\omega}$ cost to cumulative regret is only incurred if the empirical bound on optimal cover ($\alpha_m\in[1,3K]$) can get small. That is, our cumulative regret bounds degrade only if our algorithm better bounds the optimal cover and thus ensures better simple regret guarantees. \footnote{For large $t$, once our bounds on optimal cover can't be significantly improved, we have $\alpha_{m(t)}/\alpha_{m(t)-1}= O(1)$. Hence for large $T$, our cumulative regret is a factor of $\sqrt{K/\alpha_{m(T)+1}}$ larger than the near optimal minimax guarantees. As we will see in \Cref{thm:lower-new}, this multiplicative factor is unavoidable.} We now provide instance dependent simple regret guarantees for our algorithm.

\begin{restatable}{theorem}{thmRAPRSimpleRegret}
\label{theorem:Simple-Regret-for-RAPR}
Suppose Assumptions \ref{ass:regression-oracle} and \ref{ass:policy-evaluation-oracles} hold. For some $(\lambda,\Delta,A)\in [0,1]\times(0,1]\times[1,K]$, consider instances where for $1-\lambda$ fraction of contexts at most $A$ arms are $\Delta$ optimal (i.e. \eqref{eq:environment-conditions}  holds). 
\begin{equation}
\label{eq:environment-conditions}
    \mathbb{P}_{x\sim D_{\Xscript}}\Big(\mu\big(\{a\in\A: f^*(x,\pi_{f^*}(x)) - f^*(x,a) \leq \Delta \}\big) \leq A \Big) \geq 1- \lambda.
\end{equation}
Let $m'=\min(\msafealg,m(T))-1$. Let the learned policy $\hat{\pi}$ be given by \eqref{eq:variance-penalized-optimization} (equivalent to variance penalized policy optimization).%
\begin{equation}
\label{eq:variance-penalized-optimization}
  \hat{\pi}\in \arg\max_{\pi\in\Pi} \hatR_{m(T)}(\pi)- \frac{1}{2}\sqrt{\alpha_{m'}\xi_{m(T)}}\sum_{\barm\in[m']}\frac{\hatR_{m(T),\hatf_{\barm}}(\pi_{\hatf_{\barm}})-\hatR_{m(T),\hatf_{\barm}}(\pi)}{40\barm^2 \sqrt{\alpha_{\barm-1}\xi_{\barm}}} .
\end{equation}
Then with probability $1-\delta$, $\omega$-RAPR has the following simple regret bound when $T$ samples.
\begin{equation*}
    \begin{aligned}
            &\Reg_{\Pi}(\hat{\pi})\leq \ordO\Bigg(\sqrt{\alpha_{m'}\xi_{m(T)}}\Bigg)\\
            &\leq\!\ordO\Bigg(\sqrt{\xi_{m(T)}\min\bigg(K, A+K\lambda+\frac{K}{\omega} + \frac{K^{3/2}\omega^{1/2}}{\Delta}\sqrt{\xi_{\min(\msafe,m(T)-1)-\lceil\log_2\log_2(K)\rceil}}} \bigg) \bigg) \Bigg). %
    \end{aligned}
\end{equation*}

\end{restatable}

Under \eqref{eq:environment-conditions}, we can only argue that the expected (over context distribution) measure of $\Delta$ optimal arms is at most $(1-\lambda)A+K\lambda = O(A+K\lambda)$. Hence for large $T$, the best we can hope for is instance-dependant simple regret guarantees that shrink/improve over minimax guarantees by a factor of $\ordO(\sqrt{(A+K\lambda)/K})$. We show that this is guaranteed by \Cref{theorem:Simple-Regret-for-RAPR}. Suppose $\omega=K$, $\F$ has a finite pseudo dimension bounded by $d$, and $\Pi$ has a finite Natarajan dimension bounded by $d$. The simple regret guarantee of \Cref{theorem:Simple-Regret-for-RAPR} reduces to $\tilde{\ordO}(\min((\sqrt{Kd/{T}}, \sqrt{(A+K\lambda)d/T}+(K/\sqrt{\Delta})\sqrt{d/T}\sqrt[4]{B+d/T}))$. When the reward model estimation bias $B$ is small enough, the term $(K/\sqrt{\Delta})\sqrt{d/T}\sqrt[4]{B+d/T})$ is dominated by the remaining terms for large $T$. Hence, in this case, we get a simple regret bound of $\tilde{\ordO}(\sqrt{(A+K\lambda)d/T})$ for large $T$. As promised, this improves upon the minimax guarantees by a factor of $\ordO(\sqrt{(A+K\lambda)/K})$.

Note that \Cref{theorem:Cumulative-Regret-for-RAPR} guarantees are better for $\omega$ closer to $1$ whereas \Cref{theorem:Simple-Regret-for-RAPR} guarantees are better for $\omega$ closer to $K$. Hence these theorems show a tradeoff between the cumulative and simple regret guarantees for $\omega$-RAPR. \Cref{thm:lower-new} shows that improving upon minimax simple regret guarantees for instances satisfying \eqref{eq:environment-conditions} may come at the unavoidable cost of worse than minimax optimal cumulative regret guarantees. This contrasts with non-contextual bandits, where successive elimination ensures improved (compared to minimax) gap-dependent guarantees for both simple and cumulative regret. %

\begin{restatable}{theorem}{thmLowerBound}
\label{thm:lower-new}
Given parameters $K,F,T\in\N$ and $\phi\in [1,\infty)$. There exists a context space $\Xscript$ and a function class $\F\subseteq (\Xscript\times\A \rightarrow [0,1])$ with $K$ actions such that $|\F|\leq F$ and the following lower bound on cumulative regret holds:
\begin{align*}
  \inf_{\mathbf{A}\in\Psi_{\phi}}  \sup_{D\in \mathcal{D}} \; \E_D\bigg[\sum_{t=1}^T \big(r_t(\pi^*(x_t)) - r_t(a_{t}) \big) \bigg]\geq \Tilde{\Omega}\bigg( \sqrt{\frac{K}{\phi}} \sqrt{KT\log F} \bigg)
\end{align*}
Here $(a_{1},\dots a_{T})$ denotes the actions selected by an algorithm $A$. $\mathcal{D}$ denotes the set of environments such that $f^*\in\F$ and \eqref{eq:environment-conditions} hold with $(A,\lambda,\Delta)=(1,0,0.24)$. $\Pi$ denotes policies induced by $\F$. $\Psi_{\phi}$ denotes the set of CB algorithms that run for $T$ rounds and output a learned policy with a simple regret guarantee of $\sqrt{\phi \log F/T}$ for any instance in $\mathcal{D}$ with confidence at least $0.95$, i.e.,  $\Psi_{\phi}:=\{A: \mathbb{P}(\Reg(\hat{\pi}_{\mathbf{A}})\leq\sqrt{\phi \log F/T})\geq 0.95 \text{ for any instance in }\mathcal{D}\}$. Finally, $\Tilde{\Omega}(\cdot)$ hides factors logarithmic in $K$ and $T$.

\end{restatable}

\textbf{Near optimal trade-off of RAPR.} Note that the environments constructed in \Cref{thm:lower-new} satisfy $f^*\in\F$ with $\max(|\F|,|\Pi|)\leq F$ and  also satisfy \eqref{eq:environment-conditions} with $(A,\lambda,\Delta)=(1,0,0.24)$. With appropriate oracles, Assumptions \ref{ass:regression-oracle} and \ref{ass:policy-evaluation-oracles} are satisfied with $B=0$ (i.e. $\msafe=\infty$) and $\xi(n,\delta')=\ordO(\log(F/\delta')/n)$. Hence for large enough $T$, $\omega$-RAPR achieves a simple regret bound of $\tilde{\ordO}(\sqrt{(K/\omega)\log F/T})$ with probability at least $0.95$ and thus is a member of $\Psi_{\phi}$ for some $\phi=\tilde{\ordO}(K/\omega)$. \Cref{thm:lower-new} lower bounds the cumulative regret of such algorithms by $\Tilde{\Omega}( \sqrt{K/\phi} \sqrt{KT\log F} )=\Tilde{\Omega}( \sqrt{\omega} \sqrt{KT\log F} )$. Up to logarithmic factors, this matches the cumulative regret upper bound for $\omega$-RAPR. Re-emphasizing that the trade-off observed in Theorems \ref{theorem:Cumulative-Regret-for-RAPR} and \ref{theorem:Simple-Regret-for-RAPR} is near optimal for large $T$.

\textbf{Simulations.} To demonstrate the computational tractability of our approach, we ran a simulation on setting within a  $\mathbb{R}^2$context space, eight arms, linear models, and an exploration horizon of $5000$. Our algorithms ran in less than $9$ seconds on a Macbook M1 Pro. We also compare with other baselines on simple/cumulative regret. See \Cref{sec:simulations} for details.\\
\textbf{Conclusion.} We develop Risk Adjusted Proportional Response (RAPR), a computationally efficient regression-based contextual bandit algorithm. It is the first contextual bandit algorithm capable of trading-off worst-case cumulative regret guarantees with instance-dependent simple regret guarantees. The versatility of our algorithm allows for general reward models, handles misspecification, extends to finite and continuous arm settings, and allows us to choose the trade-off between simple and cumulative regret guarantees. The key ideas underlying RAPR are conformal arm sets (CASs) to quantify uncertainty, proportional response principle for cumulative regret minimization, optimal cover as a surrogate for simple regret, and risk adjustment for better bounds on the optimal cover. \footnote{
S.A. and S.K.K. are grateful for the support provided by Golub Capital Social Impact Lab and the ONR grant N00014-19-1-2468. E.B. is grateful for the support of NSF grant 2112926.}\\
\textbf{Limitations.} A limitation of our approach is that we do not utilize the structure of the policy class being explored. Further refining CASs with other forms of uncertainty quantification that leverage such structure can lead to significant improvements, and potentially avoid trade-offs between simple/cumulative regret when policy class structure allows for it.

\bibliography{ref}
\bibliographystyle{abbrvnat}

\appendix
\newpage

\section{Expanded Notations}
\label{sec:expanded-notation}

We start with expanding our notation from \Cref{sec:scb-setting} to include notation helpful for our proofs and expand to the continuous arm setting. 

\textbf{Measure over arms.} To recap, our algorithm and analysis adapt to both discrete and continuous arm spaces, where we consider a finite measure space $(\A,\Sigma,\mu)$ over the set of arms (i.e.~$\mu(\A)$ is finite) to unify the notation.\footnote{Here $\Sigma$ is a $\sigma$-algebra over $\A$ and $\mu$ is a bounded set function from $\Sigma$ to the real line.} As short hand, we use $K$ in lieu of $\mu(\A)$. We let $\Sigma_1$  be a set of arms in $\Sigma$ with measure one, i.e.~$\Sigma_1:=\{S\in\Sigma| \mu(S)=1 \}$. 

\textbf{Policies.} Let $\Tilde{\Pi}$ denote the universal set of policies. That is, $\Tilde{\Pi}$ is the set of all functions from $\Xscript$ to  $\Sigma_1$. The policy class $\Pi$ is a subset of $\Tilde{\Pi}$. We use $\pi(x)$ to denote the set of arms given $x\in\Xscript$ and use  $p_{\pi}(a|x)=\mathbb{I}(a\in\pi(x))$ to denote the induced probability measure over arms at $x$.\footnote{Note that for any $\pi\in\Tilde{\Pi}$, we have  $\int_{a\in\A} p_{\pi}(a|x) d\mu = \mu(\pi(x)) = 1$ at any $x\in\Xscript$.} With some abuse of notation, we use the notation $\pi(a|x)$ in lieu of $p_{\pi}(a|x)$. Below is the elaboration of our notation to both discrete and continuous arm spaces. 
\begin{itemize}[itemsep=-2pt,leftmargin=10pt,topsep=-1pt]
     \item \emph{Discrete arm space.} We choose $\mu$ to be the count-measure, where $\mu(S)=|S|$ for any $S\subseteq\A$ and $\mu(\A)=K$. In this case, $\Sigma_1$ contains singleton arm sets, and $\Tilde{\Pi}$ denotes  deterministic policies from $\Xscript$ to $\A$ where each policy maps a context to an action.
     \item \emph{Continuous arm space.} We choose $\mu$ to any finite measure, where $\mu(S) =\int_{ S}d\mu(a)$ for any $S\subseteq\A$, and in particular $\mu(\A)=K$. In this case, $\Sigma_1$ contains arm sets that may have an infinite number of arms but with total measure be 1 with respect to $\mu$.
\end{itemize}

\textbf{Space of action selection kernels.} In this paper, we will always define our action selection kernels with respect to the reference measure $\mu$, that is $p(S|x)=\int_{a\in S} p(a|x)d\mu$ for any $S\in\Sigma$ and $x\in\calX$. Based on the notation in \Cref{sec:scb-setting}, for any kernel $p$, we let $\Reg_f(p) = R_f(\pi_f) - R_f(p)$.

Now let $\calP$ denote the set of action selection kernels such that $p(a|x)\leq 1$ for all $(x,a)\in\calX\times \calA$, and in particular, we have the policy class $\Tilde{\Pi}\subset \calP$. We note that all  action selection kernels ($p$) considered in this paper  belong to the set $\calP$, allowing our analysis to rely on the fact that $p(\cdot|\cdot)\leq 1$. 

Note that, $\Reg_f(p)$ is non-negative for any $p\in\calP$. To see this, consider any context $x$. Recall that $\pi_f(x)\in \arg\max_{S\in\Sigma_1} f(x,S)$ for all $x$. Since $p(\cdot|\cdot)\leq 1$ for any $p\in\calP$, we have $\int_{a\in\A}p(a|x)f(x,a) d\mu$ is maximized when $p(a|x)=1$ for all $a\in\pi_f(x)$. That is, $f(x,\pi_f(x))=\max_{p\in\calP} \E_{a\sim p(\cdot|x)}[f^*(x,a)]$. Hence, $\max_{p\in\mathcal{P}}R_{f}(p) = R_{f}(\pi_f)$, so $\Reg_f(p)$ is non-negative for any $p\in\calP$.

\textbf{Connection to smooth regret \cite{zhu2022contextual}.} Recall that we define cumulative regret as $\CReg_T:= \sum_{t=1}^T \Reg_{f^*}(p_t)$, which measures regret w.r.t the benchmark $R_{f^*}(\pi_{f^*}) = \max_{p\in\mathcal{P}}R_{f^*}(p)$. As discussed earlier, \cite{zhu2022contextual} shows that smooth regret bounds are stronger than several other definitions of cumulative regret in the continuous arm setting \cite[e.g.,][]{majzoubi2020efficient}. Hence to show that our bounds are comparable/competitive for the continuous arm setting, we argue  that our definition of cumulative regret ($\CReg_T$) is equivalent to the definition of smooth regret in \cite{zhu2022contextual}. 

Let the loss vectors $l_t$ in \cite{zhu2022contextual} be given by $-r_t$. Let the smoothness parameter $h$ in \cite{zhu2022contextual} be given by $1/K$. And, let the base probability measure in \cite{zhu2022contextual} be given by $\mu/K$. Then, our benchmark ($\max_{p\in\mathcal{P}}R_{f^*}(p)$) is equal to the smooth benchmark ($\E[\text{Smooth}_h(x)]$) considered in \cite{zhu2022contextual}. Hence, our definition of cumulative regret ($\CReg_T$) is equal to smooth regret ($\Reg_{\text{CB},h}(T)$) when the loss, smoothness parameter, and base probability measure are given as above. This shows the equivalence in our definitions. 

Hence our near-optimal cumulative regret bounds (with $\omega=1$) recover several existing results for the stochastic contextual bandit setting up to logarithmic factors using only offline regression oracles. Our algorithm also handles reward model misspecification and does not assume realizability. We also provide instance-dependent simple regret bounds (for larger choices of $\omega$). The parameter $\omega$ allows us to trade-off between simple and cumulative regret bounds.

\textbf{Measure theoretic issues with continuous arms.} To avoid measure-theoretic issues, we require that for all models $f\in\F\cup\{f^*\}$, all contexts $x\in\Xscript$, and all real numbers $z\in\mathbb{R}$, we have the level set of arms $\{a| f(x,a)\leq z \}$ must lie in $\Sigma$. That is the reward models $f\in\F\cup\{f^*\}$ are measurable at every context $x$ with the Lebesgue measure on the range of $f(x,\cdot)$ and the measure $(\A,\Sigma,\mu)$ on the domain of $f(x,\cdot)$. We note that this isn't a strong condition and usually trivially holds.

Moreover, we require an additional condition as follows to simplify our arguments and allow for easy construction of our uncertainty sets (see \Cref{def:confidence_set}). We require that for all models $f\in\F$ and all contexts $x\in\Xscript$, we have $f(x,\pi_f(x))$ is equal to $\max_{a\in\A} f(x,a)$. This condition trivially holds for the finite-arm setting with $\mu$ as a count measure. For the continuous arm setting, this condition follows from requiring $\argmax_{a\in\A} f(x,a)$ lies in $\Sigma$ and has measure of at least one. 

\textbf{Additional notation.} For notational convenience, we let $U_{m}=20\sqrt{\alpha_{m-1}\xi_{m}}$ for any epoch $m$. Note that by construction (\eqref{eq:choose-etam} and $\alpha_m=3K/\eta_m$) $\alpha_m$ is non-increasing in $m$. Further, from the conditions in \Cref{sec:oracle-assumptions}, we have $\xi_{m+1}=2\xi((\tau_m-\tau_{m-1})/3,\delta/(16m^3))$ is non-increasing in $m$. Hence $U_m$ is also non-increasing in $m$. We also let $\alpha_0:=\alpha_1=3K$, and let $\alpha_m:=\alpha_{\msafealg}$ for any epoch $m\geq \msafealg$. Similarly, we let $\eta_0:=\eta_1=1$, and let $\eta_m:=\eta_{\msafealg}$ for any epoch $m\geq \msafealg$. Sometimes, we use use $C_m(x,\beta,\eta)$ in lieu of $C_m(x,\beta/\eta)$.

\begin{table}%
    \centering
    \begin{tabular}{|l|l|}
    \hline
        &Environment distribution\\
    \hline
        $\Xscript,\A$ & set of contexts and set of arms (respectively). \\
        $D$ & joint distribution over contexts and arm rewards.\\
        $D_{\Xscript}$ & marginal distribution over contexts.\\
        $D(p)$ & distribution over $\Xscript\times\A\times[0,1]$ induced by action selection kernel $p$.\\
        $f^*$ & true conditional expected reward, $f^*(x,a):=\E_D[r_t(a)|x_t=x]$.\\
    \hline
        & Measure space over arm sets\\
    \hline
        $(\A,\Sigma,\mu)$ & measure space over the set of arms with $K:=\mu(\A)$. \\
        $\Sigma_1$ & set of measurable arm sets with measure one, $\Sigma_1:=\{S\in\Sigma| \mu(S)=1 \}$. \\
        $\Tilde{\Pi}$ & set of all policies (functions from $\Xscript$ to  $\Sigma_1$).\\
        $\calP$ & set of kernels such that $p(a|x)\leq 1$ for all $(x,a)\in\calX\times \calA$.\\
    \hline
    & Algorithm inputs\\
    \hline
        $\propThreshold$ & proportional response threshold ($\propThreshold=1/2$).\\
        $\omega$ & trade-off parameter ($\omega\in[1,K]$).\\
        $\delta$ & confidence parameter ($\delta\in(0,1)$).\\
        $\F,\Pi$ & give reward model class and policy class.\\
    \hline
        & Policy value and optimal policy\\
    \hline
        $R_f(p)$ & $:=\E_{x \sim D_{\Xscript}}\E_{a\sim p(\cdot|x)}[f(x, a)]$, denotes the value of kernel $p$ under model $f$.\\
        $R(p)$ & $:=R_{f^*}(p)$, denotes the true value of kernel $p$.\\
        $\pi_f$ & $\in \arg\max_{\pi\in\Tilde{\Pi}} R_f(\pi)$, denotes the best universal policy under model $f$.\\
        $\pi^*$ & $\in\arg\max_{\pi\in{\Pi}} R(\pi)$, denotes the best policy in the class $\Pi$.\\
    \hline
        & Regret and cover\\
    \hline
        $\Reg_f(p)$ & $:=R_f(\pi_f)-R_f(p)$, denotes the regret of kernel $p$ under model $f$.\\
        $R(p)$ & $:=R_{f^*}(p)$, denotes the true value of kernel $p$.\\
        $\Reg_{\Pi}(p)$ & $:=R_f(\pi_f)-R_f(p)$, denotes the regret of kernel $p$ under model $f$.\\
        $V(p,q)$ & cover $V(p,q) := \E_{x\sim D_{\Xscript},a\sim q(\cdot|x)}\big[{q(a|x)}/{p(a|x)}\big]$.\\
    \hline
        & Epochs\\
    \hline
        $m$ & epoch index.\\
        $\xi_{m+1}$ & $:=2\xi((\tau_m-\tau_{m-1})/3,\delta/(16m^3))$, where $\xi$ is given in \Cref{sec:oracle-assumptions}. \\
        $\msafe$ & safe epoch, last epoch where $\xi_{m+1}\geq 2B$ (variance dominates bias). \\
        $\msafealg$ & last epoch that starts with $\safe$ set as $\true$. \\
    \hline
        & Algorithmic parameters\\
    \hline
        $\hatf_{m+1}\in\mathcal{F}$ & fitted reward model via regressions on samples collected in epoch $m$.\\
        $C_{m+1}$ & conformal arm set defined in \Cref{def:confidence_set}. \\
        $\eta_{m+1}$ & risk adjustment parameter \eqref{eq:choose-etam}.\\
        $\alpha_{m+1}$ & empirical bound on optimal cover ($\alpha_{m+1}=\frac{3K}{\eta_{m+1}}$).\\
        $p_{m+1}$ & action selection kernel corresponding to epoch $m+1$ defined in \eqref{eq:define_pm_continuous}. \\
        $U_{m+1}$ & $:=20\sqrt{\alpha_m\xi_{m+1}}$.\\
    \hline
    \end{tabular}
    \caption{Table of notations}
    \label{tab:notation}
\end{table}

\section{Bounding Cumulative Regret}
\label{sec:bounding-cumulative-regret}

This section derives the cumulative regret bounds for $\omega$-RAPR. We start with analyzing the output of oracles described in Assumptions \ref{ass:regression-oracle} and \ref{ass:policy-evaluation-oracles}. Note that we do not make the ``realizability" assumption in this work -- i.e., we do not assume that $f^*$ lies in $\F$. Hence, as in \Cref{ass:regression-oracle}, the expected squared error of our estimated models need not go to zero (even with infinite data) and may contain an unknown non-zero irreducible error term ($B$) that captures the bias of the model class $\F$. It is useful to split our analysis into two regimes to handle this unknown term $B$, similar to the approach in \cite{krishnamurthy2021adapting}. In particular, we separately analyze oracle outputs for epochs before and after a so-called ``safe epoch". Where we define the safe epoch $\msafe$ as the epoch where the variance of estimating from the model class $\F$ ($\xi_m$) is dominated by the bias of estimating from the class $\F$ ($B$). That is, $\msafe:=\argmax\{m\geq 1|\xi_{m+1}\geq 2B\}$.

\subsection{High Probability Events}
\label{sec:high-probability-events}

We start with defining high-probability events under which our key theoretical guarantees hold. The first high probability event characterizes the accuracy of estimated reward models and the policy evaluation estimators, the tail bound of which can be obtained by taking a union bound of each epoch-specific event that happens with probability $1-\frac{\delta}{4m^2}$ under assumptions in \Cref{sec:oracle-assumptions}.

\begin{lemma}
\label{lemma:event-pure-exploration}
Suppose \Cref{ass:regression-oracle} and \Cref{ass:policy-evaluation-oracles} hold. The following event holds with probability $1-\delta/2$,
\begin{equation}
    \label{eq:w-event-Pure-Exp}
    \begin{aligned}
      \eventModelOracle := \Bigg\{ &\forall m , \forall \pi\in\Pi\cup\{p_{m+1}\}, \forall f\in \{\hatf_1,\dots,\hatf_{m+1}\}, \\ 
      &\E_{x\sim D_{\Xscript}}\E_{a\sim p_m(\cdot|x)}[ (\hatf_{m+1}(x, a) - f^*(x,a))^2 ] \leq B + \xi_{m+1}/2\\
    & |\hatR_{m+1,f}(\pi)-R_{f}(\pi)| \leq \sqrt{\xi_{m+1}},\\
    & |\hatR_{m+1}(\pi)-R(\pi)| \leq \sqrt{V(p_m,\pi)\xi_{m+1}} + \xi_{m+1}/(\min_{(x,a)\in\Xscript\times\calA}p_m(a|x)). \Big) \Bigg\}.
    \end{aligned}
\end{equation}
Where $\xi_{m+1}=2\xi((\tau_m-\tau_{m-1})/3,\delta/(16m^3))$.\footnote{Our epoch schedules will always be increasing in epoch length. Under such conditions, we have $\xi_m$ is non-increasing in $m$.} %
\end{lemma}

The second high probability event characterizes the measure of conformal arm sets, which  directly follows from Hoeffding's inequality and union bound. 
\begin{lemma}
\label{lemma:event-risk-estimation}
The following event holds with probability $1-\delta/2$,
\begin{equation}
    \label{eq:event-risk-estimation}
    \begin{aligned}
      \eventRiskEstimation := \Bigg\{ &\forall m, \forall \eta\in \Big\{\frac{|S_{m,2}|}{n}|n \in [|S_{m,2}|] \Big\}, \\ 
      & \bigg| \E_{x\sim D_{\Xscript}}\bigg[\mu\bigg( \bar{C}_{m+1}\bigg(x, \frac{\propThreshold}{\eta}\bigg)\bigg)\bigg] - \frac{1}{|S_{m,2}|}\sum_{t\in S_{m,2}}\mu\bigg(\bar{C}_{m+1}\bigg(x_t,\frac{\propThreshold}{\eta}\bigg)\bigg) \bigg|\\ 
      & \leq \sqrt{\frac{K^2\ln(8|S_{m,2}|m^2/\delta)}{2|S_{m,2}|}} \Bigg\}.
    \end{aligned}
\end{equation}
\end{lemma}

Together both $\eventModelOracle$ and $\eventRiskEstimation$ hold with probability $1-\delta$. The rest of our analysis works under these events.

\subsection{Analyzing the Cover} \label{sec:analyzing-the-cover}

In this sub-section, we upper bound the cover ($V(p_m,\cdot)$) for the action selection kernel used in epoch $m$.\footnote{Recall that $V(p,q):=\E_{x\sim D_{\Xscript}, a\sim q(\cdot|x)}[q(a|x)/p(a|x)]$.} To upper bound $V(p_m,q)$, we first lower bound $p_m(\cdot|\cdot)$. Recall that in \Cref{sec:expanded-notation}, we define $U_{m}=20\sqrt{\alpha_{m-1}\xi_{m}}$ for any epoch $m$. Starting from here, our lemmas and proofs will use $U_{m}$ and $20\sqrt{\alpha_{m-1}\xi_{m}}$ interchangeably.

\begin{lemma}
\label{lemma:lower_bound_pm}
    For any epoch $m$, we have \eqref{eq:pm_lower_bound} holds.
    \begin{equation}
        \label{eq:pm_lower_bound}
        \begin{aligned}
        p_m(a|x)&\geq \begin{cases}
			\frac{1-\propThreshold}{\mu(C_m(x, \propThreshold, \eta_m))} + \frac{\propThreshold}{\mu(\A)}, & \text{if $a\in C_m(x, \propThreshold, \eta_m)$}\\
            \frac{\eta_{m}}{\mu(\calA)} \min_{\barm\in[m]}  \frac{2\barm^2U_{\barm}}{\hatf_{\barm}(x,\pi_{\hatf_{\barm}}(x))-\hatf_{\barm}(x,a)}, & \text{if $a\notin C_m(x, \propThreshold, \eta_m)$}
		 \end{cases}\\
            &\geq \begin{cases}
			\frac{1}{\mu(\A)}, & \text{if $a\in C_m(x, \propThreshold, \eta_m)$}\\
            \frac{\eta_{m}\min_{\barm\in[m]}  2\barm^2U_{\barm}}{\mu(\calA)}, & \text{if $a\notin C_m(x, \propThreshold, \eta_m)$}
		 \end{cases}
        \end{aligned}
    \end{equation}
\end{lemma}
\begin{proof}
Recall that $p_{m}$ given by \eqref{eq:define_pm_continuous_2}.
\begin{equation}
\label{eq:define_pm_continuous_2}
    p_{m}(a|x) = \frac{(1-\beta_{\max})I[a\in C_{m}(x,\propThreshold,\eta_{m})]}{\mu\big(C_{m}(x,\beta_{\max},\eta_{m})\big)} + \int_0^{\propThreshold}\frac{I[a\in C_{m}(x,\beta,\eta_{m}) ]}{\mu\big(C_{m}(x,\beta,\eta_{m})\big)}\mbox{d}\beta.
\end{equation}
We divide our analysis into two cases based on whether $a$ lies in $C_m(x, \propThreshold, \eta_m)$, and lower bound $p_m(a|x)$ in each case.

\textbf{Case 1 ($a\in C_m(x, \propThreshold, \eta_m)$).} Note that $C_m(x, \propThreshold, \eta_m)\subseteq C_m(x, \beta, \eta_m) \subseteq \A$ for all $\beta\in[0,\propThreshold]$. Hence, $a\in C_m(x, \beta, \eta_m)$ and $\mu(C_m(x, \beta, \eta_m))\leq \mu(\A)$ for all $\beta\in[0,\propThreshold]$. Therefore, in this case, $p_m(a|x)\geq \frac{(1-\beta_{\max})}{\mu\big(C_{m}(x,\beta_{\max},\eta_{m})\big)}+\frac{\propThreshold}{\mu(\A)}\geq \frac{1}{\mu(\A)}$.

\textbf{Case 2 ($a\notin C_m(x, \propThreshold, \eta_m)$).} For this case, the proof follows from \eqref{eq:lower_bound_term2_in_pm}.

\begin{equation}
\label{eq:lower_bound_term2_in_pm}
    \begin{aligned}
        p_m(a|x) \geq & \int_0^{\propThreshold}\frac{I[a\in C_{m}(x,\beta,\eta_{m}) ]}{\mu\big(C_{m}(x,\beta,\eta_{m})\big)}\mbox{d}\beta\\
       \stackrel{(i)}{\geq}  & \frac{1}{\mu(\calA)} \int_0^{\propThreshold}I[a\in C_{m}(x,\beta,\eta_{m}) ]\mbox{d}\beta\\
       \stackrel{(ii)}{\geq}  & \frac{I(a\notin C_{m}(x,\propThreshold,\eta_{m}))}{\mu(\calA)} \int_0^{\propThreshold}I[a\in C_{m}(x,\beta,\eta_{m}) ]\mbox{d}\beta\\
       \stackrel{(iii)}{=}  & \frac{I(a\notin C_{m}(x,\propThreshold,\eta_{m}))}{\mu(\calA)} \int_0^{1}I[a\in C_{m}(x,\beta,\eta_{m}) ]\mbox{d}\beta\\
      \stackrel{(iv)}{=} & \frac{I(a\notin C_{m}(x,\propThreshold,\eta_{m}))}{\mu(\calA)} \int_0^{1} \prod_{\barm\in[m]} I\big[\hatf_{\barm}(x,\pi_{\hatf_{\barm}}(x))-\hatf_{\barm}(x,a)\leq \frac{2\barm^2\eta_{m}U_{\barm}}{\beta}\big]\mbox{d}\beta\\
      = & \frac{I(a\notin C_{m}(x,\propThreshold,\eta_{m}))}{\mu(\calA)} \int_0^{1} \prod_{\barm\in[m]} I\big[\beta\leq \frac{2\barm^2\eta_{m}U_{\barm}}{\hatf_{\barm}(x,\pi_{\hatf_{\barm}}(x))-\hatf_{\barm}(x,a)}\big]\mbox{d}\beta\\
      = & \frac{\eta_{m}I(a\notin C_{m}(x,\propThreshold,\eta_{m}))}{\mu(\calA)} \min_{\barm\in[m]}  \frac{2\barm^2U_{\barm}}{\hatf_{\barm}(x,\pi_{\hatf_{\barm}}(x))-\hatf_{\barm}(x,a)}\\
      \stackrel{(v)}{\geq} & \frac{\eta_{m}I(a\notin C_{m}(x,\propThreshold,\eta_{m}))}{\mu(\calA)} \min_{\barm\in[m]} 2\barm^2U_{\barm}
    \end{aligned}
\end{equation}
where (i) is because the measure of the conformal set $C_m$ can be no larger than the measure of the action space $\calA$; (ii) follows from $I(a\notin C_{m}(x,\propThreshold,\eta_{m}))\leq 1$; (iii) follows from the fact that if $a\notin C_{m}(x,\propThreshold,\eta_{m})$ then $a\notin C_{m}(x,\beta,\eta_{m})$ for all $\beta\geq \propThreshold$; (iv) follows from \Cref{def:confidence_set}; and (v) follows from .\footnote{Note that if $a\in\pi_{\hatf_m}(x)$ then $I(a\notin C_{m}(x,\propThreshold,\eta_{m})) = 0$.}
\end{proof}

Using \Cref{lemma:lower_bound_pm}, we get an upper bound on $V(p_m,q)$ in terms of $\E[\mu(C_m(x, \propThreshold, \eta_m))]$, $K/\eta_m$, and expected regret with respect to the models $\hatf_1,\dots,\hatf_m$. 

\begin{restatable}{lemma}{lemBoundVCont}
\label{lemma:bound_v_continuous}
For any epoch $m$ and any action selection kernel $q\in\calP$, we have \eqref{eq:v_upper_bound_1} holds.
\begin{equation}
    \label{eq:v_upper_bound_1}
    V(p_m, q)\leq \frac{\E[\mu(C_m(x, \propThreshold, \eta_m))]}{1-\propThreshold} + \frac{K}{\eta_m}\sum_{\barm\in[m]}\frac{\Reg_{\hatf_{\barm}}(q)}{2\barm^2U_{\barm}}.
\end{equation}
\end{restatable}

\begin{proof}
From \Cref{lemma:lower_bound_pm}, we have \eqref{eq:pm_lowerbounds} holds. \footnote{Note that we require $f(x,\pi_{f}(x))\geq f(x,a)$, for all $x\in\Xscript$, $a\in\A$, and $f\in\F$. }
\begin{equation}
\label{eq:pm_lowerbounds}
    \begin{aligned}
        & \frac{I(a\in C_m(x, \propThreshold, \eta_m))}{p_m(a|x)}\leq  \frac{\mu\big(C_{m}(x,\propThreshold,\eta_{m})\big)}{1-\propThreshold},\\
        \mbox{and}\quad & \frac{I(a\notin C_m(x, \propThreshold, \eta_m))}{p_m(a|x)}\leq \frac{\mu(\calA)}{\eta_m} \max_{\barm\in[m]}  \frac{\hatf_{\barm}(x,\pi_{\hatf_{\barm}}(x))-\hatf_{\barm}(x,a)}{2\barm^2U_{\barm}}.
    \end{aligned}
\end{equation}
We now bound the cover $V(p_m,q)$ as follows,
\begin{equation}
\begin{aligned}
    &V(p_m,q)=\E_{x\sim D_{\Xscript}, a\sim q(\cdot|x)}\bigg[\frac{q(a|x)}{p_m(a|x)}\bigg]\\ %
    \stackrel{(i)}{\leq} & \E_{x\sim D_{\Xscript}, a\sim q(\cdot|x)}\bigg[\frac{1}{p_m(a|x)}\bigg]\\
    = & \E_{x\sim D_{\Xscript}, a\sim q(\cdot|x)}\bigg[\frac{I[a\in C_m(x,\propThreshold, \eta_m)] + I[a\notin C_m(x,\propThreshold, \eta_m)]}{p_m(a|x)}\bigg]\\
    \stackrel{(ii)}{\leq} & \E_{x\sim D_{\Xscript}, a\sim q(\cdot|x)}\bigg[\frac{\mu\big(C_{m}(x,\propThreshold,\eta_{m})\big)}{1-\propThreshold}\bigg] + \E_{x\sim D_{\Xscript}, a\sim q(\cdot|x)}\bigg[\frac{\mu(\calA)}{\eta_m} \max_{\barm\in[m]}  \frac{\hatf_{\barm}(x,\pi_{\hatf_{\barm}}(x))-\hatf_{\barm}(x,a)}{2\barm^2U_{\barm}}\bigg]\\
    \leq  & \frac{\E[\mu\big(C_m(x, \propThreshold, \eta_m)\big)]}{1-\propThreshold} + \frac{\mu(\calA)}{\eta_m} \sum_{\barm\in[m]}  \frac{\Reg_{\hatf_{\barm}}(q)}{2\barm^2U_{\barm}}. 
\end{aligned}
\end{equation}
Here (i) follows from the fact that $q\in\calP$ and (ii) follows from \eqref{eq:pm_lowerbounds}.
\end{proof}

Having bounded the cover for the kernel $p_m$ in terms of $\E[\mu(C_m(x, \propThreshold, \eta_m))]$ and $K/\eta_m$. We now bound these terms with $\alpha_m$. 
\begin{restatable}{lemma}{lemConditionAlpham}
\label{lemma:proving-condition-alpham}
Suppose $\eventRiskEstimation$ holds. Then for any epoch $m$, we have \eqref{eq:condition_alpham} holds.
\begin{equation}
\label{eq:condition_alpham}
\begin{aligned}
   \frac{\E[\mu( C_{m}(x, \propThreshold, \eta_{m}))]}{1-\propThreshold} + \frac{K}{\eta_{m}}\leq \alpha_{m} \leq \frac{3K}{\eta_{m}}.
\end{aligned}
\end{equation}
\end{restatable}

\begin{proof}
Since $\mu( C_{m}(x, \propThreshold, \eta_{m}))\leq K$ and $\propThreshold = 1/2$, the bound trivially holds if $\eta_m=1$.  Suppose $\eta_m>1$. Note that $\eta_{\barm}$ is non-decreasing in $\barm$ by construction. Let $m'$ be the smallest epoch index such that $\eta_{m'}=\eta_m$. We now have the following holds.
\begin{equation}
\label{eq:implication-of-eventRE-3}
    \begin{aligned}
        &\frac{\E[\mu( C_{m}(x, \propThreshold, \eta_{m}))]}{1-\propThreshold} + \frac{K}{\eta_{m}}\\ 
        &\stackrel{(i)}{\leq} \frac{\min\{1+\E[\mu( \bar{C}_{m}(x, \propThreshold, \eta_{m}))],K\}}{1-\propThreshold} + \frac{K}{\eta_{m}} \\
        &\stackrel{(ii)}{\leq} \frac{\min\{1+\E[\mu( \bar{C}_{m'}(x, \propThreshold, \eta_{m}))],K\}}{1-\propThreshold} + \frac{K}{\eta_{m}} \\
        &\stackrel{(iii)}{\leq} \frac{\lambda_{m'}(\eta_m)}{1-\propThreshold}+\frac{K}{\eta_{m}}\\
        &\stackrel{(iv)}{\leq} 2\lambda_{m'}(\eta_m)+\frac{K}{\eta_{m}}\\
        &\stackrel{(v)}{\leq} \frac{3K}{\eta_{m}} 
    \end{aligned}
\end{equation}
Here (i) follows from the definition of CASs. (ii) follows from $\bar{C}_{m}\subseteq\bar{C}_{m'}$ which follows from the fact that $\bar{C}_m=\cap_{\barm\in[m]}\tilde{C}_{\barm}$ and $m'\leq m$. (iii) follows from $\eventRiskEstimation$ and the definition of $\lambda_{m'}$ in \eqref{eq:choose-etam}. (iv) follows from the fact that $\propThreshold= 1/2$. (v) follows from \eqref{eq:choose-etam} and $\eta_{m'-1}<\eta_{m'}=\eta_m$ -- note that $\eta_{m'-1}\neq \eta_{m'}$ gives us that $\eta_{m'}$ was set using the constrained maximization procedure in \eqref{eq:choose-etam}, hence the constraint $\lambda_{m'}(\eta)\leq K/\eta$ is satisfied at $\eta=\eta_{m'}=\eta_m$. 
\end{proof}

\subsection{Evaluation Guarantees Under Safe Epoch}\label{sec:evaluation-guarantees-under-safe-epoch}

This sub-section provides guarantees on how accurate $R_{\hatf_{m+1}}$ is at evaluating policies when we are within the safe epoch. 

\begin{restatable}{lemma}{lemBoundSafeError}
\label{lemma:bound_reward_diff_continuous}
Suppose $\eventModelOracle$ holds. For all epochs $m\in [\msafe]$, for any $q\in \calP$, we have,
\begin{equation}
    |R_{\hat{f}_{m+1}}(q)-R(q)|\leq \sqrt{V(p_m, q)\xi_{m+1}}.
\end{equation}
\end{restatable}
\begin{proof}
Consider any epoch $m\in[\msafe]$ and policy $q\in\calP$. We then have,
\begin{equation}
\begin{aligned}
  |R_{\hat{f}_{m+1}}(q)-R(q)| &=| \E_{x\sim D_{\calX}, a\sim q}[\hatf_{m+1}(x,a)-f^*(x,a)]|\\
  &\stackrel{(i)}{=}\bigg| \E_{x\sim D_{\calX}, a\sim p_m}\Big[\frac{q(a|x)}{p_m(a|x)}\big(\hatf_{m+1}(x,a)-f^*(x,a)\big)\Big]\bigg|\\
  &\leq \E_{x\sim D_{\calX}, a\sim p_m}\Big[\frac{q(a|x)}{p_m(a|x)}\big|\hatf_{m+1}(x,a)-f^*(x,a)\big|\Big]\\
   &= \E_{x\sim D_{\calX}, a\sim p_m}\Big[\sqrt{\Big(\frac{q(a|x)}{p_m(a|x)}\Big)^2\big|\hatf_{m+1}(x,a)-f^*(x,a)\big|^2}\Big]\\
  &\stackrel{(ii)}{\leq}\sqrt{\E_{x\sim D_{\calX}, a\sim p_m}\Big[\Big(
  \frac{q(a|x)}{p_m(a|x)}\Big)^2\Big]}\sqrt{\E_{x\sim D_{\calX}, a\sim p_m}\big[ (\hatf_{m+1}(x,a)-f^*(x,a))^2\big] }\\
  & \stackrel{(iii)}{=}\sqrt{\E_{x\sim D_{\calX}, a\sim q}\Big[
  \frac{q(a|x)}{p_m(a|x)}\Big]} \sqrt{\E_{x\sim D_{\calX}, a\sim p_m}\big[ (\hatf_{m+1}(x,a)-f^*(x,a))^2\big] }\\
  & \stackrel{(iv)}{\leq} \sqrt{V(p_m, q)\xi_{m+1}} ,
\end{aligned}
\end{equation}
where (i) and (iii) follow from change of measure arguments,  (ii) follows from Cauchy-Schwartz inequality, and (iv) follows from $\eventModelOracle$.
\end{proof}

By combining the guarantees of \Cref{lemma:bound_v_continuous} and \Cref{lemma:bound_reward_diff_continuous}, we get \Cref{lemma:bound_reward_diff_continuous_refined}.

\begin{restatable}{lemma}{lemRefineSafeErrorBound}
\label{lemma:bound_reward_diff_continuous_refined}
Suppose $\eventModelOracle$ and $\eventRiskEstimation$ hold. Then for any action selection kernel $q\in\calP$, we have:
\begin{equation}
  |R_{\hat{f}_{m+1}}(q)-R(q)| \leq \sqrt{\alpha_m\xi_{m+1}} + \frac{1}{2}\sqrt{\alpha_m\xi_{m+1}}\sum_{\barm\in[m]}\frac{\Reg_{\hatf_{\barm}}(q)}{2\barm^2U_{\barm}}.
\end{equation}
\end{restatable}
\begin{proof}
From \Cref{lemma:bound_v_continuous}, we have \eqref{eq:refined-bound-on-V} holds for any $q\in\calP$.
\begin{equation}
\label{eq:refined-bound-on-V}
\begin{aligned}
    &V(p_m, q)\\
    &\leq \frac{\E[\mu( C_m(x, \propThreshold, \eta_m))]}{1-\propThreshold} + \frac{K}{\eta_m}\sum_{\barm\in[m]}\frac{\Reg_{\hatf_{\barm}}(q)}{2\barm^2U_{\barm}}\\
    &\leq \Big(\frac{\E[\mu( C_m(x, \propThreshold, \eta_m))]}{1-\propThreshold} + \frac{K}{\eta_m}\Big) + \Big(\frac{\E[\mu( C_m(x, \propThreshold, \eta_m))]}{1-\propThreshold} + \frac{K}{\eta_m}\Big)\sum_{\barm\in[m]}\frac{\Reg_{\hatf_{\barm}}(q)}{2\barm^2U_{\barm}}\\
     &=  \alpha_m + \alpha_m \sum_{\barm\in[m]}\frac{\Reg_{\hatf_{\barm}}(q)}{2\barm^2U_{\barm}}
\end{aligned}
\end{equation}
Combining \eqref{eq:refined-bound-on-V} with \Cref{lemma:bound_reward_diff_continuous} we have:
\begin{equation}
\begin{aligned}
    & |R_{\hat{f}_{m+1}}(q)-R(q)| \\
    \leq & \sqrt{V(p_m,q)\xi_{m+1}}\\
    \stackrel{(i)}{\leq} & \frac{1}{2}\sqrt{\alpha_m\xi_{m+1}} + \frac{1}{2}\sqrt{\frac{\xi_{m+1}}{\alpha_m}}V(p_m,q)\\
    \stackrel{(i)}{\leq} & \sqrt{\alpha_m\xi_{m+1}} + \frac{1}{2}\sqrt{\alpha_m\xi_{m+1}}\sum_{\barm\in[m]}\frac{\Reg_{\hatf_{\barm}}(q)}{2\barm^2U_{\barm}}.
\end{aligned}
\end{equation}
Where (i) follows from AM-GM inequality and (ii) follows from \eqref{eq:refined-bound-on-V}.
\end{proof}

\subsection{Testing Safety}\label{sec:testing-safety}

The misspecification test \eqref{eq:misspecification-test} is designed to test if we are within the safe epoch. In principle, it works by comparing the accuracy of  $R_{\hatf_{m+1}}$ (\Cref{lemma:bound_reward_diff_continuous_refined}) and $\hatR_{m+1}$ (\Cref{lemma:bound_IPS_error}). Formally, \Cref{lemma:misspecification-test} shows that the misspecification test in \eqref{eq:misspecification-test} fails only after $\msafe$. Hence, $\msafealg\geq \msafe+1$. \Cref{lemma:test-implication} then describes the implication of \eqref{eq:misspecification-test} continuing to hold. In what follows, we let $\hatReg_{m+1,\hatf_{\barm}}(\pi):=\hatR_{m+1,\hatf_{\barm}}(\pi_{\hatf_{\barm}})-\hatR_{m+1,\hatf_{\barm}}(\pi)$. We start with \Cref{lemma:bound_IPS_error} which provides accuracy guarantees for $\hatR_{m+1}$ in any epoch.

\begin{restatable}{lemma}{lemBoundIPSError}
\label{lemma:bound_IPS_error}
Suppose $\eventModelOracle$ and $\eventRiskEstimation$ hold. Then for any epoch $m$ and all $\pi\in\Pi\cup\{p_{m+1}\}$, we have,
\begin{equation}
  |\hatR_{m+1}(\pi)-R(\pi)| \leq \sqrt{\alpha_m\xi_{m+1}} + \frac{1}{2}\sqrt{\alpha_m\xi_{m+1}}\sum_{\barm\in[m]}\frac{\Reg_{\hatf_{\barm}}(\pi)}{2\barm^2U_{\barm}} + \frac{K\xi_{m+1}}{\eta_m \min_{\barm\in[m]}U_{\barm}}.
\end{equation}
\end{restatable}
\begin{proof}
From \Cref{lemma:lower_bound_pm}, we have \eqref{eq:worst-case-lower-bound-on-pm} holds, which provides a worst-case lower bound on $p_m$.
\begin{equation}
\label{eq:worst-case-lower-bound-on-pm}
    \begin{aligned}
        \min_{(x,a)\in\Xscript\times\calA}p_m(a|x)\geq \min\bigg(\frac{1}{K},\frac{\eta_m \min_{\barm\in[m]} (2\barm^2)U_{\barm}}{K}\bigg) \geq \frac{\eta_m \min_{\barm\in[m]} U_{\barm}}{K}
    \end{aligned}
\end{equation}
Where the last inequality follows from $U_m\leq 1$. Now from $\eventModelOracle$, we have,
\begin{equation}
\begin{aligned}
    & |\hatR_{m+1}(\pi)-R(\pi)| \stackrel{(\eventModelOracle)}{\leq} \sqrt{V(p_m,\pi)\xi_{m+1}} + \frac{K\xi_{m+1}}{\eta_m \min_{\barm\in[m]} U_{\barm}}\\
    \stackrel{(i)}{\leq} & \frac{1}{2}\sqrt{\alpha_m\xi_{m+1}} + \frac{1}{2}\sqrt{\frac{\xi_{m+1}}{\alpha_m}}V(p_m,\pi) + \frac{K\xi_{m+1}}{\eta_m \min_{\barm\in[m]} U_{\barm}}\\
    \stackrel{(ii)}{\leq} & \sqrt{\alpha_m\xi_{m+1}} + \frac{1}{2}\sqrt{\alpha_m\xi_{m+1}}\sum_{\barm\in[m]}\frac{\Reg_{\hatf_{\barm}}(\pi)}{2\barm^2U_{\barm}} + \frac{K\xi_{m+1}}{\eta_m \min_{\barm\in[m]} U_{\barm}}.
\end{aligned}
\end{equation}
Where (i) follows from AM-GM inequality, and (ii) follows from \eqref{eq:refined-bound-on-V} in the proof of \Cref{lemma:bound_reward_diff_continuous_refined}.
\end{proof}

Lemmas \ref{lemma:triangle-inequality} and \ref{lemma:bound-model-regret} provide useful inequalities that help construct the misspecification test \eqref{eq:misspecification-test}.

\begin{restatable}{lemma}{lemTriangleIneq}
\label{lemma:triangle-inequality}
For any epoch $m$, policy $\pi\in \calP$, and model $f\in \{\hatf_1, \hatf_2, \dots,\hatf_{m+1}\}$. We have,
\begin{align*}
  &||R_{f}(\pi)-R(\pi)| - |\hatR_{m+1,f}(\pi)-\hatR_{m+1}(\pi)| |\\ 
  &\leq |\hatR_{m+1}(\pi)-R(\pi)|+|R_{f}(\pi)-\hatR_{m+1,f}(\pi)|.
\end{align*}
\end{restatable}
\begin{proof}
The proof follows from noting that,
\begin{equation}
\begin{aligned}
  &|R_{f}(\pi)-R(\pi)| \\ 
  &= |\hatR_{m+1}(\pi)-R(\pi)+R_{f}(\pi)-\hatR_{m+1,f}(\pi)+\hatR_{m+1,f}(\pi)-\hatR_{m+1}(\pi)|\\
  &\leq |\hatR_{m+1}(\pi)-R(\pi)|+|R_{f}(\pi)-\hatR_{m+1,f}(\pi)|+|\hatR_{m+1,f}(\pi)-\hatR_{m+1}(\pi)|.
\end{aligned}
\end{equation}
and from noting that,
\begin{equation}
\begin{aligned}
  &|\hatR_{m+1}(\pi)-\hatR_{m+1,f}(\pi)| \\ 
  &= |\hatR_{m+1}(\pi)-R(\pi)+R_{f}(\pi)-\hatR_{m+1,f}(\pi)+R(\pi)-R_{f}(\pi)|\\
  &\leq |\hatR_{m+1}(\pi)-R(\pi)|+|R_{f}(\pi)-\hatR_{m+1,f}(\pi)|+|R_{f}(\pi)-R(\pi)|.
\end{aligned}
\end{equation}
\end{proof}

\begin{restatable}{lemma}{lemBoundModelRegret}
\label{lemma:bound-model-regret}
Suppose $\eventModelOracle$ and $\eventRiskEstimation$ hold. Then for any epoch $m$, any model $f\in\{\hatf_i|i\in[m+1] \}$, and any policy $\pi\in\Pi\cup\{p_{m+1}\}$, we have,
\begin{align*}
    |\Reg_{f}(\pi)-\hatReg_{m+1,f}(\pi)| \leq 2\sqrt{\xi_{m+1}}
\end{align*}
\end{restatable}
\begin{proof}
Follows from triangle inequality and $\eventModelOracle$,
\begin{equation}
\begin{aligned}
    &|\Reg_{f}(\pi)-\hatReg_{m+1,f}(\pi)| \\
    & \leq |R_{f}(\pi_{f})-\hatR_{m+1,f}(\pi_{f})| + |R_{f}(\pi)-\hatR_{m+1,f}(\pi)| \leq 2\sqrt{\xi_{m+1}}
\end{aligned}
\end{equation}
\end{proof}

As discussed earlier, \Cref{lemma:misspecification-test} shows that the misspecification test in \eqref{eq:misspecification-test} fails only after $\msafe$. Hence, $\msafealg\geq \msafe+1$.

\begin{restatable}{lemma}{lemMisspecificationTest}
\label{lemma:misspecification-test}
Suppose $\eventModelOracle$ and $\eventRiskEstimation$ hold. Now for any epoch $m\in[\msafe]$ we have that,
\begin{equation}
    \begin{aligned}
             &\max_{\pi\in\Pi\cup\{p_{m+1}\}}|\hatR_{m+1,\hatf_{m+1}}(\pi)-\hatR_{m+1}(\pi)|  - \sqrt{\alpha_m\xi_{m+1}}\sum_{\barm\in[m]}\frac{\hatReg_{m+1,\hatf_{\barm}}(\pi)}{2\barm^2U_{\barm}}\\
            & \leq 2.05\sqrt{\alpha_m\xi_{m+1}} + 1.1\sqrt{\xi_{m+1}}.
    \end{aligned}
\end{equation}
\end{restatable}
\begin{proof}
For any epoch $m\in[\msafe]$ and for any $\pi\in\Pi\cup\{p_{m+1}\}$, we have,
\begin{equation}
\begin{aligned}
    &|\hatR_{\hatf_{m+1}}(\pi)-\hatR_{m+1}(\pi)|\\
    & \stackrel{(i)}{\leq} |\hatR_{m+1}(\pi)-R(\pi)| + |R_{\hat{f}_{m+1}}(\pi)-R(\pi)| +|R_{\hat{f}_{m+1}}(\pi)-\hatR_{\hatf_{m+1}}(\pi)| \\
    & \stackrel{(ii)}{\leq} 2\sqrt{\alpha_m\xi_{m+1}} + \sqrt{\alpha_m\xi_{m+1}}\sum_{\barm\in[m]}\frac{\Reg_{\hatf_{\barm}}(\pi)}{2\barm^2U_{\barm}}  + \frac{K\xi_{m+1}}{\eta_m \min_{\barm\in[m]} U_{\barm}} + \sqrt{\xi_{m+1}}\\
    & \stackrel{(iii)}{\leq} 2\sqrt{\alpha_m\xi_{m+1}} + \sqrt{\alpha_m\xi_{m+1}}\sum_{\barm\in[m]}\frac{\Reg_{\hatf_{\barm}}(\pi)}{2\barm^2U_{\barm}}  + \frac{\alpha_m\xi_{m+1}}{U_{m}}+ \sqrt{\xi_{m+1}}\\
    & \stackrel{(iv)}{\leq} 2\sqrt{\alpha_m\xi_{m+1}} + \sqrt{\alpha_m\xi_{m+1}}\sum_{\barm\in[m]}\frac{\hatReg_{m+1,\hatf_{\barm}}(\pi)}{2\barm^2U_{\barm}} + \frac{2\sqrt{\alpha_m}\xi_{m+1}}{U_m} + \frac{\alpha_m\xi_{m+1}}{U_{m}} + \sqrt{\xi_{m+1}}\\
    & \stackrel{(v)}{\leq} 2.05\sqrt{\alpha_m\xi_{m+1}} + \sqrt{\alpha_m\xi_{m+1}}\sum_{\barm\in[m]}\frac{\hatReg_{m+1,\hatf_{\barm}}(\pi)}{2\barm^2U_{\barm}} + 1.1\sqrt{\xi_{m+1}}
\end{aligned}
\end{equation}
Where (i) follows from \Cref{lemma:triangle-inequality}. (ii) follows from \Cref{lemma:bound_reward_diff_continuous_refined}, \Cref{lemma:bound_IPS_error}, and $\eventModelOracle$. (iii) follows from \Cref{eq:condition_alpham} and the fact that $U_{\barm}$ is non-increasing in $\barm$ (giving us $\min_{\barm\in[m]} U_{\barm}=U_m$). (iv) follows from \Cref{lemma:bound-model-regret}, the fact that $U_{\barm}$ is non-increasing in $\barm$, and the fact that $\sum_{\barm=1}^{\infty}1/(2\barm^2)\leq 1$. (v) follows from $U_m=20\sqrt{\alpha_{m-1}\xi_m}$, $\alpha_{m}\leq \alpha_{m-1}$, and $\xi_{m+1}\leq \xi_{m}$.
\end{proof}

\Cref{lemma:test-implication} now describes the implication of \eqref{eq:misspecification-test} continuing to hold.
\begin{restatable}{lemma}{lemTestImplication}
\label{lemma:test-implication}
Suppose $\eventModelOracle$ and $\eventRiskEstimation$ hold. Now for any epoch $m\in[\msafealg-1]$ and any policy $\pi\in\Pi\cup\{p_{m+1}\}$, we then have that,
\begin{equation}
    |R_{\hatf_{m+1}}(\pi)-R(\pi)|\leq  2.2\sqrt{\xi_{m+1}} + 3.1\sqrt{\alpha_m\xi_{m+1}} + \frac{3}{2}\sqrt{\alpha_m\xi_{m+1}}\sum_{\barm\in[m]}\frac{\Reg_{\hatf_{\barm}}(\pi)}{2\barm^2U_{\barm}}.
\end{equation}
\end{restatable}
\begin{proof}
\begin{equation}
\begin{aligned}
    &|R_{\hatf_{m+1}}(\pi)-R(\pi)|\\
    & \stackrel{(i)}{\leq} |\hatR_{m+1}(\pi)-R(\pi)| + |\hatR_{\hatf_{m+1}}(\pi)-\hatR_{m+1}(\pi)| +|R_{\hat{f}_{m+1}}(\pi)-\hatR_{\hatf_{m+1}}(\pi)| \\
    & \stackrel{(ii)}{\leq} 3.1\sqrt{\alpha_m\xi_{m+1}} + \frac{3}{2}\sqrt{\alpha_m\xi_{m+1}}\sum_{\barm\in[m]}\frac{\Reg_{\hatf_{\barm}}(\pi)}{2\barm^2U_{\barm}} + 2.2\sqrt{\xi_{m+1}}
\end{aligned}
\end{equation}
Where (i) follows from \Cref{lemma:triangle-inequality}. And (ii) follows from \Cref{eq:misspecification-test}, \Cref{lemma:bound_IPS_error}, \Cref{lemma:bound-model-regret}, and $\eventModelOracle$.
\end{proof}

\subsection{Inductive Argument}\label{sec:inductive-argument}

This sub-section leverages the guarantee of \Cref{lemma:test-implication} and applies it inductively to derive \Cref{lemma:regret-under-algsafe-epoch}. This lemma bounds $\Reg_{\Pi}(\pi)$ in terms of $\Reg_{\hatf_{m+1}}(\pi)$ and vice-versa for any policy $\pi\in\Pi$. The proof of \Cref{lemma:regret-under-algsafe-epoch}, relies on the following helpful lemma.

\begin{restatable}{lemma}{lemInductiveArg}
\label{lemma:inductive-argument}
Consider any class of policies $\Pi' \supseteq \Pi$ and consider any fixed constants $l_1,l_2,l_3,C'>0$. At any epoch $m$, suppose the policy evaluation guarantee of \Cref{eq:evaluation-guarantee} holds.
\begin{equation}
\label{eq:evaluation-guarantee}
    \begin{aligned}
             &\forall \pi\in \Pi',\;\;|R_{\hatf_{m+1}}(\pi)-R(\pi)| \leq l_1\sqrt{\xi_{m+1}} + l_2\sqrt{\alpha_m\xi_{m+1}} + \frac{l_3}{C'}\sum_{\barm\in[m]}\frac{z_{\barm,m+1}\Reg_{\hatf_{\barm}}(\pi)}{2\barm^2} 
    \end{aligned}
\end{equation}
Now consider fixed constants $C_1,C_2\geq 0$. As an inductive hypothesis, suppose \Cref{eq:inductive-hypothesis-1} holds.
\begin{equation}
\label{eq:inductive-hypothesis-1}
    \begin{aligned}
             \forall \barm\in[m],\forall \pi\in \Pi',\;\; &\Reg_{\hatf_{\barm}}(\pi) \leq \frac{4}{3}\Reg_{\Pi}(\pi) + C_1\sqrt{\xi_{\barm}}+ C_2\sqrt{\alpha_{\barm-1}\xi_{\barm}}.
    \end{aligned}
\end{equation}
We then have that \Cref{eq:inductive-implication-1} holds.
\begin{equation}
\label{eq:inductive-implication-1}
    \begin{aligned}
             &\forall\pi\in \Pi', \Reg_{\Pi}(\pi) \leq \frac{6}{5}\Reg_{\hatf_{m+1}}(\pi) + \frac{12}{5}\Big(l_1 + \frac{l_3C_1}{C'}\Big)\sqrt{\xi_{m+1}} + \frac{12}{5}\Big(l_2 + \frac{l_3C_2}{C'}\Big)\sqrt{\alpha_m\xi_{m+1}}.
    \end{aligned}
\end{equation}
Now consider $C_3\geq 0$ and further suppose \Cref{eq:inductive-hypothesis-2} holds.
\begin{equation}
\label{eq:inductive-hypothesis-2}
    \begin{aligned}
             \forall \barm\in[m], \Reg_{\hatf_{\barm}}(\pi_{\hatf_{m+1}})\leq  C_3\sqrt{\alpha_{\barm-1}\xi_{\barm}}.
    \end{aligned}
\end{equation}
We then also have that \Cref{eq:inductive-implication-2} holds,
\begin{equation}
\label{eq:inductive-implication-2}
    \begin{aligned}
             &\forall\pi\in\Pi', \Reg_{\hatf_{m+1}}(\pi) \leq \frac{7}{6}\Reg_{\Pi}(\pi) + \Big(2l_1 + \frac{l_3C_1}{C'}\Big)\sqrt{\xi_{m+1}} + \Big(2l_2 + \frac{l_3(C_2+C_3)}{C'}\Big)\sqrt{\alpha_m\xi_{m+1}}.
    \end{aligned}
\end{equation}
Where $C'\geq 8l_3$, $z_{\barm,m+1}:=\sqrt{\frac{\alpha_m\xi_{m+1}}{\alpha_{\barm-1}\xi_{\barm}}}\leq 1$, and $\alpha_m\leq \alpha_{\barm-1}$ for all $\barm\in[m]$.
\end{restatable}
\begin{proof}
Consider any policy $\pi\in \Pi'$. Suppose \eqref{eq:evaluation-guarantee} and \eqref{eq:inductive-hypothesis-1} hold. We first show \eqref{eq:bound_eq_1_upper_bounding_Reg}.
\begin{equation}
\label{eq:bound_eq_1_upper_bounding_Reg}
    \begin{aligned}
    &\Reg_{\Pi}(\pi) - \Reg_{\hatf_{m+1}}(\pi)\\
    &= R(\pi^*) - R(\pi) - R_{\hatf_{m+1}}(\pi_{\hatf_{m+1}}) + R_{\hatf_{m+1}}(\pi)\\
    &\leq R(\pi^*) - R_{\hatf_{m+1}}(\pi^*) + (R_{\hatf_{m+1}}(\pi) - R(\pi) )\\
    &\stackrel{(i)}{\leq} 2l_1\sqrt{\xi_{m+1}} + 2l_2\sqrt{\alpha_m\xi_{m+1}} + \frac{l_3}{C'}\sum_{\barm\in[m]}\frac{z_{\barm,m+1}}{2\barm^2}(\Reg_{\hatf_{\barm}}(\pi)+\Reg_{\hatf_{\barm}}(\pi^*))\\
    &\stackrel{(ii)}{\leq} 2l_1\sqrt{\xi_{m+1}} + 2l_2\sqrt{\alpha_m\xi_{m+1}}\\ 
    &\;\;\;\;\;\;\;\;\;\;\;\;+ \frac{l_3}{C'}\sum_{\barm\in[m]}\frac{z_{\barm,m+1}}{2\barm^2}\bigg(
    \frac{4}{3}\Reg_{\Pi}(\pi)+ 2C_1\sqrt{\xi_{\barm}} + 2C_2\sqrt{\alpha_{\barm-1}\xi_{\barm}} \bigg)\\
    &= 2l_1\sqrt{\xi_{m+1}} + 2l_2\sqrt{\alpha_m\xi_{m+1}}\\ 
    &\;\;\;\;\;\;\;\;\;\;\;\;+ \frac{l_3}{C'}\sum_{\barm\in[m]}\frac{1}{2\barm^2}\bigg(
    \frac{4z_{\barm,m+1}}{3}\Reg_{\Pi}(\pi)+ \frac{2C_1\sqrt{\xi_{m+1}}}{\sqrt{\alpha_{\barm-1}/\alpha_m}} + 2C_2\sqrt{\alpha_{m}\xi_{m+1}} \bigg)\\
    & \stackrel{(iii)}{\leq}\Big(2l_1 + \frac{2l_3C_1}{C'}\Big)\sqrt{\xi_{m+1}} + \Big(2l_2 + \frac{2l_3C_2}{C'}\Big)\sqrt{\alpha_m\xi_{m+1}} + \frac{4}{3}\frac{l_3}{C'}\Reg_{\Pi}(\pi) ,
\end{aligned}
\end{equation}
Where (i) follows from \eqref{eq:evaluation-guarantee}, (ii) follows from \eqref{eq:inductive-hypothesis-1} and from $\Reg_{\Pi}(\pi^*)=0$, and finally (iii) follows from $z_{\barm,m+1}\leq 1$, $\alpha_m\leq \alpha_{\barm-1}$, and $\sum_{\barm\in[m]}1/(2\barm^2)\leq 1$. Now \eqref{eq:bound_eq_1_upper_bounding_Reg} immediately implies \eqref{eq:bound_eq_2_upper_bounding_Reg}.
\begin{equation}
\label{eq:bound_eq_2_upper_bounding_Reg}
    \begin{aligned}
    &\Big(1-\frac{4l_3}{3C'}\Big)\Reg_{\Pi}(\pi) \leq \Reg_{\hatf_{m+1}}(\pi) + \Big(2l_1 + \frac{2l_3C_1}{C'}\Big)\sqrt{\xi_{m+1}} + \Big(2l_2 + \frac{2l_3C_2}{C'}\Big)\sqrt{\alpha_m\xi_{m+1}} \\
    \stackrel{(i)}{\implies} & \Reg_{\Pi}(\pi) \leq \frac{6}{5}\Reg_{\hatf_{m+1}}(\pi) + \frac{12}{5}\Big(l_1 + \frac{l_3C_1}{C'}\Big)\sqrt{\xi_{m+1}} + \frac{12}{5}\Big(l_2 + \frac{l_3C_2}{C'}\Big)\sqrt{\alpha_m\xi_{m+1}} 
\end{aligned}
\end{equation}
Where (i) follows from the fact that $C'\geq 8l_3$. Similar to \eqref{eq:bound_eq_1_upper_bounding_Reg}, we will now show \eqref{eq:bound_eq_1_lower_bounding_Reg}.
\begin{equation}
\label{eq:bound_eq_1_lower_bounding_Reg}
    \begin{aligned}
    &\Reg_{\hatf_{m+1}}(\pi) - \Reg_{\Pi}(\pi)\\
    & = R_{\hatf_{m+1}}(\pi_{\hatf_{m+1}}) - R_{\hatf_{m+1}}(\pi) - (R(\pi^*)-R(\pi))\\
    &\leq \big(R_{\hatf_{m+1}}(\pi_{\hatf_{m+1}}) -R(\pi_{\hatf_{m+1}})\big) +\big( R(\pi)-R_{\hatf_{m+1}}(\pi)\big)\\
    & \stackrel{(i)}{\leq} 2l_1\sqrt{\xi_{m+1}} + 2l_2\sqrt{\alpha_m\xi_{m+1}} + \frac{l_3}{C'}\sum_{\barm\in[m]}\frac{z_{\barm,m+1}}{2\barm^2}(\Reg_{\hatf_{\barm}}(\pi_{\hatf_{m+1}})+\Reg_{\hatf_{\barm}}(\pi))\\
    &\stackrel{(ii)}{\leq} 2l_1\sqrt{\xi_{m+1}} + 2l_2\sqrt{\alpha_m\xi_{m+1}}\\ 
    &\;\;\;\;\;\;\;\;\;\;\;\;+ \frac{l_3}{C'}\sum_{\barm\in[m]}\frac{z_{\barm,m+1}}{2\barm^2}\bigg(
    \frac{4}{3}\Reg_{\Pi}(\pi)+ C_1\sqrt{\xi_{\barm}} + (C_2+C_3)\sqrt{\alpha_{\barm-1}\xi_{\barm}} \bigg)\\
    &= 2l_1\sqrt{\xi_{m+1}} + 2l_2\sqrt{\alpha_m\xi_{m+1}}\\ 
    &\;\;\;\;\;\;\;\;\;\;\;\;+ \frac{l_3}{C'}\sum_{\barm\in[m]}\frac{1}{2\barm^2}\bigg(
    \frac{4z_{\barm,m+1}}{3}\Reg_{\Pi}(\pi)+ \frac{C_1\sqrt{\xi_{m+1}}}{\sqrt{\alpha_{\barm-1}/\alpha_m}} + (C_2+C_3)\sqrt{\alpha_{m}\xi_{m+1}} \bigg)\\
    & \stackrel{(iii)}{\leq} \Big(2l_1 + \frac{l_3C_1}{C'}\Big)\sqrt{\xi_{m+1}} + \Big(2l_2 + \frac{l_3(C_2+C_3)}{C'}\Big)\sqrt{\alpha_m\xi_{m+1}} + \frac{4l_3}{3C'}\Reg_{\Pi}(\pi)
    \end{aligned}
\end{equation}
Where (i) follows from \eqref{eq:evaluation-guarantee}, (ii) follows from \eqref{eq:inductive-hypothesis-1}, \eqref{eq:inductive-hypothesis-2}, and (iii) follows from $z_{\barm,m+1}\leq 1$, $\alpha_m\leq \alpha_{\barm-1}$, and $\sum_{\barm\in[m]}1/(2\barm^2)\leq 1$. Now \eqref{eq:bound_eq_1_lower_bounding_Reg} immediately implies \eqref{eq:bound_eq_2_lower_bounding_Reg}.
\begin{equation}
\label{eq:bound_eq_2_lower_bounding_Reg}
    \begin{aligned}
    &\Reg_{\hatf_{m+1}}(\pi) \leq \Big(1+\frac{4l_3}{3C'}\Big)\Reg_{\Pi}(\pi)  + \Big(2l_1 + \frac{l_3C_1}{C'}\Big)\sqrt{\xi_{m+1}} + \Big(2l_2 + \frac{l_3(C_2+C_3)}{C'}\Big)\sqrt{\alpha_m\xi_{m+1}}\\
    &\stackrel{(i)}{\implies}  \Reg_{\hatf_{m+1}}(\pi) \leq \frac{7}{6}\Reg_{\Pi}(\pi) + \Big(2l_1 + \frac{l_3C_1}{C'}\Big)\sqrt{\xi_{m+1}} + \Big(2l_2 + \frac{l_3(C_2+C_4)}{C'}\Big)\sqrt{\alpha_m\xi_{m+1}}.
\end{aligned}
\end{equation}
Where (i) follows from the fact that $C'\geq 8l_3$.
\end{proof}

\begin{restatable}{lemma}{lemRegret}
\label{lemma:regret-under-algsafe-epoch}
Suppose $\eventModelOracle$ and $\eventRiskEstimation$ hold. Now for any epoch $m\in[\msafealg-1]$, we then have that \eqref{eq:regret-bounds-within-algsafe-epoch} holds.
\begin{equation}
\label{eq:regret-bounds-within-algsafe-epoch}
    \begin{aligned}
             \forall \pi\in\Pi,\; &\Reg_{\Pi}(\pi)\leq \frac{4}{3}\Reg_{\hatf_{m+1}}(\pi) + 6.5\sqrt{\xi_{m+1}} + 12\sqrt{\alpha_{m}\xi_{m+1}},\\
             &\Reg_{\hatf_{m+1}}(\pi)\leq \frac{4}{3}\Reg_{\Pi}(\pi) + 6.5\sqrt{\xi_{m+1}} + 12\sqrt{\alpha_{m}\xi_{m+1}}.
    \end{aligned}
\end{equation}
Moreover when $m\in[\msafe]$, we have \eqref{eq:regret-bounds-within-algsafe-epoch} holds for all policies $\pi\in\calP$. 
\end{restatable}
\begin{proof}
Note that \eqref{eq:regret-bounds-within-algsafe-epoch} trivially holds for $m=0$. We will now use an inductive argument. Consider any epoch $m\in[\msafealg]$. As an inductive hypothesis, let us assume \eqref{eq:inductive-hypothesis-within-algsafe-epoch} holds. (i.e. \eqref{eq:regret-bounds-within-algsafe-epoch} holds for epoch $m-1$.)
\begin{equation}
\label{eq:inductive-hypothesis-within-algsafe-epoch}
    \begin{aligned}
             & \forall \pi\in\Pi, \barm\in[m],\\ 
             &\Reg_{\Pi}(\pi)\leq \frac{4}{3}\Reg_{\hatf_{\barm}}(\pi) + 6.5\sqrt{\xi_{\barm}} + 12\sqrt{\alpha_{\barm-1}\xi_{\barm}},\\
             &\Reg_{\hatf_{\barm}}(\pi)\leq \frac{4}{3}\Reg_{\Pi}(\pi) + 6.5\sqrt{\xi_{\barm}} + 12\sqrt{\alpha_{\barm-1}\xi_{\barm}}.
    \end{aligned}
\end{equation}
Hence from \eqref{eq:inductive-hypothesis-within-algsafe-epoch}, we have \eqref{eq:inductive-hypothesis-1} holds with $C_1=6.5$ and $C_2=12$. Since $m\in[\msafealg]$, from \Cref{lemma:test-implication}, we have \eqref{eq:evaluation-guarantee-under-algsafe-epoch} holds.
\begin{equation}
\label{eq:evaluation-guarantee-under-algsafe-epoch}
    \begin{aligned}
        &\forall \pi\in\Pi\cup\{p_{m+1}\}, \\
        &|R_{\hatf_{m+1}}(\pi)-R(\pi)| \leq \frac{22}{10}\sqrt{\xi_{m+1}} + \frac{31}{10}\sqrt{\alpha_m\xi_{m+1}} + \frac{3}{40}\sum_{\barm\in[m]}\frac{z_{\barm,m+1}\Reg_{\hatf_{\barm}}(\pi)}{2\barm^2}\\ 
    \end{aligned}
\end{equation}
Hence from \eqref{eq:evaluation-guarantee-under-algsafe-epoch}, we have \eqref{eq:evaluation-guarantee} holds with $l_1=2.2$, $l_2=3.1$, $l_3=1.5$, $C' = 20$, and $\Tilde{\Pi}=\Pi$. Hence from \Cref{lemma:inductive-argument}, we have \eqref{eq:inductive-implication-within-algsafe-epoch-1} holds.
\begin{equation}
\label{eq:inductive-implication-within-algsafe-epoch-1}
    \begin{aligned}
             &\forall \pi\in\Pi,\\  
             &\Reg_{\Pi}(\pi) \leq \frac{6}{5}\Reg_{\hatf_{m+1}}(\pi) + \frac{12}{5}\Big(\frac{22}{10} + \frac{1.5*6.5}{20}\Big)\sqrt{\xi_{m+1}} + \frac{12}{5}\Big(\frac{31}{10} + \frac{1.5*12}{20}\Big)\sqrt{\alpha_m\xi_{m+1}} \\
             &= \frac{6}{5}\Reg_{\hatf_{m+1}}(\pi) + 6.45\sqrt{\xi_{m+1}} + 9.6\sqrt{\alpha_m\xi_{m+1}} \\
             &\leq \frac{4}{3}\Reg_{\hatf_{m+1}}(\pi) + 6.5\sqrt{\xi_{m+1}} + 12\sqrt{\alpha_m\xi_{m+1}} 
    \end{aligned}
\end{equation}
Now from \eqref{eq:inductive-hypothesis-within-algsafe-epoch} and \eqref{eq:inductive-implication-within-algsafe-epoch-1}, we have \eqref{eq:bounding-C3} holds. %
\begin{equation}
\label{eq:bounding-C3}
    \begin{aligned}
             &\forall \barm\in[m],\\  
             &\Reg_{\hatf_{\barm}}(\pi_{\hatf_{m+1}}) \leq \frac{4}{3}\Reg_{\Pi}(\pi_{\hatf_{m+1}}) + 6.5\sqrt{\xi_{\barm}} + 12\sqrt{\alpha_{\barm-1}\xi_{\barm}} \\
             &\leq \frac{4}{3}\bigg(0 + 6.5\sqrt{\xi_{m+1}} + 12\sqrt{\alpha_m\xi_{m+1}} \bigg) + 6.5\sqrt{\xi_{\barm}} + 12\sqrt{\alpha_{\barm-1}\xi_{\barm}} \\
             &\leq 43.2\sqrt{\alpha_{\barm-1}\xi_{\barm}}  
    \end{aligned}
\end{equation}
Hence from \eqref{eq:bounding-C3}, we have \eqref{eq:inductive-hypothesis-2} holds with $C_3 = 43.2$. Therefore from \Cref{lemma:inductive-argument} we have \eqref{eq:inductive-implication-within-algsafe-epoch-2}.
\begin{equation}
\label{eq:inductive-implication-within-algsafe-epoch-2}
    \begin{aligned}
             &\forall \pi\in\Pi,\\ &\Reg_{\hatf_{m+1}}(\pi) \leq \frac{7}{6}\Reg_{\Pi}(\pi) + \Big(2l_1 + \frac{l_3C_1}{C'}\Big)\sqrt{\xi_{m+1}} + \Big(2l_2 + \frac{l_3(C_2+C_3)}{C'}\Big)\sqrt{\alpha_m\xi_{m+1}}\\
             &= \frac{7}{6}\Reg_{\Pi}(\pi) + \Big(2*2.2 + \frac{1.5*6.5}{20}\Big)\sqrt{\xi_{m+1}} + \Big(2*3.1 + \frac{1.5(12+43.2)}{20}\Big)\sqrt{\alpha_m\xi_{m+1}}\\
             & = \frac{7}{6}\Reg_{\Pi}(\pi) + 4.8875\sqrt{\xi_{m+1}} + 10.34\sqrt{\alpha_m\xi_{m+1}}\\
             & \leq \frac{4}{3}\Reg_{\Pi}(\pi) + 6.5\sqrt{\xi_{m+1}} + 12\sqrt{\alpha_m\xi_{m+1}}
    \end{aligned}
\end{equation}
From \eqref{eq:inductive-implication-within-algsafe-epoch-1} and \eqref{eq:inductive-implication-within-algsafe-epoch-2}, we have \eqref{eq:regret-bounds-within-algsafe-epoch} holds for epoch $m$. This completes our inductive argument.
\end{proof}

An immediate implication of \Cref{lemma:regret-under-algsafe-epoch} is that we have $\Reg_{\hatf_m}(\pi^*)\leq U_m$ for all $m\in [\msafealg]$. Hence, from \Cref{lemma:bound_v_continuous} and \Cref{lemma:proving-condition-alpham}, we have \eqref{eq:condition_alpham_inductive} holds. 
\begin{equation}
\label{eq:condition_alpham_inductive}
   V(p_{m},\pi^*) \leq \frac{\E[\mu( C_{m}(x, \propThreshold, \eta_{m}))]}{1-\propThreshold} + \frac{K}{\eta_{m}} \leq \alpha_m,\; \forall m\in[\msafealg].
\end{equation}

\subsection{Bounding Exploration and Cumulative Regret}\label{sec:bounding-exploration-n-cumulative-regret}

This sub-section leverages the structure of the kernel $p_{m+1}$, and the guarantees in Lemmas \ref{lemma:test-implication} and \ref{lemma:regret-under-algsafe-epoch} to bound the expected regret during exploration (\Cref{lemma:exploration-regret-under-algsafe-epoch-with-benchmark}). Then, summing up these exploration regret bounds, we get our cumulative regret bound in \Cref{theorem:Cumulative-Regret-for-RAPR}. We start with \Cref{lemma:model-exploration-regret-under-algsafe-epoch} which leverages structure in $p_m$ to bound $\Reg_{\hatf_{\barm}}(p_m)$ for any $\barm\in[m]$.

\begin{restatable}{lemma}{lemExplorationModelRegret}
\label{lemma:model-exploration-regret-under-algsafe-epoch}
For any pair of epochs $m\in[\msafealg+1]$ and $\barm\in[m]$, we have that \eqref{eq:model-exploration-regret-under-algsafe-epoch} holds. 
\begin{equation}
\label{eq:model-exploration-regret-under-algsafe-epoch}
    \begin{aligned}
             \Reg_{\hatf_{\barm}}(p_m)\leq 15.2\sqrt{\xi_{\barm}}+28\sqrt{\alpha_{\barm-1}\xi_{\barm}}+2\barm^2\eta_mU_{\barm}\bigg(\frac{1}{\propThreshold}+ \ln\frac{\propThreshold}{2\barm^2\eta_m U_{\barm}}\bigg)
    \end{aligned}
\end{equation}
\end{restatable}
\begin{proof}
We first make the following observation.
\begin{equation}
\label{eq:kernel-to-CAS}
    \begin{aligned}
        &\E_{x\sim D_x}\E_{a\sim p_m(a|x)}[I(a\notin \pi_{\hatf_m}(x))\cdot(\hatf_{\barm}(x,\pi_{\hatf_{\barm}}(x))-\hatf_{\barm}(x,a))]\\
        &=\E_{x\sim D_x}\bigg[\int_{a\in\A\setminus \pi_{\hatf_m}(x)}(\hatf_{\barm}(x,\pi_{\hatf_{\barm}}(x))-\hatf_{\barm}(x,a))p_m(a|x) d\mu(a)\bigg]\\
        &\stackrel{(i)}{=}\E_{x\sim D_x}\bigg[\int_{a\in\A\setminus \pi_{\hatf_m}(x)}\int_{\beta\in[0,1]}(\hatf_{\barm}(x,\pi_{\hatf_{\barm}}(x))-\hatf_{\barm}(x,a))\frac{I[a\in C_{m}(x,\min(\beta,\propThreshold)/\eta_{m}) ]}{\mu\big(C_{m}(x,\min(\beta,\propThreshold)/\eta_{m})\big)} \mbox{d}\beta \mbox{d}\mu(a)\bigg]\\
        &\stackrel{(ii)}{\leq}\E_{x\sim D_x}\bigg[\int_{\beta\in[0,1]}\int_{a\in\A\setminus \pi_{\hatf_m}(x)}\min\bigg(1, \frac{2\barm^2\eta_mU_{\barm}}{\min(\beta,\propThreshold)}\bigg)\frac{I[a\in C_{m}(x,\min(\beta,\propThreshold)/\eta_{m}) ]}{\mu\big(C_{m}(x,\min(\beta,\propThreshold)/\eta_{m})\big)} \mbox{d}\mu(a)\mbox{d}\beta\bigg]\\
        &\leq \int_{\beta\in[0,1]}\min\bigg(1, \frac{2\barm^2\eta_mU_{\barm}}{\min(\beta,\propThreshold)}\bigg)\mbox{d}\beta  \leq (1-\propThreshold)\frac{2\barm^2\eta_m U_{\barm}}{\propThreshold} +\int_0^{\propThreshold}\min\bigg(1, \frac{2\barm^2\eta_mU_{\barm}}{\beta}\bigg)\mbox{d}\beta
    \end{aligned}
\end{equation}
where (i) follows from the definition of $p_m$ given in \eqref{eq:define_pm_continuous}. (ii) follows from the fact that for any $\zeta\in(0,1)$ and $a\in C_{m}(x,\zeta)\setminus \pi_{\hatf_m}(x)$ we have $\hatf_{\barm}(x,\pi_{\hatf_{\barm}}(x))-\hatf_{\barm}(x,a)\leq \min(1,2\barm^2/\zeta)$, since $C_{m}(x,\zeta)\setminus \pi_{\hatf_m}(x) \subseteq \bar{C}_{m}(x,\zeta) \subseteq \tilde{C}_{\barm}(x,\zeta/(2\barm^2))$ by \Cref{def:confidence_set}, . We now bound $\Reg_{\hatf_{\barm}}(p_{m})$.
\begin{equation}
    \begin{aligned}
    &\Reg_{\hatf_{\barm}}(p_{m}) = \E_{x\sim D_x}\E_{a\sim p_m(a|x)}[\hatf_{\barm}(x,\pi_{\hatf_{\barm}}(x))-\hatf_{\barm}(x,a)] \\
    &= \E_{x\sim D_x}\E_{a\sim p_m(a|x)}[(I(a\in \pi_{\hatf_m}(x))+I(a\notin \pi_{\hatf_m}(x)))\cdot(\hatf_{\barm}(x,\pi_{\hatf_{\barm}}(x))-\hatf_{\barm}(x,a))] \\
    &\stackrel{(i)}{\leq} \Reg_{\hatf_{\barm}}(\pi_{\hatf_m})+\bigg(
              (1-\propThreshold)\frac{2\barm^2\eta_m U_{\barm}}{\propThreshold} +\int_0^{\propThreshold}\min\bigg(1, \frac{2\barm^2\eta_mU_{\barm}}{\beta}\bigg)\mbox{d}\beta 
              \bigg)\\
    & \leq \Reg_{\hatf_{\barm}}(\pi_{\hatf_m})+\bigg(
              \frac{2\barm^2 \eta_m  U_{\barm}(1-\propThreshold)}{\propThreshold} + 2\barm^2\eta_mU_{\barm} + 2\barm^2\eta_mU_{\barm}\int_{2\barm^2\eta_mU_{\barm}}^{\propThreshold} \frac{1}{\beta}\mbox{d}\beta
              \bigg)\\
              &= \Reg_{\hatf_{\barm}}(\pi_{\hatf_m})+2\barm^2\eta_mU_{\barm}\bigg(\frac{1}{\propThreshold}+ \ln\frac{\propThreshold}{2\barm^2\eta_m U_{\barm}}\bigg)\\
              &\stackrel{(ii)}{\leq} \frac{4}{3}\Reg_{\Pi}(\pi_{\hatf_m})+6.5\sqrt{\xi_{\barm}}+12\sqrt{\alpha_{\barm-1}\xi_{\barm}}+2\barm^2\eta_mU_{\barm}\bigg(\frac{1}{\propThreshold}+ \ln\frac{\propThreshold}{2\barm^2\eta_m U_{\barm}}\bigg)\\
              &\stackrel{(iii)}{\leq} \frac{4}{3}\big(6.5\sqrt{\xi_{m}}+12\sqrt{\alpha_{m-1}\xi_{m}}\big)+6.5\sqrt{\xi_{\barm}}+12\sqrt{\alpha_{\barm-1}\xi_{\barm}}\\
              &\;\;\;\;\;\;\;\;\;\;\;\;+2\barm^2\eta_mU_{\barm}\bigg(\frac{1}{\propThreshold}+ \ln\frac{\propThreshold}{2\barm^2\eta_m U_{\barm}}\bigg)\\
              &\stackrel{(iv)}{\leq} 15.2\sqrt{\xi_{\barm}}+28\sqrt{\alpha_{\barm-1}\xi_{\barm}}+2\barm^2\eta_mU_{\barm}\bigg(\frac{1}{\propThreshold}+ \ln\frac{\propThreshold}{2\barm^2\eta_m U_{\barm}}\bigg)
    \end{aligned}
\end{equation}
where (i) follows from \eqref{eq:kernel-to-CAS}, %
(ii) follows from \Cref{lemma:regret-under-algsafe-epoch}, (iii) follows from \Cref{lemma:regret-under-algsafe-epoch} and $\Reg_{\hatf_m}(\pi_{\hatf_m})=0$, and (iv) follows from $\barm\leq m$.
\end{proof}
Now from the guarantees in Lemmas \ref{lemma:test-implication}, \ref{lemma:regret-under-algsafe-epoch}, and \ref{lemma:model-exploration-regret-under-algsafe-epoch} we get the following bound on $\Reg_{\Pi}(p_{m+1})$.

\begin{restatable}{lemma}{lemExplorationRegret}
\label{lemma:exploration-regret-under-algsafe-epoch}
Suppose $\eventModelOracle$ and $\eventRiskEstimation$ hold. Now for any epoch $m\in[\msafealg-1]$, we have that \eqref{eq:exploration-regret-bounds-within-algsafe-epoch} holds.
\begin{equation}
\label{eq:exploration-regret-bounds-within-algsafe-epoch}
    \begin{aligned}
             &\Reg_{\Pi}(p_{m+1})\leq  100(m+1)^2\eta_{m+1}\sqrt{\alpha_{m}\xi_{m+1}}\bigg(\frac{1}{\propThreshold}+ \ln\frac{\propThreshold}{40\eta_{m+1} \sqrt{\alpha_{m}\xi_{m+1}}}\bigg)
    \end{aligned}
\end{equation}
\end{restatable}
\begin{proof}
Since $m\in[\msafealg]$, from \Cref{lemma:test-implication}, we have \eqref{eq:evaluation-guarantee-under-algsafe-epoch-rep2} holds.
\begin{equation}
\label{eq:evaluation-guarantee-under-algsafe-epoch-rep2}
    \begin{aligned}
        &\forall \pi\in\Pi\cup\{p_{m+1}\}, \\
        &|R_{\hatf_{m+1}}(\pi)-R(\pi)| \leq \frac{22}{10}\sqrt{\xi_{m+1}} + \frac{31}{10}\sqrt{\alpha_m\xi_{m+1}} + \frac{3}{40}\sum_{\barm\in[m]}\frac{z_{\barm,m+1}\Reg_{\hatf_{\barm}}(\pi)}{2\barm^2}\\ 
    \end{aligned}
\end{equation}
We will now bound $\Reg_{\Pi}(p_{m+1})$ in terms of $\Reg_{\hatf_{\barm}}(p_{m+1})$ for $\barm\in[m+1]$.
\begin{equation}
\label{eq:bound_eq_1_upper_bounding_Reg_exploration}
    \begin{aligned}
    &\Reg_{\Pi}(p_{m+1}) - \Reg_{\hatf_{m+1}}(p_{m+1})\\
    &= R(\pi^*) - R(p_{m+1}) - R_{\hatf_{m+1}}(\pi_{\hatf_{m+1}}) + R_{\hatf_{m+1}}(p_{m+1})\\
    &\leq R(\pi^*) - R_{\hatf_{m+1}}(\pi^*) + (R_{\hatf_{m+1}}(p_{m+1}) - R(p_{m+1}) )\\
    &\stackrel{(i)}{\leq} \frac{44}{10}\sqrt{\xi_{m+1}} + \frac{62}{10}\sqrt{\alpha_m\xi_{m+1}} + \frac{3}{40}\sum_{\barm\in[m]}\frac{z_{\barm,m+1}}{2\barm^2}(\Reg_{\hatf_{\barm}}(p_{m+1})+\Reg_{\hatf_{\barm}}(\pi^*))\\
    &\stackrel{(ii)}{\leq} \frac{44}{10}\sqrt{\xi_{m+1}} + \frac{62}{10}\sqrt{\alpha_m\xi_{m+1}} \\
    &\;\;\;\;\;\;\;\;\;\;\;\;+ \frac{3}{40}\sum_{\barm\in[m]}\frac{z_{\barm,m+1}}{2\barm^2}(\Reg_{\hatf_{\barm}}(p_{m+1}) + 6.5\sqrt{\xi_{\barm}}+12\sqrt{\alpha_{\barm-1}\xi_{\barm}})\\
    &\stackrel{(iii)}{\leq} \frac{44}{10}\sqrt{\xi_{m+1}} + \frac{62}{10}\sqrt{\alpha_m\xi_{m+1}}\\ 
    &\;\;\;\;\;\;\;\;\;\;\;\; + \frac{3}{40}\sum_{\barm\in[m]}\frac{1}{2\barm^2}(z_{\barm,m+1}\Reg_{\hatf_{\barm}}(p_{m+1}) + 6.5\sqrt{\xi_{m+1}}+12\sqrt{\alpha_{m}\xi_{m+1}})\\
    &\stackrel{(iv)}{\leq} 4.9\sqrt{\xi_{m+1}} + 7.1\sqrt{\alpha_m\xi_{m+1}}  + \frac{3}{40}\sum_{\barm\in[m]}\frac{z_{\barm,m+1}}{2\barm^2}\Reg_{\hatf_{\barm}}(p_{m+1}).
\end{aligned}
\end{equation}
Where (i) follows from \eqref{eq:evaluation-guarantee-under-algsafe-epoch-rep2}, (ii) follows from \Cref{lemma:regret-under-algsafe-epoch} and from $\Reg_{\Pi}(\pi^*)=0$, (iii) follows from $z_{\barm,m+1}:=\sqrt{\frac{\alpha_m\xi_{m+1}}{\alpha_{\barm-1}\xi_{\barm}}}$ and $\alpha_m\leq \alpha_{\barm-1}$, finally (iv) follows from  $\sum_{\barm\in[m]}1/(2\barm^2)\leq 1$. We now simplify the last term in the upper bound of \eqref{eq:bound_eq_1_upper_bounding_Reg_exploration}.
\begin{equation}
\label{eq:simplify_last_term_in_bound_eq_1_upper_bounding_Reg_exploration}
\begin{aligned}
&\sum_{\barm\in[m]}\frac{z_{\barm,m+1}}{2\barm^2}\Reg_{\hatf_{\barm}}(p_{m+1})\\
& \stackrel{(i)}{\leq} \sum_{\barm\in[m]}\frac{z_{\barm,m+1}}{2\barm^2}\bigg(15.2\sqrt{\xi_{\barm}}+28\sqrt{\alpha_{\barm-1}\xi_{\barm}}+2\barm^2\eta_{m+1}U_{\barm}\bigg(\frac{1}{\propThreshold}+ \ln\frac{\propThreshold}{2\barm^2\eta_{m+1} U_{\barm}}\bigg)  \bigg)\\
& \stackrel{(ii)}{\leq} \sum_{\barm\in[m]}\frac{1}{2\barm^2}\bigg(15.2\sqrt{\xi_{m+1}}+28\sqrt{\alpha_{m}\xi_{m+1}}\\
&\;\;\;\;\;\;\;\;\;\;\;\;\;\;\;\;\;\;\;\;\;\;\;\;+40\barm^2\eta_{m+1}\sqrt{\alpha_{m}\xi_{m+1}}\bigg(\frac{1}{\propThreshold}+ \ln\frac{\propThreshold}{40\eta_{m+1} \sqrt{\alpha_{m}\xi_{m+1}}}\bigg)  \bigg)\\
& \stackrel{(iii)}{\leq} 15.2\sqrt{\xi_{m+1}}+28\sqrt{\alpha_{m}\xi_{m+1}}+20m\eta_{m+1}\sqrt{\alpha_{m}\xi_{m+1}}\bigg(\frac{1}{\propThreshold}+ \ln\frac{\propThreshold}{40\eta_{m+1} \sqrt{\alpha_{m}\xi_{m+1}}}\bigg)
\end{aligned}
\end{equation}
Where (i) follows from \Cref{lemma:model-exploration-regret-under-algsafe-epoch}, (ii) follows from $z_{\barm,m+1}:=\sqrt{\frac{\alpha_m\xi_{m+1}}{\alpha_{\barm-1}\xi_{\barm}}}$, choice of $U_m$, and $\alpha_m\leq \alpha_{\barm-1}$, finally (iii) follows from  $\sum_{\barm\in[m]}1/(2\barm^2)\leq 1$. By combining \eqref{eq:bound_eq_1_upper_bounding_Reg_exploration}, \eqref{eq:simplify_last_term_in_bound_eq_1_upper_bounding_Reg_exploration}, and \Cref{lemma:model-exploration-regret-under-algsafe-epoch}, we get our final result.
\begin{equation}
\label{eq:bound_eq_2_upper_bounding_Reg_exploration}
    \begin{aligned}
    &\Reg_{\Pi}(p_{m+1})\\ 
    & \stackrel{(i)}{\leq} \Reg_{\hatf_{m+1}}(p_{m+1}) + 6.04\sqrt{\xi_{m+1}} + 9.2\sqrt{\alpha_m\xi_{m+1}}\\ 
    &\;\;\;\;\;\;\;\;\;\;\;\;+ 1.5m\eta_{m+1}\sqrt{\alpha_{m}\xi_{m+1}}\bigg(\frac{1}{\propThreshold}+ \ln\frac{\propThreshold}{40\eta_{m+1} \sqrt{\alpha_{m}\xi_{m+1}}}\bigg)  \\
    & \stackrel{(ii)}{\leq} 21.3\sqrt{\xi_{m+1}}+37.2\sqrt{\alpha_{m}\xi_{m+1}}\\
    &\;\;\;\;\;\;\;\;\;\;\;\;+41.5(m+1)^2\eta_{m+1}\sqrt{\alpha_{m}\xi_{m+1}}\bigg(\frac{1}{\propThreshold}+ \ln\frac{\propThreshold}{40\eta_{m+1} \sqrt{\alpha_{m}\xi_{m+1}}}\bigg) 
\end{aligned}
\end{equation}
Where (i) follows from \eqref{eq:bound_eq_1_upper_bounding_Reg_exploration} and \eqref{eq:simplify_last_term_in_bound_eq_1_upper_bounding_Reg_exploration}, and (ii) follows from \Cref{lemma:model-exploration-regret-under-algsafe-epoch}.
\end{proof}

The earlier bound on $\Reg_{\Pi}(p_{m+1})$ now immediately gives us the following bound on $\Reg_{f^*}(p_{m+1})$.

\begin{restatable}{lemma}{lemExplorationRegretWithBenchmark}
\label{lemma:exploration-regret-under-algsafe-epoch-with-benchmark}
Suppose $\eventModelOracle$ and $\eventRiskEstimation$ hold. Now for any epoch $m\in[\msafealg-1]$, we have that \eqref{eq:exploration-regret-bounds-within-algsafe-epoch-with-benchmark} %
\begin{equation}
\label{eq:exploration-regret-bounds-within-algsafe-epoch-with-benchmark}
    \begin{aligned}
             &\Reg_{f^*}(p_{m+1})\leq 2\sqrt{KB} + 100(m+1)^2\eta_{m+1}\sqrt{\alpha_{m}\xi_{m+1}}\bigg(\frac{1}{\propThreshold}+ \ln\frac{\propThreshold}{40\eta_{m+1} \sqrt{\alpha_{m}\xi_{m+1}}}\bigg)
    \end{aligned}
\end{equation}
\end{restatable}
\begin{proof}
From \Cref{ass:regression-oracle} (properties of $\EstOracle$), we know the bias of the model class $\F$ is bounded by $B$. In particular, we know there exists $g\in\F$ such that $\E_{x\sim D_{\calX}, a\sim \Unif(\A)}\big[ (g(x,a)-f^*(x,a))^2\big]\leq B$. Hence, we have, the following.
\begin{equation}
\begin{aligned}
  &\Reg_{f^*}(\pi^*) \stackrel{(i)}{\leq} \Reg_{f^*}(\pi_g) = R(\pi_{f^*}) - R(\pi_g)\\ 
  &= (R(\pi_{f^*})-R_g(\pi_{f^*})) -\Reg_g(\pi_{f^*}) + (R_g(\pi_{g})-R(\pi_{g}))  \\
  &\stackrel{(ii)}{\leq} |R(\pi_{f^*})-R_g(\pi_{f^*})| + |R_g(\pi_{g})-R(\pi_{g})|  \\
  &\stackrel{(iii)}{\leq} \bigg(\sqrt{\E_{x\sim D_{\calX}, a\sim \pi_{f^*}}\Big[
  \frac{\pi_{f^*}(a|x)}{1/K}\Big]} + \sqrt{\E_{x\sim D_{\calX}, a\sim \pi_{g}}\Big[
  \frac{\pi_{g}(a|x)}{1/K}\Big]}\bigg)\\
  &\;\;\;\;\;\;\;\;\;\;\;\; \cdot\sqrt{\E_{x\sim D_{\calX}, a\sim \Unif(\A)}\big[ (g(x,a)-f^*(x,a))^2\big]} \\
  &\stackrel{(iv)}{\leq} 2\sqrt{KB}.
\end{aligned}
\end{equation}
Here (i) follows from the fact that $\pi_g\in\Pi$ since $g\in\F$. (ii) follows from triangle inequality and the fact that $\Reg_g(\pi_{f^*})\geq 0$. (iii) follows from the proof of \Cref{lemma:bound_reward_diff_continuous_refined}. And (iv) follows from $\E_{x\sim D_{\calX}, a\sim \Unif(\A)}\big[ (g(x,a)-f^*(x,a))^2\big]\leq B$ and $\pi_f(a|x)=I(a\in\pi_f(x))$.

Since $\Reg_{f^*}(p_{m+1}) = R(\pi_{f^*}) - R(\pi^*) + R(\pi^*) - R(p_{m+1}) = \Reg_{f^*}(\pi^*) + \Reg_{\Pi}(p_{m+1})$, the result follows from combining the above with \Cref{lemma:exploration-regret-under-algsafe-epoch}.
\end{proof}

We now get our final cumulative regret bound by summing up the exploration regret bounds in \Cref{lemma:exploration-regret-under-algsafe-epoch-with-benchmark}.

\thmRAPRCumulativeRegret*
\begin{proof}
From \Cref{sec:high-probability-events}, both $\eventModelOracle$ and $\eventRiskEstimation$ hold with probability $1-\delta$. We prove our cumulative regret bounds under these events. Under $\eventModelOracle$, from \Cref{lemma:misspecification-test}, we have $\msafealg\geq\msafe+1$. Further, from conditions in \Cref{sec:oracle-assumptions}, we have $\xi_m$ is non-increasing in $m$. Since $\xi(n,\delta')$ scales polynomially in $1/n$ and $\log(1/\delta')$, there exists a constant $Q_0> 1$ such that the doubling epoch structure ensures $\xi_{m}\leq Q_0\xi_{m+1}$ for all $m$. Hence $\xi_{\msafealg}\leq Q_0\xi_{\msafealg+1}\leq Q_0\xi_{\msafe+2}\leq 2Q_0B$. Let $m'(t)=\min(m(t),\msafealg)$. Hence, $\xi_{m'(t)}\leq \max(\xi_{m(t)},\xi_{\msafealg}) \leq 2Q_0B+\xi_{m(t)}$. Therefore, by summing up the bounds in \Cref{lemma:exploration-regret-under-algsafe-epoch-with-benchmark}, we have the following cumulative regret bound.
\begin{equation}
    \begin{aligned}
        &\CReg_T \leq \sum_{t=1}^T \Reg_{f^*}(p_{m'(t)})\\
        &\leq \tau_1+\sum_{t=\tau_1+1}^T \Bigg(2\sqrt{KB}\\ 
        &\;\;\;\;\;\;\;\;\;\;\;\;+ 100(m'(t))^2\eta_{m'(t)}\sqrt{\alpha_{m'(t)-1}\xi_{m'(t)}}\bigg(\frac{1}{\propThreshold}+ \ln\frac{\propThreshold}{40\eta_{m'(t)}\sqrt{\alpha_{m'(t)-1}\xi_{m'(t)}}}\bigg) \Bigg)\\
        &\leq \ordOt\Bigg(\sum_{t=\tau_1+1}^T \bigg(\eta_{m'(t)}\sqrt{\alpha_{m'(t)-1}\xi_{m'(t)}}\bigg)\Bigg) = \ordOt\Bigg(\sum_{t=\tau_1+1}^T \eta_{m(t)}\sqrt{\frac{\alpha_{m(t)-1}}{K}}\big(\sqrt{K\xi_{m'(t)}}  \big)\Bigg)\\
        &\leq \ordOt\Bigg(\sum_{t=\tau_1+1}^T \eta_{m(t)}\sqrt{\frac{\alpha_{m(t)-1}}{K}}\bigg(\sqrt{KB}+\sqrt{K\xi_{m(t)}}  \bigg)\Bigg)
    \end{aligned}
\end{equation}
Now the theorem follows from the fact that we have:
\begin{equation}
    \begin{aligned}
    \eta_{m}\sqrt{\frac{\alpha_{m-1}}{K}}\stackrel{\Cref{lemma:proving-condition-alpham}}{\leq} 3\sqrt{\frac{K}{\alpha_{m}}\frac{\alpha_{m-1}}{\alpha_{m}}}\stackrel{\Cref{lemma:proving-condition-alpham}}{\leq} 3\eta_{m}\sqrt{\frac{\alpha_{m-1}}{K}} \stackrel{\eqref{eq:choose-etam}}{\leq} 3\sqrt{\omega}.
    \end{aligned}
\end{equation}
\end{proof}

\section{Bounding Simple Regret}
\label{sec:bounding-simple-regret}

In this section, we prove our simple regret bound (\Cref{theorem:Simple-Regret-for-RAPR}). Our analysis starts with \Cref{lemma:bound_conformal_set_size}, which provides instance dependent bounds on $\E_{x\sim D_{\Xscript}}\Big[\mu\Big(C_m(x,\beta,\eta)\Big) \Big]$. We will later use \Cref{lemma:bound_conformal_set_size} to derive instance-dependent bounds on $\alpha_m$. This bound then helps us derive instance-dependant bounds on simple regret.

\begin{restatable}{lemma}{lemBoundConformalSetSize}
\label{lemma:bound_conformal_set_size}
For some environment parameters $\lambda\in (0,1)$, $\Delta>0$, and $A\in[1,K]$, consider an instance where \eqref{eq:environment-conditions} holds.
\begin{equation}
    \mathbb{P}_{x\sim D_{\Xscript}}\Big(\mu\big(\{a\in\A: f^*(x,\pi_{f^*}(x)) - f^*(x,a) \leq \Delta \}\big) \leq A \Big) \geq 1- \lambda.
\end{equation}
Suppose $\eventModelOracle$ and $\eventRiskEstimation$ hold. For all epochs $m$, suppose the action selection kernel is given by \cref{eq:define_pm_continuous}, and suppose \eqref{eq:misspecification-test} holds for all $\barm\in[m]$. Then for any epoch $m\in[\msafe]$, we have \eqref{eq:bound-setsize} holds.
\begin{equation}
\label{eq:bound-setsize}
    \E_{x\sim D_{\Xscript}}\Big[\mu\Big(C_m(x,\beta,\eta)\Big) \Big] \leq \Big(1+A+K\lambda\Big)+25\frac{K}{\Delta}\frac{\eta}{\beta}\sqrt{\alpha_{m-1}\xi_{m}}.
\end{equation}
For any $\beta\in(0,1/2]$ and $\eta\in[1,K]$.
\end{restatable}
\begin{proof}
Consider any epoch $m\in[\msafe]$. In this proof, for short-hand, let $C := \E[\mu(C_m(x,\beta,\eta))]$. We then have,
\begin{equation*}
    \begin{aligned}
             C &= \E[\mu(C_m(x,\beta,\eta))] \\
             &\leq (A+1) P(\mu(C_m(x,\beta,\eta))\leq A+1) + K P(\mu(C_m(x,\beta,\eta))> A+1)\\
             &\leq A+1 + K P(\mu(C_m(x,\beta,\eta))> A+1)\\
             &\leq A+1 + K - K P(\mu(C_m(x,\beta,\eta))\leq A+1)
    \end{aligned}
\end{equation*}
The above immediately implies \eqref{eq:lower-bound-setsize-greater-than-A}.
\begin{equation}
\label{eq:lower-bound-setsize-greater-than-A}
    \begin{aligned}
             P\big(\mu(C_m(x,\beta,\eta))\leq A+1\big) \leq \frac{A+1+K-C}{K}.
    \end{aligned}
\end{equation}
Let $\pi_0\in\Tilde{\Pi}$ be defined by \eqref{eq:define-pi0}.
\begin{equation}
\label{eq:define-pi0}
    \begin{aligned}
             \forall x\in\Xscript,\; \pi_0(x) \in \argmin_{S\in\Sigma_1| S\subseteq C_m(x,\beta,\eta)} f^*(x,S).
    \end{aligned}
\end{equation}
Since $\pi_0$ only selects arms in $C_m(x,\beta,\eta)$, from \Cref{def:confidence_set}, we have \eqref{eq:upper_bound_pi0}.
\begin{equation}
\label{eq:upper_bound_pi0}
    \Reg_{\hatf_m}(\pi_0)\leq \frac{\eta}{\beta}U_m.
\end{equation}
We can lower bound the regret of $\pi_0$ as follows,
\begin{equation}
\label{eq:lower_bound_pi0}
    \begin{aligned}
     &\Reg_{f^*}(\pi_0)\\ 
     \geq &
     P( f^*(x,\pi_{f^*}(x)) - f^*(x, \pi_0(x))> \Delta)\cdot \Delta \\
     \stackrel{(i)}{=} & P(\exists \; S\in\Sigma_1|\; S\subseteq C_m(x,\beta,\eta),\;  \; f^*(x,\pi_{f^*}(x)) - f^*(x, S)> \Delta)\cdot \Delta \\
   \stackrel{(ii)}{\geq} & P\Big(\mu\big(C_m(x,\beta,\eta)\big)\geq A+1~ \mbox{and}~ \mu\big(\{a :(f^*(x,\pi_{f^*}(x)) - f^*(x, a)> \Delta\}\big) \geq K-A\Big)\cdot \Delta \\
   = & P\Big(\mu\big(C_m(x,\beta,\eta)\big)\geq A+1~ \mbox{and}~ \mu\big(\{a :(f^*(x,\pi_{f^*}(x)) - f^*(x, a)\leq \Delta\}\big) \leq A\Big)\cdot \Delta \\
   = & \bigg(1-P\Big(\mu\big(C_m(x,\beta,\eta)\big)< A+1~ \mbox{or}~ \mu\big(\{a :(f^*(x,\pi_{f^*}(x)) - f^*(x, a)\leq \Delta\}\big) > A\Big) \bigg) \cdot \Delta \\
   \stackrel{(iii)}{\geq} & \bigg(1-P\Big(\mu\big(C_m(x,\beta,\eta)\big)< A+1\Big) -P\Big(\mu\big(\{a :(f^*(x,\pi_{f^*}(x)) - f^*(x, a)\leq \Delta\}\big) > A\Big) \bigg) \cdot \Delta \\
   = & \bigg(P\Big(\mu\big(\{a :(f^*(x,\pi_{f^*}(x)) - f^*(x, a)\leq \Delta\}\big) \leq A\Big) -P\Big(\mu\big(C_m(x,\beta,\eta)\big)< A+1\Big) \bigg) \cdot \Delta \\
    \stackrel{(iv)}{\geq}  & \bigg(1-\lambda- \frac{A+1+K-C}{K} \bigg)\Delta = \bigg(\frac{C-A-1}{K} - \lambda\bigg)\Delta.
\end{aligned}
\end{equation}
where (i) is because by construction $\pi_0(x)\subseteq C_m(x,\beta,\eta)$ for all $x$, (ii) is by the fact that $\mu$ is a finite measure with $\mu(\A)=:K$, (iii) follows from union bound, and (iv) follows from \eqref{eq:lower-bound-setsize-greater-than-A} and \eqref{eq:environment-conditions}.

We will now work towards upper bounding $\Reg_{f^*}(\pi_0)$, and use this bound in conjunction with \eqref{eq:lower_bound_pi0} to obtain our desired bound on $C$. To upper bound $\Reg_{f^*}(\pi_0)$ using \Cref{lemma:regret-under-algsafe-epoch}, we will upper bound $\Reg_{\hatf_{m-1}}(\pi_0)$ and $\Reg_{\hatf_{m-1}}(\pi_{f^*})$. 
\begin{equation}
\label{eq:bound-regret-of-pi-0-based-on-f_m-1}
    \begin{aligned}
             \Reg_{\hatf_{m-1}}(\pi_0) &\stackrel{(i)}{\leq} \frac{4}{3}\Reg_{\Pi}(\pi_0) + 12\sqrt{\alpha_{m-2}\xi_{m-1}} + 6.5 \sqrt{\xi_{m-1}} \\
             &\stackrel{(ii)}{\leq} \frac{4}{3}\bigg(\frac{4}{3}\Reg_{\hatf_{m}}(\pi_0) + 12\sqrt{\alpha_{m-1}\xi_{m}} + 6.5 \sqrt{\xi_{m}} \bigg) + 12 \sqrt{\alpha_{m-2}\xi_{m-1}} + 6.5 \sqrt{\xi_{m-1}}\\
             &\stackrel{(iii)}{\leq} \frac{16}{9}\Reg_{\hatf_{m}}(\pi_0) + 28\sqrt{\alpha_{m-2}\xi_{m-1}} + \frac{91}{6}\sqrt{\xi_{m-1}} \\
             &\stackrel{(iv)}{\leq}\frac{16}{9}\Reg_{\hatf_{m}}(\pi_0) + \frac{259}{6}\sqrt{\alpha_{m-2}\xi_{m-1}}.
    \end{aligned}
\end{equation}
Where (i) and (ii) follow from \Cref{lemma:regret-under-algsafe-epoch}, (iii) follows from $z_{m-1}=\sqrt{\frac{\alpha_{m-1}\xi_m}{\alpha_{m-2}\xi_{m-1}}} \leq 1$, and (iv) follows from $\alpha_{m-2}\geq 1$.
\begin{equation}
\label{eq:bound-regret-of-pi_f*-based-on-f_m-1}
    \begin{aligned}
             \Reg_{\hatf_{m-1}}(\pi_{f^*}) &\stackrel{(i)}{\leq}  12\sqrt{\alpha_{m-2}\xi_{m-1}} + 6.5\sqrt{\xi_{m-1}}\\
             &\stackrel{(ii)}{\leq}  \frac{37}{2}\sqrt{\alpha_{m-2}\xi_{m-1}}.
    \end{aligned}
\end{equation}
Where (i) follows from \Cref{lemma:regret-under-algsafe-epoch}, (ii) follows from $\alpha_{m-2}\geq 1$.

\begin{equation}
\label{eq:upper_bound_true_regret_of_pi-0}
    \begin{aligned}
        & \Reg_{f^*}(\pi_0) \\
        & =  R(\pi_{f^*}) - R(\pi_0)\\
        & = \big(R(\pi_{f^*}) - R_{\hatf_{m}}(\pi_{f^*}) \big) - \big(R(\pi_0) - R_{\hatf_{m}}(\pi_0) \big) + \big(R_{\hatf_{m}}(\pi_{f^*}) - R_{\hatf_m}(\pi_0) \big)\\
        & \stackrel{(i)}{\leq} 2\sqrt{\alpha_{m-1}\xi_{m}} + \frac{1}{2}\sqrt{\alpha_{m-1}\xi_{m}}\sum_{\bar{m}\in[m]}\frac{1}{2\bar{m}^2U_{\bar{m}}}\Big(\Reg_{\hatf_{m-1}}(\pi_{f^*})+\Reg_{\hatf_{m-1}}(\pi_0)\Big) +   \Reg_{\hatf_m}(\pi_0) \\
        & \stackrel{(ii)}{\leq} 2\sqrt{\alpha_{m-1}\xi_{m}} + \frac{1}{40}\sum_{\bar{m}\in[m]}\frac{z_{\bar{m},m-1}}{2\bar{m}^2}\Big(\frac{16}{9}\Reg_{\hatf_{m}}(\pi_{0})+\Big(37+18.5*\frac{4}{3}\Big)\sqrt{\alpha_{m-2}\xi_{m-1}}\Big) +   \Reg_{\hatf_m}(\pi_0)\\
        & \stackrel{(iii)}{\leq} 3.6\sqrt{\alpha_{m-1}\xi_{m}}+ \frac{47}{45}\Reg_{\hatf_m}(\pi_0) \stackrel{(iv)}{\leq} \sqrt{\alpha_{m-1}\xi_{m}}\bigg(3.6+\frac{47}{45}* 20 \frac{\eta}{\beta}\bigg) \leq 25\frac{\eta}{\beta}\sqrt{\alpha_{m-1}\xi_{m}}
    \end{aligned}
\end{equation}
Where (i) follows from \Cref{lemma:bound_reward_diff_continuous_refined}. (ii) follows from \eqref{eq:bound-regret-of-pi-0-based-on-f_m-1}, \eqref{eq:bound-regret-of-pi_f*-based-on-f_m-1}, and $U_{m-1}=20\sqrt{\alpha_{m-2}\xi_{m-1}}$, (iii) follows from $z_{m-1}\leq 1$, and (iv) follows from \eqref{eq:upper_bound_pi0} and $U_{m}=20\sqrt{\alpha_{m-1}\xi_{m}}$. Finally, combining \eqref{eq:lower_bound_pi0} and \eqref{eq:upper_bound_true_regret_of_pi-0}, we have,
\begin{equation}
    \begin{aligned}
     &\bigg(\frac{C-A-1}{K} - \lambda\bigg)\Delta \leq \Reg_{f^*}(\pi_0) \leq 25\frac{\eta}{\beta}\sqrt{\alpha_{m-1}\xi_{m}}\\
     \implies& C \leq A+1+K\lambda+25\frac{K}{\Delta}\frac{\eta}{\beta}\sqrt{\alpha_{m-1}\xi_{m}}.
\end{aligned}
\end{equation}
\end{proof}

In \Cref{lemma:bound_alpham_theoretical}, we use the bound from \Cref{lemma:bound_conformal_set_size} to derive instance-dependent bounds on $\alpha_m$. \Cref{cor:bound_only_alpham_theoretical} is an immediate implication of \Cref{lemma:bound_alpham_theoretical}, and provides a bound on $\alpha_m$ that doesn't depend on $\alpha_{m-1}$. Finally, \Cref{cor:bound_only_alpham_theoretical} is  used to derive our instance-dependant bound on simple regret.

\begin{restatable}{lemma}{lemboundalphamtheoretically}
\label{lemma:bound_alpham_theoretical}
For some environment parameters $\lambda\in (0,1)$, $\Delta>0$, and $A\in[1,K]$, consider an instance where \eqref{eq:environment-conditions} holds. Suppose $\eventModelOracle$ and $\eventRiskEstimation$ hold, and $\eta_m$ is chosen using \eqref{eq:choose-etam}. For all epochs $m$, suppose the action selection kernel is given by \cref{eq:define_pm_continuous}, suppose \cref{eq:condition_alpham} holds, and suppose \eqref{eq:misspecification-test} holds for all $\barm\in[m]$. Then for any epoch $m\in[\msafe]$, we have \eqref{eq:bound-alpha-m-with-environment-parameters-and-previous-alpha} holds.
\begin{equation}
\label{eq:bound-alpha-m-with-environment-parameters-and-previous-alpha}
    \alpha_m \leq \ordO\bigg(\max\bigg(\sqrt{\frac{K\alpha_{m-1}}{\omega}}, A+K\lambda + \frac{\sqrt{K^3\omega\xi_{m}}}{\Delta}\;\bigg) \bigg)
\end{equation}
\end{restatable}
\begin{proof}
Suppose $\eta_m\leq \sqrt{\frac{K\omega}{\alpha_{m-1}}}-1/|S_{m-1,2}|$, we then have,
\begin{equation}
    \begin{aligned}
             & \frac{K}{\eta_m} \stackrel{(i)}{\leq} \frac{|S_{m-1,2}|+1}{|S_{m-1,2}|}\frac{K}{\eta_m+\frac{1}{|S_{m-1,2}|}}\\ 
             &\stackrel{(ii)}{\leq}\frac{|S_{m-1,2}|+1}{|S_{m-1,2}|}\lambda_m\bigg(\eta_m+\frac{1}{|S_{m-1,2}|}\bigg) \\
             &\stackrel{(iii)}{\leq}\frac{|S_{m-1,2}|+1}{|S_{m-1,2}|}\bigg(1+\E\bigg[\mu\bigg(C_{m}\bigg(x_t,\beta_{\max},\eta_m+\frac{1}{|S_{m-1,2}|}\bigg)\bigg)\bigg] + \sqrt{\frac{2K^2\ln(8|S_{m-1,2}|m^2/\delta)}{|S_{m-1,2}|}}\bigg)\\
             &\stackrel{(iv)}{\leq}\frac{|S_{m-1,2}|+1}{|S_{m-1,2}|}\bigg((2+A+K\lambda) + 25\frac{K}{\Delta}\frac{\eta_m+\frac{1}{|S_{m-1,2}|}}{\beta_{\max}}\sqrt{\alpha_{m-1}\xi_{m}} + \sqrt{\frac{2K^2\ln(8|S_{m-1,2}|m^2/\delta)}{|S_{m-1,2}|}}\bigg)\\
             &\stackrel{(v)}{\leq}\frac{|S_{m-1,2}|+1}{|S_{m-1,2}|}\bigg((1+A+K\lambda) + \frac{50}{\Delta}\sqrt{K^3\omega\xi_{m}} + \sqrt{\frac{2K^2\ln(8|S_{m-1,2}|m^2/\delta)}{|S_{m-1,2}|}}\bigg)
    \end{aligned}
\end{equation}
Where (i) follows from $\eta_m\geq 1$, (ii) follows from \eqref{eq:choose-etam}, (iii) follows from $\eventRiskEstimation$, (iv) follows from \Cref{lemma:bound_conformal_set_size}, and (v) follows from \eqref{eq:choose-etam} and the fact that $\propThreshold= 0.5$. Finally, the result now follows from \Cref{lemma:proving-condition-alpham}.
\end{proof}

\begin{corollary}
\label{cor:bound_only_alpham_theoretical}
For some environment parameters $\lambda\in (0,1)$, $\Delta>0$, and $A\in[1,K]$, consider an instance where \eqref{eq:environment-conditions} holds. Suppose $\eventModelOracle$ and $\eventRiskEstimation$ hold. For all epochs $m$, suppose the action selection kernel is given by \cref{eq:define_pm_continuous}, suppose \cref{eq:condition_alpham} holds, and suppose suppose \eqref{eq:misspecification-test} holds for all $\barm\in[m]$. Then for any epoch $m\in[\msafe]$, we have \eqref{eq:bound-alpha-m-with-only-environment-parameters} holds.
\begin{equation}
\label{eq:bound-alpha-m-with-only-environment-parameters}
    \alpha_m \leq \ordO\bigg(\frac{K}{\omega}+ A+K\lambda + \frac{\sqrt{K^3\omega\xi_{m-\lceil\log_2\log_2(K)\rceil}}}{\Delta} \bigg)
\end{equation}
Where for notational convenience, we let $\xi_i=1$ for $i\leq 0$.
\end{corollary}
\begin{proof}
By repeatedly applying \Cref{lemma:bound_alpham_theoretical}, we have:
\begin{equation}
\begin{aligned}
    &\alpha_m\\ 
    &\leq \ordO\bigg(\max\bigg(\bigg(\frac{K}{\omega}\bigg)^{\frac{1}{2}+\frac{1}{4}+\cdots+\frac{1}{2^{\lceil\log_2\log_2(K)\rceil}}}K^{0.5^{\lceil\log_2\log_2(K)\rceil}}, A+K\lambda + \frac{\sqrt{K^3\omega\xi_{m-\lceil\log_2\log_2(K)\rceil}}}{\Delta}\;\bigg) \bigg)\\
    &\stackrel{(i)}{\leq} \ordO\bigg(\max\bigg(\bigg(\frac{K}{\omega}\bigg)K^{0.5^{\lceil\log_2\log_2(K)\rceil}}, A+K\lambda + \frac{\sqrt{K^3\omega\xi_{m-\lceil\log_2\log_2(K)\rceil}}}{\Delta}\;\bigg) \bigg)\\
    &\stackrel{(ii)}{\leq} \ordO\bigg(\max\bigg(\bigg(\frac{K}{\omega}\bigg), A+K\lambda + \frac{\sqrt{K^3\omega\xi_{m-\lceil\log_2\log_2(K)\rceil}}}{\Delta}\;\bigg) \bigg)\\
    &\leq \ordO\bigg(\frac{K}{\omega}+ A+K\lambda + \frac{\sqrt{K^3\omega\xi_{m-\lceil\log_2\log_2(K)\rceil}}}{\Delta} \bigg)
\end{aligned}
\end{equation}
where (i) follows from $\sum_{i=1}^{\infty}1/2^i=1$, and (ii) follows from $K^{1/2^{\lceil\log_2\log_2(K)\rceil}}\leq K^{1/2^{\log_2\log_2(K)}}= K^{1/\log_2K} = K^{\log_K 2}=2$.
\end{proof}

We now re-state and prove \Cref{theorem:Simple-Regret-for-RAPR}. As discussed earlier, this result relies on the bound in \Cref{cor:bound_only_alpham_theoretical}.

\thmRAPRSimpleRegret*
\begin{proof}
From \Cref{sec:high-probability-events}, both $\eventModelOracle$ and $\eventRiskEstimation$ hold with probability $1-\delta$. We prove our simple regret bounds under these events. Let $m=m(T)$, we then have the following bound.
    \begin{equation}
    \begin{aligned}
        &R(\hat{\pi}) \stackrel{(i)}{\geq} \hatR_{m}(\hat{\pi}) - \sqrt{\alpha_{m'}\xi_{m}} - \frac{1}{2}\sqrt{\alpha_{m'}\xi_{m}}\sum_{\barm\in[m']}\frac{\Reg_{\hatf_{\barm}}(\hat{\pi})}{2\barm^2U_{\barm}} - \frac{K\xi_{m}}{\eta_{m'} \min_{\barm\in[m']}U_{\barm}}\\
        &\stackrel{(ii)}{\geq} \hatR_{m}(\hat{\pi}) - \sqrt{\alpha_{m'}\xi_{m}} - \frac{1}{2}\sqrt{\alpha_{m'}\xi_{m}}\sum_{\barm\in[m']}\frac{\hatReg_{m,\hatf_{\barm}}(\hat{\pi})}{2\barm^2U_{\barm}} - \frac{2\sqrt{\alpha_{m'}}\xi_{m}}{U_{m'}} - \frac{\alpha_{m'}\xi_{m}}{U_{m'}}\\
        &\stackrel{(iii)}{\geq} \hatR_{m}(\pi^*) - \sqrt{\alpha_{m'}\xi_{m}} - \frac{1}{2}\sqrt{\alpha_{m'}\xi_{m}}\sum_{\barm\in[m']}\frac{\hatReg_{m,\hatf_{\barm}}(\pi^*)}{2\barm^2U_{\barm}} - \frac{2\sqrt{\alpha_{m'}}\xi_{m}}{U_{m'}} - \frac{\alpha_{m'}\xi_{m}}{U_{m'}}\\
        &\stackrel{(iv)}{\geq} \hatR_{m}(\pi^*) - \sqrt{\alpha_{m'}\xi_{m}} - \frac{1}{2}\sqrt{\alpha_{m'}\xi_{m}}\sum_{\barm\in[m']}\frac{\Reg_{\hatf_{\barm}}(\pi^*)}{2\barm^2U_{\barm}} - \frac{4\sqrt{\alpha_{m'}}\xi_{m}}{U_{m'}} - \frac{\alpha_{m'}\xi_{m}}{U_{m'}}\\
        &\stackrel{(v)}{\geq} R(\pi^*) - 2\sqrt{\alpha_{m'}\xi_{m}} - \sqrt{\alpha_{m'}\xi_{m}}\sum_{\barm\in[m']}\frac{\Reg_{\hatf_{\barm}}(\pi^*)}{2\barm^2U_{\barm}} - \frac{4\sqrt{\alpha_{m'}}\xi_{m}}{U_{m'}} - \frac{2\alpha_{m'}\xi_{m}}{U_{m'}}\\
        &\stackrel{(vi)}{\geq} R(\pi^*) - 3.3\sqrt{\alpha_{m'}\xi_{m}} .
    \end{aligned}
    \end{equation}
Here (i) follows from \Cref{lemma:bound_IPS_error}. (ii) follows from \Cref{lemma:bound-model-regret}, \Cref{lemma:proving-condition-alpham}, and the fact that $U_{m'}\leq U_{\barm}$ for any $\barm\in[m']$. (iii) follows from \eqref{eq:variance-penalized-optimization}. (iv) follows from \Cref{lemma:bound-model-regret}. (v) follows from \Cref{lemma:bound_IPS_error}, \Cref{lemma:proving-condition-alpham}, and the fact that $U_{m'}\leq U_{\barm}$ for any $\barm\in[m']$. Finally, (vi) follows from \Cref{lemma:regret-under-algsafe-epoch} and $U_{m'}\leq 20\sqrt{\alpha_{m'}\xi_m}$. Hence $\Reg_{\Pi}(\hat{\pi})\leq \ordO(\sqrt{\alpha_{m'}\xi_{m}})$. Now the final bound follows from the fact that $\alpha_{m'}\leq \alpha_1=3K$, $\alpha_{m'}\leq \alpha_{\min(\msafe,m(T)-1)}$, and \Cref{cor:bound_only_alpham_theoretical}.
\end{proof}

\section{Lower bound}

\thmLowerBound*
We prove \cref{thm:lower-new} in the following sub-sections.

\subsection{Basic Technical Results}

The following result is established in \cite{raginsky2011lower}, with this version taken from the proof of Lemma D.2 in \cite{foster2020instance}.

\begin{lemma}[Fano's inequality with reverse KL-divergence]
\label{lem:fano-reverse-KL}
Let $$\calH = (x_1,a_1,r_1(a_1)), \dots, (x_T,a_T,r_T(a_T)),$$ 
and let $\{\supscript{\mathbb{P}}{i}\}_{i\in[M]}$ be a collection of measures over $\calH$, where $M\geq 2$. Let $\calQ$ be any reference measure over $\calH$, and let $\mathbb{P}$ be the law of $(m^*,\calH)$ under the following process:
\begin{itemize}
    \item Sample $m^*$ uniformly from $[M]$.
    \item Sample $\calH\sim\supscript{\mathbb{P}}{m^*}$.
\end{itemize}
Then for any function $\hatm(\calH)$, if $\mathbb{P}(\hatm = \starm) \geq 1-\delta$, then
\begin{equation}
    \Big(1-\frac{1}{M}\Big)\log(1/\delta) - \log2 \leq \frac{1}{M} \sum_{i=1}^M \dkl(Q||\supscript{\mathbb{P}}{i}).
\end{equation}
\end{lemma}

\subsection{Construction}

If $K\leq 10$ or $T\leq 152^2K\log F$ or $\phi\geq K$, our lower bound directly follows from the cumulative regret lower bound in \cite{foster2020instance}. Hence, without loss of generality, we can assume $K\geq 10, T\geq 152^2K\log F,$ and $\phi\leq K$.

The following construction closely follows lower bound arguments in \cite{foster2020instance}. Let $\A = \{\supscript{a}{1}, \supscript{a}{2}, \dots, \supscript{a}{K} \}$ be an arbitrary set of discrete actions. Let $k = \lfloor 1/\epsilon \rfloor$, and $d$ be parameters that will be fixed later. With $\epsilon\in(0,1)$, note that $1/(2\epsilon) \leq k \leq 1/\epsilon$. We will now define the context set $\Xscript$ as the union of $d$ disjoint partitions $\supscript{\Xscript}{1}, \supscript{\Xscript}{2}, \dots, \supscript{\Xscript}{d}$, where $\supscript{\Xscript}{i} = \{\supscript{x}{i,0}, \supscript{x}{i,1},\dots, \supscript{x}{i,k} \}$ for all $i\in[d]$. Hence, we have $\Xscript = \cup \supscript{\Xscript}{i}$  and $|\Xscript|=d(k+1)$.

For each partition index $i\in[d]$, we construct a policy class $\supscript{\Pi}{i} \subseteq (\supscript{\Xscript}{i} \rightarrow \A)$ as follows. First we let $\supscript{\pi}{i,0}:\supscript{\Xscript}{i}\rightarrow\A$ be the policy that always selects arm $\supscript{a}{1}$, and let $\supscript{\pi}{i,l,b}:\supscript{\Xscript}{i}\rightarrow\A$ be defined as follows for all $l\in [k]$ and $b\in\A_0:=\A\setminus\{\supscript{a}{1}\}$,
\begin{equation}
    \forall \supscript{x}{i,j}\in\supscript{\Xscript}{i}, \;  \supscript{\pi}{i,l,b}(\supscript{x}{i,j})=
        \begin{cases}
        \supscript{a}{1}, & \text{ if $j \neq l$},\\
        b, & \text{ if $j=l$}.
        \end{cases}
\end{equation}
Construct $\supscript{\Pi}{i}:=\{\supscript{\pi}{i,l,b}|l\in [k] \text{ and } b\in\A_0 \}\cup\{\supscript{\pi}{i,0} \}$.\footnote{Here $[k]=\{1,2,\dots,k\}$} Finally, we let $\Pi := \supscript{\Pi}{1}\times \supscript{\Pi}{2} \times \dots \times \supscript{\Pi}{d}$. We will now construct a reward model class $\F$ that induces $\Pi$.

Let $\Delta:=1/4$. For each partition index $i\in[d]$, we construct a reward model class $\supscript{\F}{i} \subseteq (\supscript{\Xscript}{i}\times\A \rightarrow [0,1])$ as follows. First we let $\supscript{f}{i,0}:\supscript{\Xscript}{i}\times\A \rightarrow [0,1]$ be defined as follows,
\begin{equation}
     \forall (\supscript{x}{i,j},a)\in\supscript{\Xscript}{i}\times\A, \; \supscript{f}{i,0}(\supscript{x}{i,j},a) =
        \begin{cases}
        \frac{1}{2}+\Delta, & \text{ if $a=\supscript{a}{0}$},\\
        \frac{1}{2}, & \text{ if $a\in\A_0$}.
        \end{cases}
\end{equation}
For all $l\in [k]$ and $b\in\A_0$, we define $\supscript{f}{i,l,b}:\supscript{\Xscript}{i}\times\A \rightarrow [0,1]$ as follows,
\begin{equation}
    \forall (\supscript{x}{i,j},a)\in\supscript{\Xscript}{i}\times\A, \; \supscript{f}{i,l,b}(\supscript{x}{i,j},a) =
        \begin{cases}
        \frac{1}{2}+\Delta, & \text{ if $a=\supscript{a}{0}$}\\
        \frac{1}{2}+2\Delta, & \text{ if $j=l$ and $a=b$},\\
        \frac{1}{2}, & \text{ otherwise.}
        \end{cases}
\end{equation}
Note that $\supscript{f}{i,l,b}$ differs from $\supscript{f}{i,0}$ only at context $(\supscript{x}{i,l},b)$. Construct $\supscript{\F}{i}:=\{\supscript{f}{i,l,b}|l\in [k] \text{ and } b\in\A_0 \}\cup\{\supscript{f}{i,0} \}$. Finally, we let $\F := \supscript{\F}{1}\times \supscript{\F}{2} \times \dots \times \supscript{\F}{d}$. %
Hence, we have,
\begin{equation}
\label{eq:upper-bound-d-in-lb}
    |\F| = |\supscript{\F}{i}|^d \leq (k \cdot K)^d \implies d \geq \frac{\log|\F|}{\log(K\cdot k)} \geq \frac{\log|\F|}{\log(K/\epsilon)}.
\end{equation}
We choose $d$ to be the largest value such that $F\geq (k\cdot K)^d$. Hence we choose $d = \lfloor \log F/\log (K\cdot k) \rfloor \geq \log F/(2\log (K\cdot k))  $.

To use \cref{lem:fano-reverse-KL}, we will describe a collection of environments that share a common distribution over contexts and only differ in the reward distribution. The context distribution $D_{\Xscript}$ is given by $D_{\Xscript}:=\frac{1}{d}\sum_i \supscript{D_{\Xscript}}{i}$, where $\supscript{D_{\Xscript}}{i}$ is a distribution over $\supscript{\Xscript}{i}$, with $\epsilon$ probability of sampling each context in $\supscript{\Xscript}{i}\setminus\{\supscript{x}{i,0}\}$, and $1-k\epsilon$ probability of sampling the context $\supscript{x}{i,0}$. %

For each block $\supscript{\calX}{i}$, we let $\supscript{\mathbb{P}}{i,0}$ denote the law given by the reward distribution $r(a)\sim\Ber(\supscript{f}{i,0}(x,a))$ for all $x\in\supscript{\calX}{i}$. Further, for any $l\in[k]$ and $b\in\A_0$, we let $\supscript{\mathbb{P}}{i,l,b}$ denote the law given by the reward distribution $r(a)\sim\Ber(\supscript{f}{i,l,b}(x,a))$ for all $x\in\supscript{\calX}{i}$. For any policy $\pi\in\supscript{\Pi}{i}$, we let $\supscript{R}{i,l,b}(\pi)=\E_{
\supscript{\mathbb{P}}{i,l,b}}[r(\pi(x))]$ denote expected reward under $\supscript{\mathbb{P}}{i,l,b}$, and let $\supscript{\Reg}{i,l,b}(\pi)=\supscript{R}{i,l,b}(\supscript{\pi}{i,l,b})-\supscript{R}{i,l,b}(\pi)$ denote expected simple regret under $\supscript{\mathbb{P}}{i,l,b}$.

We use $\rho$ to index environments. Here $\rho=(\rho_1,\dots,\rho_d)$, where $\rho_i=(l_i,b_i)$ for $l_i\in\{0,1,\dots,k\}$ and $b_i\in\A_0$. We let $\mathbb{P}_{\rho}$ denote an environment with the law $\supscript{\mathbb{P}}{i,l_i,b_i}$ for contexts in $\supscript{\calX}{i}$.\footnote{Here $\supscript{\mathbb{P}}{i,0}\equiv \supscript{\mathbb{P}}{i,0,b}$ for all $b\in\A_0$.} Finally let $\pi_{\rho}$ denote the optimal policy under $\mathbb{P}_{\rho}$, and let $\supscript{\pi_{\rho}}{i}$ denote its restriction to $\supscript{\calX}{i}$. Let $\E_{\rho}[\cdot]$ denote the expectation under $\mathbb{P}_{\rho}$. Let $R_{\rho}(\pi)=\E_{\rho}[r(\pi(x))]$ denote the expected reward of $\pi$ under $\mathbb{P}_{\rho}$, and let $\Reg_{\rho}(\pi)=R_{\rho}(\pi_{\rho})-R_{\rho}(\pi)$ denote the simple regret of $\pi$ under $\mathbb{P}_{\rho}$. 

\subsection{Lower bound argument}

We sample $\rho$ from a distribution $\nu$ defined as follows. For each $i\in[d]$, set $l_i=0$ with probability $0.5$, otherwise $l_i$ is selected uniformly from $[k]$. Select $b_i$ uniformly from $\A_0$. Note that when $l_i=0$, we disregard the value of $b_i$.

We let $\hat{\pi}_{\mathbf{A}}\in\Pi$ denotes the policy recommended by the contextual bandit algorithm $\mathbf{A}$ at the end of $T$ rounds, and let $\supscript{\hat{\pi}_{\mathbf{A}}}{i}\in\supscript{\Pi}{i}$ be the restriction of $\hat{\pi}_{\mathbf{A}}$ to block $\supscript{\calX}{i}$. Note that the policy recommended by $\mathbf{A}$ will depend on the environment $\rho$. 

Let $\calI:=\{i\in[d]| \supscript{\Reg}{i,l_i,b_i}(\supscript{\pi_{\mathbf{A}}}{i}) \leq 19\sqrt{\phi\log F/T} \}$. Since we only consider algorithms that guarantee the following with probability at least $19/20$,
\begin{equation}
    \frac{1}{d} \sum_{i=1}^d \supscript{\Reg}{i,l_i,b_i}(\supscript{\pi_{\mathbf{A}}}{i}) = \Reg_{\rho}(\hat{\pi}_{\mathbf{A}}) \leq \sqrt{\phi\log F/T}. 
\end{equation}
Under this event, we have that at most $d/19$ block indices satisfy $\supscript{\Reg}{i,l_i,b_i}(\supscript{\pi_{\mathbf{A}}}{i}) > 19\sqrt{\phi\log F/T}$. Therefore, we have $|\calI|\geq 18d/19$.\\
Define event $M_{i} = \{i\in\calI\}$. We have
\begin{equation}
\label{eq:mi}
\begin{aligned}
     \sum_{i=1}^d P(M_{i}) = \sum_{i=1}^d \E[1\{i\in \calI\}\}]\geq \frac{19}{20}\E[|\calI||\Reg_{\rho}(\hat{\pi}_{\mathbf{A}}) \leq \sqrt{\phi\log F/T}]\geq \frac{9d}{10}.
\end{aligned}
\end{equation}
Consider any fixed index $i$, under the event $M_{i}$, we have the following. First observe for any $(l,b)\neq (l',b')$, we have $\supscript{\Reg}{i,l,b}(\supscript{\pi}{i,l',b'})\geq \epsilon\Delta$. Let $(\hat{l}_i,\hat{b}_i)$ be indices such that $\supscript{\pi_{\mathbf{A}}}{i}=\supscript{\pi}{i,\hat{l}_i,\hat{b}_i}$. We now choose $\epsilon$ such that,
\begin{equation}
\label{eq:choose-epsilon-in-lb}
    \epsilon = \frac{38}{\Delta}\sqrt{\frac{\phi \log F}{T}} \iff \frac{\epsilon\Delta}{2} = 19\sqrt{\frac{\phi \log F}{T}}.
\end{equation}
Hence from definition of $\calI$, we have $\supscript{\Reg}{i,l_i,b_i}(\supscript{\pi_{\mathbf{A}}}{i}) \leq \epsilon\Delta/2$. Further since $\supscript{\Reg}{i,l_i,b_i}(\pi)\geq \epsilon\Delta$ for all $\pi\in\supscript{\Pi}{i}\setminus\{\supscript{\pi}{i,l_i,b_i} \}$, we have $(\hat{l}_i,\hat{b}_i) = (l^*_i,b^*_i)$.

Restating the above result, we have the following. For any $i\in\calI$, with probability $1-1/16$, we have $(\hat{l}_i,\hat{b}_i) = (l^*_i,b^*_i)$. Hence from \cref{lem:fano-reverse-KL}, we have,
\begin{equation}
\label{eq:apply-fano-reverse-kl}
    \begin{aligned}
    &\Big(1 - \frac{1}{(K-1)k}\Big)\log\Big(1/P\big(\overline{M_{1}}\big)\Big) - \log2 \\
    &\leq \frac{1}{(K-1)k} \sum_{l=1}^k\sum_{b\in\A_0} \dkl(\supscript{\mathbb{P}}{i,0}||\supscript{\mathbb{P}}{i,l,b})\\
    &\stackrel{(i)}{=} \frac{1}{(K-1)k} \sum_{l=1}^k\sum_{b\in\A_0} \dkl(\Ber(1/2)||\Ber(1/2+2\Delta))\E_{\supscript{\mathbb{P}}{i,0}}[|\{t|x_t=\supscript{x}{i,l},a_t=b \}|]\\
    &\stackrel{(ii)}{\leq} \frac{1}{(K-1)k} \sum_{l=1}^k\sum_{b\in\A_0} 4\Delta^2\E_{\supscript{\mathbb{P}}{i,0}}[|\{t|x_t=\supscript{x}{i,l},a_t=b \}|]\\
    &= \frac{4\Delta^2}{(K-1)k} \E_{\supscript{\mathbb{P}}{i,0}}[|\{t|x_t\in\supscript{\calX}{i}\setminus\{\supscript{x}{i,0}\},a_t\in\A_0 \}|].
\end{aligned}
\end{equation}
Where (i) follows from the fact that $\supscript{\mathbb{P}}{i,0}$ and $\supscript{\mathbb{P}}{i,l,b}$ are identical unless $x_t=\supscript{x}{i,l}$ and $a_t=b$, and (ii) follows from $\Delta\leq 1/4$. Clearly we have:
\begin{equation}
    \begin{aligned}
             &\E_{\rho\sim\nu} \E_{\rho} \bigg[\sum_{t=1}^T \big(r_t(\pi^*(x_t)) - r_t(a_t) \big) \bigg]\\
             & \geq \Delta \E_{\rho\sim\nu} \E_{\rho}\bigg[\sum_{t=1}^T \sum_{i=1}^d \I\Big(\Big\{x_t\in\supscript{\calX}{i}\setminus\{\supscript{x}{i,0}\}, a_t\in\A_0, l_i=0 \Big\} \Big) \bigg] \\
             & \stackrel{(i)}{\geq} \frac{\Delta}{2} \sum_{i=1}^d\E_{\supscript{\mathbb{P}}{i,0}}\bigg[ |\Big\{t|\;x_t\in\supscript{\calX}{i}\setminus\{\supscript{x}{i,0}\}, a_t\in\A_0 \Big\} | \bigg] \\
             &\stackrel{(ii)}{\geq}\frac{\Delta}{2}\sum_{i=1}^d \bigg(-\Big(1 - \frac{1}{k(K-1)}\Big)\log\big(P\big(\overline{M_{i}}\big) \big) - \log2\bigg) \cdot \frac{(K-1)k}{4\Delta^2}\\
             &=\sum_{i=1}^d \bigg(-\Big(1 - \frac{1}{k(K-1)}\Big)\log\big(1-P(M_{i}) \big) - \log2\bigg) \cdot \frac{(K-1)k}{8\Delta}\\
             &\stackrel{(iii)}{\geq} \sum_{i=1}^d \bigg(\Big(1 - \frac{1}{k(K-1)}\Big)P(M_{i}) - \log2\bigg) \cdot \frac{(K-1)k}{8\Delta}\\
              &\stackrel{(iv)}{\geq} \Big\{\Big(1 - \frac{1}{k(K-1)}\Big)\frac{9}{10} - \log 2\Big\} \cdot \frac{(K-1)kd}{8\Delta} \\
              &\stackrel{(v)}{\geq} \frac{Kkd}{100\Delta} \stackrel{(vi)}{\geq} \frac{dK}{200\Delta\epsilon} \stackrel{(vii)}{=} \frac{1}{7600} \sqrt{\frac{T d^2 K^2}{\phi \log F}} \stackrel{(viii)}{\geq} \frac{1}{15200} \sqrt{\frac{K^2 T \log F}{\phi \log^2 (K\cdot k)}}\\ 
             & \stackrel{(ix)}{\geq} \frac{1}{15200} \sqrt{\frac{K^2 T \log F}{\phi \log^2 (K\cdot T)}}.  \\
    \end{aligned}
\end{equation}
Where (i) follows from the fact that $\nu(l_i=0)=1/2$, (ii) follows from \eqref{eq:apply-fano-reverse-kl} and that $|\calI|\geq d/2$, (iii) uses $\log(1+x)\leq x$, for $x>-1$, , (iv) uses \eqref{eq:mi}, (v) follows from $k\geq 1$ and $K\geq 10$, (vi) follows from $k\geq 1/(2\epsilon)$, (vii) follows from choice of $\epsilon$, (viii) follows from \eqref{eq:upper-bound-d-in-lb}, and (ix) since $k\leq 1/\epsilon \stackrel{\ref{eq:choose-epsilon-in-lb}}{=} \frac{1}{152}\sqrt{\frac{T}{\phi \log F}} \leq T$. This completes the proof of \cref{thm:lower-new}.

\section{Additional Details}

\subsection{Conformal Arm Sets}
\label{sec:basic-result-on-CAS}
The below lemma shows that, for any given policy $\pi$, the conformal arm sets given in \cref{def:confidence_set} can be probabilistically relied on (over the distribution of contexts) to contain arms recommended by  $\pi$, with low regret under the models estimated up to epoch $m$. Recall we earlier define $U_{m}=20\sqrt{\alpha_{m-1}\xi_{m}}$.

\begin{restatable}[Conformal Uncertainty]{lemma}{lemConformalUncertainty}
\label{lem:conformal-uncertainty} 
For any policy $\pi$ and epoch $m$, we have:
\begin{equation}
    \Pr_{x\sim D_{\Xscript},a\sim \pi(\cdot|x) }(a\in C_{m}(x,\zeta)) \geq 1 - \zeta\sum_{\barm \in [m]}\frac{ \Reg_{\hatf_{\barm}}(\pi)}{ (2\barm^2)U_{\barm}}
\end{equation}
\end{restatable}

\begin{proof}
For any policy $\pi$, we have \eqref{eq:confidence_set} holds.
\begin{equation}
\label{eq:confidence_set}
\begin{aligned}
        &\Pr_{x\sim D_{\Xscript},a\sim \pi(\cdot|x) }(a\notin C_{m}(x,\zeta))\\
        &\leq \Pr_{x\sim D_{\Xscript},a\sim \pi(\cdot|x) }\Big(\bigcup_{\barm\in[m]} \{a\notin \Tilde{C}_{\barm}(x,\zeta/(2\barm^2)) \}\Big) \\
        &\stackrel{(i)}{\leq} \sum_{\barm\in[m]}\Pr_{x\sim D_{\Xscript},a\sim \pi(\cdot|x) }\Big( \hatf_{\barm}(x, \pi_{\hatf_{\barm}}(x)) - \hatf_{\barm}(x,a)> \frac{(2\barm^2)U_{\barm}}{\zeta}\Big)\\
        &\stackrel{(ii)}{\leq} \sum_{\barm\in[m]}\frac{ \Reg_{\hatf_{\barm}}(\pi)}{(2\barm^2)U_{\barm}/\zeta}.
\end{aligned}
\end{equation}
Where (i) follows from union bound and (ii) follows from Markov's inequality.
\end{proof} 
Recall that right after \Cref{lemma:regret-under-algsafe-epoch}, we show that $\Reg_{\hatf_{\barm}}\leq U_{\barm}$ with high-probability for any $\barm\in[\msafealg]$. Hence, \Cref{lem:conformal-uncertainty} gives us that with high-probability we have $\Pr_{x\sim D_{\Xscript},a\sim \pi^*(\cdot|x) }(a\in C_{m}(x,\zeta))\geq 1-\zeta$. While we don't directly use \Cref{lem:conformal-uncertainty}, this lemma helps demonstrate the utility of CASs.

\subsection{Argument for Surrogate Objective}
\label{app:prove-objective}

\Cref{lem:variance-penalized-policy-learning} is a self-contained result proving that guarantying tighter bounds on the optimal cover leads to tighter simple regret bounds for any contextual bandit algorithm. Hence the optimal cover is a valid surrogate objective for simple regret. This lemma is not directly used in the analysis of $\omega$-RAPR, however similar results (see \Cref{theorem:Simple-Regret-for-RAPR}) were proved and used. Note that the parameters below (including $\alpha$) are not directly related to parameters maintained by $\omega$-RAPR.

\begin{restatable}[Valid Surrogate Objective]{lemma}{lemPolicyLearning}
\label{lem:variance-penalized-policy-learning} 
Suppose $\Pi$ is a finite class and suppose a contextual bandit algorithm collects $T$ samples using kernels $(p_t)_{t\in[T]}$ such that $p_t(\cdot|\cdot)\geq \sqrt{\frac{\ln(4|\Pi|/\delta)}{\alpha T}}$. %
Further suppose that the following condition holds with some $\alpha\in[1,\infty)$:
\begin{equation}
    \label{eq:objective-minimized}
    \frac{1}{T}\sum_{t=1}^T V(p_t,\pi^*) \leq \alpha
\end{equation}
Then we can estimate a policy $\hat{\pi}\in \Pi$ such that with probability at least $1-\delta$, we have:
\begin{equation}
    |R(\pi^*)-R(\hat{\pi})| \leq \ordO\bigg(\sqrt{\frac{\alpha \ln(4|\Pi|/\delta)}{T}}\bigg).
\end{equation}
\end{restatable}

\begin{proof}
WOLG we assume $T\geq \ln(4|\Pi|/\delta)$, since otherwise the result trivially holds. Now consider any policy $\pi$. Let $y_t := \frac{r_t\I(\pi(x_t)=a_t)}{p_t(\pi(x_t)|x_t)}$. Now note that:
\begin{equation}
\label{eq:var-yt}
    \text{Var}_t[y_t] \leq \E_{D(p_t)}[y_t^2] = \E_{(x_t,a_t,r_t)\sim D(p_t)}\bigg[\frac{r_t^2\I(\pi(x_t)=a_t)}{p^2_t(\pi(x_t)|x_t)}\bigg] \leq \E_{x\sim D_{\Xscript}}\bigg[\frac{1}{p_t(\pi(x_t)|x_t)}\bigg]=V(p_t,\pi).
\end{equation}
Then from from a Freedman-style inequality \citep[See theorem 13 in ][]{dudik2011efficient}, we have with probability at least $1-\delta/(2|\Pi|)$ that the following holds:
\begin{equation}
\label{eq:freedman-policy-value}
\begin{aligned}
    &\bigg| \sum_{t=1}^T (y_t - R(\pi)) \bigg| \leq 2\max\Bigg\{ \sqrt{\sum_{t=1}^T\text{Var}(y_t)\ln(4|\Pi|/\delta)}, \frac{\ln(4|\Pi|/\delta)}{\sqrt{\frac{\ln(4|\Pi|/\delta)}{\alpha T}}}\Bigg\}\\
    \stackrel{(i)}{\implies} &\bigg| \frac{1}{T}\sum_{t=1}^T \frac{r_t\I(\pi(x_t)=a_t)}{p_t(\pi(x_t)|x_t)} - R(\pi) \bigg| \leq 2\sqrt{\frac{\ln(4|\Pi|/\delta)}{T}\max\Bigg\{ \frac{1}{T}\sum_{t=1}^T V(p_t,\pi), \alpha\Bigg\}}
\end{aligned}
\end{equation}
Here (i) follows from \eqref{eq:var-yt}. Similarly with probability at least $1-\delta/(2|\Pi|)$ the following holds:
\begin{equation}
\label{eq:freedman-variance-bound}
\begin{aligned}
    &\bigg| \frac{1}{T}\sum_{t=1}^T \frac{1}{p_t(\pi(x_t)|x_t)} - \frac{1}{T}\sum_{t=1}^T V(p_t,\pi) \bigg|\\ 
    &\stackrel{(i)}{\leq} \frac{2}{T}\max\Bigg\{ \sqrt{\sum_{t=1}^T\text{Var}\Big(\frac{1}{p_t(\pi(x_t)|x_t)}\Big)\ln(4|\Pi|/\delta)}, \frac{\ln(4|\Pi|/\delta)}{\sqrt{\frac{\ln(4|\Pi|/\delta)}{\alpha T}}}\Bigg\}\\
    &\stackrel{(ii)}{\leq} \frac{2}{T}\max\Bigg\{ \sqrt{T\frac{\alpha T}{\ln(4|\Pi|/\delta)}\ln(4|\Pi|/\delta)}, \sqrt{\alpha T\ln(4|\Pi|/\delta)}\Bigg\} \stackrel{(iii)}{\leq} 2\sqrt{\alpha} \stackrel{(iv)}{\leq} 2\alpha.
\end{aligned}
\end{equation}
Here (i) follows from Freedman's inequality, (ii) follows from the lower bound on $p_t$, (iii) follows from $T\geq \ln(4|\Pi|/\delta)$, and (iv) follows from $\alpha\geq 1$. Hence the above events hold with probability at least $1-\delta$ for all policies $\pi\in\Pi$. Now let $\hat{\pi}$ be given as follows.
\begin{equation}
\label{eq:hatpi-freedman}
    \hat{\pi} \in \argmax_{\pi\in\Pi} \frac{1}{T}\sum_{t=1}^T \frac{r_t\I(\pi(x_t)=a_t)}{p_t(\pi(x_t)|x_t)} - 2\sqrt{\frac{\ln(4|\Pi|/\delta)}{T}\Bigg( 2\alpha + \frac{1}{T}\sum_{t=1}^T \frac{1}{p_t(\pi(x_t)|x_t)}\Bigg)}
\end{equation}
We then have the following lower bound on $R(\hat{\pi})$ using the definition of $\hat{\pi}$ and the above to high-probability events.
\begin{equation}
\begin{aligned}
    &R(\hat{\pi})\stackrel{(i)}{\geq} \frac{1}{T}\sum_{t=1}^T \frac{r_t\I(\hat{\pi}(x_t)=a_t)}{p_t(\hat{\pi}(x_t)|x_t)} - 2\sqrt{\frac{\ln(4|\Pi|/\delta)}{T}\max\Bigg\{ \frac{1}{T}\sum_{t=1}^T V(p_t,\hat{\pi}), \alpha\Bigg\}} \\ 
    &\stackrel{(ii)}{\geq} \frac{1}{T}\sum_{t=1}^T \frac{r_t\I(\hat{\pi}(x_t)=a_t)}{p_t(\hat{\pi}(x_t)|x_t)} - 2\sqrt{\frac{\ln(4|\Pi|/\delta)}{T}\Bigg( 2\alpha + \frac{1}{T}\sum_{t=1}^T \frac{1}{p_t(\hat{\pi}(x_t)|x_t)}\Bigg)}\\
    &\stackrel{(iii)}{\geq} \frac{1}{T}\sum_{t=1}^T \frac{r_t\I(\pi^*(x_t)=a_t)}{p_t(\pi^*(x_t)|x_t)} - 2\sqrt{\frac{\ln(4|\Pi|/\delta)}{T}\Bigg( 2\alpha + \frac{1}{T}\sum_{t=1}^T \frac{1}{p_t(\pi^*(x_t)|x_t)}\Bigg)}\\
    &\stackrel{(iv)}{\geq} R(\pi^*) - 4\sqrt{\frac{\ln(4|\Pi|/\delta)}{T}\Bigg( 4\alpha + \frac{1}{T}\sum_{t=1}^T V(p_t,\pi^*)\Bigg)} \geq R(\pi^*) - 4\sqrt{\frac{5\alpha\ln(4|\Pi|/\delta)}{T}} 
\end{aligned}
\end{equation}
Here (i) follows from \eqref{eq:freedman-policy-value}, (ii) follows from \eqref{eq:freedman-variance-bound}, (iii) follows from \eqref{eq:hatpi-freedman}, and (iv) follows from \eqref{eq:freedman-policy-value} and \eqref{eq:freedman-variance-bound}. This completes the proof.
\end{proof}

\subsection{Testing Misspecification via CSC}

We restate the misspecification test that is used at the end of epoch $m$ and argue how this test can be solved via two calls to a cost sensitive classification solver. First, let us restate the test in \eqref{eq:misspecification-test-restated}.
\begin{equation}
\label{eq:misspecification-test-restated}
    \begin{aligned}
             &\max_{\pi\in\Pi\cup\{p_{m+1}\}}{|\hatR_{m+1,\hatf_{m+1}}(\pi)-\hatR_{m+1}(\pi)|}  - {\sqrt{\alpha_m\xi_{m+1}}\sum_{\barm\in[m]}\frac{\hatR_{m+1,\hatf_{\barm}}(\pi_{\hatf_{\barm}})-\hatR_{m+1,\hatf_{\barm}}(\pi)}{40\barm^2 \sqrt{\alpha_{\barm-1}\xi_{\barm}}}}\\ 
             &\leq 2.05\sqrt{\alpha_m\xi_{m+1}} + 1.1\sqrt{\xi_{m+1}},
    \end{aligned}
\end{equation}
We are interested in calculating the value of the maximization problem in \eqref{eq:misspecification-test-restated}. To calculate this maximum, we need to fix our estimators. Let $\hatR_{m+1,f}(\pi):=\frac{1}{|S_{m,3}|}\sum_{t\in S_{m,3}}f(x_t,\pi(x_t))=\frac{1}{|S_{m,3}|}\sum_{t\in S_{m,3}}\E_{a\sim \pi(\cdot|x_t)} f(x_t,a)$ for any policy $\pi$ and reward model $f$, which is the only obvious estimator we could think off for $R_f(\pi)$. Also let us use IPS estimaton for policy evaluation (the same argument works for DR), $\hatR_{m+1}(\pi):=\frac{1}{|S_{m,3}|}\sum_{t\in S_{m,3}}\frac{\pi(a_t|x_t)r_t(a_t)}{p_m(a_t|x_t)}$. \footnote{When evaluating a general kernel $q$, we use the natural extension of these estimators of policy value. In particular, simply replace $\pi(\cdot|x)$ with $q(\cdot|x)$ in their formulas.} \footnote{Up to constant factors, these estimators give us the best rates in \Cref{ass:policy-evaluation-oracles} with finite classes. These estimators are also used in several contextual bandit papers \citep[e.g.,][]{agarwal2014taming,li2022instance}.} Note that the value of the maximization problem in \eqref{eq:misspecification-test-restated} is equal to $\max(L_1,L_2,L_3)$, where $\{L_i|i\in[3]\}$ are defined as follows.
\begin{equation}
\label{eq:misspecification-test-decomposed}
    \begin{aligned}
             &L_1:=\max_{\pi\in\Pi}{\hatR_{m+1,\hatf_{m+1}}(\pi)-\hatR_{m+1}(\pi)}  - {\sqrt{\alpha_m\xi_{m+1}}\sum_{\barm\in[m]}\frac{\hatR_{m+1,\hatf_{\barm}}(\pi_{\hatf_{\barm}})-\hatR_{m+1,\hatf_{\barm}}(\pi)}{40\barm^2 \sqrt{\alpha_{\barm-1}\xi_{\barm}}}}\\ 
             &L_2:=\max_{\pi\in\Pi}{\hatR_{m+1}(\pi)-\hatR_{m+1,\hatf_{m+1}}(\pi)}  - {\sqrt{\alpha_m\xi_{m+1}}\sum_{\barm\in[m]}\frac{\hatR_{m+1,\hatf_{\barm}}(\pi_{\hatf_{\barm}})-\hatR_{m+1,\hatf_{\barm}}(\pi)}{40\barm^2 \sqrt{\alpha_{\barm-1}\xi_{\barm}}}}\\
             &L_3:={|\hatR_{m+1,\hatf_{m+1}}(p_{m+1})-\hatR_{m+1}(p_{m+1})|}  - {\sqrt{\alpha_m\xi_{m+1}}\sum_{\barm\in[m]}\frac{\hatR_{m+1,\hatf_{\barm}}(\pi_{\hatf_{\barm}})-\hatR_{m+1,\hatf_{\barm}}(p_{m+1})}{40\barm^2 \sqrt{\alpha_{\barm-1}\xi_{\barm}}}}
    \end{aligned}
\end{equation}
Note that $L_3$ doesn't involve any optimization and can be easily calculated. Substituting value of these estimators for $L_1$ and $L_2$, we get.
\begin{equation}
\label{eq:misspecification-test-decomposed-substituted}
    \begin{aligned}
             &L_1=\max_{\pi\in\Pi} \sum_{t\in S_{m,3}}\frac{1}{|S_{m,3}|}\bigg(\hatf_{m+1}(x_t,\pi(x_t))-\frac{\pi(a_t|x_t)r_t(a_t)}{p_m(a_t|x_t)}\\ 
             &\;\;\;\;\;\;\;\;\;\;\;\;\;\;\;\;\;\;\;\;\;\;\;\;\;\;\;\;\;\;\;- {\sqrt{\alpha_m\xi_{m+1}}\sum_{\barm\in[m]}\frac{\hatf_{\barm}(x_t,\pi_{\hatf_{\barm}}(x_t))-\hatf_{\barm}(x_t,\pi(x_t))}{40\barm^2 \sqrt{\alpha_{\barm-1}\xi_{\barm}}}}\bigg)\\ 
             &L_2=\max_{\pi\in\Pi} \sum_{t\in S_{m,3}}\frac{1}{|S_{m,3}|}\bigg(\frac{\pi(a_t|x_t)r_t(a_t)}{p_m(a_t|x_t)}-\hatf_{m+1}(x_t,\pi(x_t))\\ 
             &\;\;\;\;\;\;\;\;\;\;\;\;\;\;\;\;\;\;\;\;\;\;\;\;\;\;\;\;\;\;\;- {\sqrt{\alpha_m\xi_{m+1}}\sum_{\barm\in[m]}\frac{\hatf_{\barm}(x_t,\pi_{\hatf_{\barm}}(x_t))-\hatf_{\barm}(x_t,\pi(x_t))}{40\barm^2 \sqrt{\alpha_{\barm-1}\xi_{\barm}}}}\bigg)
    \end{aligned}
\end{equation}
Clearly, both $L_1$ and $L_2$ are cost-sensitive classification problems \citep[see][for problem definition]{krishnamurthy2017active}.In both, we need to find a policy (classifier) that maps contexts to arms (classes), incurring a score (cost) for each decision such that the total score (cost) is maximized (minimized). Hence the misspecification test we use only requires two calls to CSC solvers.

\subsection{Simulation}
\label{sec:simulations}

We ran uniform RCT, LinUCB, LinTS, $1$-RAPR, and $4$-RAPR with linear function classes and an exploration horizon of $5000$ on a synthetic data generating process (DGP).\footnote{The LinUCB scaling parameter was set to a default of $0.25$, we similarly let $\sqrt{\xi(T,0.5)} = 0.25\times \sqrt{d/T}$. We also set the bloated constant of $20$ in \Cref{def:confidence_set} to be $1$.}
\paragraph{Data generating process.} We consider four arms, i.e., $\calA=[8]$. The context $x=(x_1,x_2)$ is uniformly sampled from four regions on the  two-dimensional unit ball; and in specific, $x$ is generated via the following distribution:
\begin{enumerate}
    \item $\tilde{x}_1\sim \mbox{Uniform}(0.8, 1.0)$
    \item $\tilde{x}_2=\sqrt{1-\tilde{x}_1^2} \cdot z$, where $z\sim \mbox{Uniform}\{-1,1\}$.
    \item Sample region index $r\sim \mbox{Uniform}\{0,1,2,3\}$:
    \begin{itemize}
        \item if $r=0$: $\{x_1,x_2\}=\{\tilde{x}_1,\tilde{x}_2\}$.
        \item if $r=1$: $\{x_1,x_2\}=\{\tilde{x}_2,\tilde{x}_1\}$.
        \item if $r=2$: $\{x_1,x_2\}=\{-\tilde{x}_1,-\tilde{x}_2\}$.
        \item if $r=3$: $\{x_1,x_2\}=\{-\tilde{x}_1,-\tilde{x}_2$\}.
    \end{itemize}
    \item The reward for each arm is $0.4$ plus a linear function of the contexts. The linear parameters for the $8$ arms are $\{(a,b)||a|+|b|=1,|a|,|b|\in\{0,0.4,0.6,1\} \}$. Hence, the conditional expected rewards lies in the range $[0.2,0.6]$. Finally, the noise was sampled uniformly at random from $[-0.4,0.4]$.
\end{enumerate}

\begin{wrapfigure}{r}{0.5\textwidth} %
    \centering
    \includegraphics[width=0.5\textwidth]{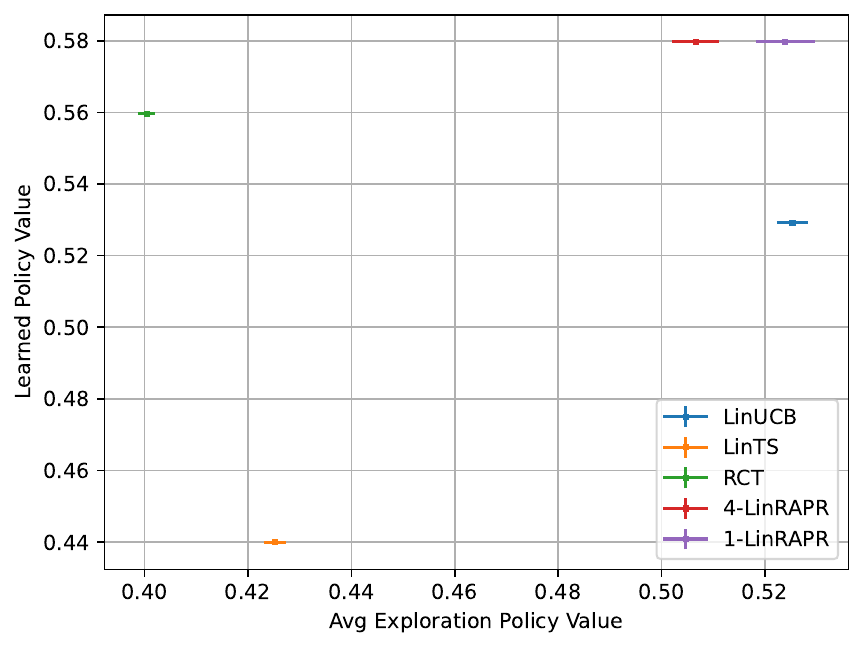}
\end{wrapfigure}

A simulation run takes less than $9$ seconds for any of these algorithms on a laptop with 16GB RAM and an Apple M1 Pro chip, demonstrating the computational tractability of this approach. We provide a scatter plot (aggregating results from $50$ runs) showing (i) the value of the average reward during exploration (as a proxy for cumulative regret, $x$-axis) and (ii) the value of the learned policy at the end of the experiment (as a proxy for simple regret of learned policy, $y$-axis). On simple regret, we see that $4$-RAPR $\approx$ $1$-RAPR $>$ RCT $>$ LinUCB $>$ LinTS. On cumulative regret, we see that LinUCB $>$  $1$-RAPR $>$ $4$-RAPR $>$ LinTS $>$ RCT. The RAPR algorithms achieve the best simple regret performance and achieve competitive performance on cumulative regret for this DGP. However, the fact that both RAPR algorithms learn policies of similar values suggests that our CASs are larger than necessary for at least some risk levels on this DGP. Further refining CAS is an important direction of future work.

\end{document}